%% file: arxiv.tex
\pdfoutput=1
\documentclass[11pt]{article}

\usepackage[preprint]{acl}

\usepackage{times}
\usepackage{latexsym}

\usepackage[T1]{fontenc}
\usepackage[utf8]{inputenc}

\usepackage{microtype}

\usepackage{inconsolata}

\usepackage{amsthm}
\usepackage{amsmath}
\usepackage{amssymb}
\usepackage{algorithm}
\usepackage{algorithmic}
\usepackage{hyperref}
\usepackage{url}
\usepackage[pdftex]{graphicx}
\usepackage{subcaption}
\usepackage{multirow}
\usepackage{mathtools}
\usepackage{xcolor}
\usepackage{authblk}
\usepackage{placeins}
\usepackage{wrapfig}
\usepackage{listings}
\usepackage[most]{tcolorbox}
\usepackage{enumitem}

\usepackage{booktabs}
\usepackage{longtable}
\usepackage{array}
\usepackage{tikz}
\usepackage{colortbl}
\usepackage{pgfplots}
\usepackage{caption}
\usepackage{subcaption}
\usepackage{xcolor}
\definecolor{dlow}{RGB}{225,225,240}
\definecolor{low}{RGB}{255,245,240}
\definecolor{med}{RGB}{252,146,114}
\definecolor{high}{RGB}{165,21,22}

\definecolor{codegreen}{rgb}{0,0.6,0}
\definecolor{codegray}{rgb}{0.5,0.5,0.5}
\definecolor{codepurple}{rgb}{0.58,0,0.82}
\definecolor{backcolour}{rgb}{0.95,0.95,0.92}

\lstdefinestyle{mystyle}{
    backgroundcolor=\color{backcolour},   
    commentstyle=\color{codegreen},
    keywordstyle=\color{magenta},
    numberstyle=\tiny\color{codegray},
    stringstyle=\color{codepurple},
    basicstyle=\ttfamily\footnotesize,
    breaklines=true,                 
    captionpos=b,                    
    keepspaces=true,                 
    showspaces=false,                
    showtabs=false,                  
}

\usepackage[normalem]{ulem}
\usepackage{booktabs} % for better table aesthetics
\if aa
\newcommand{\maz}[1]{{\color{red}Maz: #1}}
\else 
\newcommand{\maz}[1]{}
\fi
\newcommand{\coticl}{\emph{CoT-ICL Lab}}
\newcommand{\coticlnew}{\emph{CoT-ICL Lab-2.0}}
\input{notations}

\title{To Think or Not to Think: The Hidden Cost of Meta-Training with Excessive CoT Examples}

\author{
 \textbf{Vignesh Kothapalli\textsuperscript{1}},
 \textbf{Ata Fatahibaarzi\textsuperscript{2}},
 \textbf{Hamed Firooz}, %\textsuperscript{3}}
 \textbf{Maziar Sanjabi}%\textsuperscript{4}}
\\
 \textsuperscript{1}Stanford University,
 \textsuperscript{2}LinkedIn AI
}

\begin{document}
\maketitle

\begin{abstract}

Chain-of-thought (CoT) prompting combined with few-shot in-context learning (ICL) has unlocked significant reasoning capabilities in large language models (LLMs). 
However, ICL with CoT examples is ineffective on novel tasks when the pre-training knowledge is insufficient. We study this problem in a controlled setting using the \coticl~\citep{kothapalli2025cot} framework, and propose meta-training techniques to learn novel abstract reasoning tasks in-context. Although CoT examples facilitate reasoning, we noticed that their excessive inclusion during meta-training degrades performance when CoT supervision is limited. To mitigate such behavior, we propose \texttt{CoT-Recipe}, a formal approach to modulate the mix of CoT and non-CoT examples in meta-training sequences. We demonstrate that careful modulation via \texttt{CoT-Recipe} can increase the accuracy of transformers on novel tasks by up to $300\%$ even when there are no CoT examples available in-context. We confirm the broader effectiveness of these techniques by applying them to pretrained LLMs (Qwen2.5 series) for symbolic reasoning tasks and observing gains of up to $130\%$ in accuracy. Code is available at: \href{https://github.com/kvignesh1420/cot-icl-lab}{https://github.com/kvignesh1420/cot-icl-lab}
% We address this gap in abstract reasoning tasks by proposing meta-training techniques in a controlled setting using the \coticl~\citep{kothapalli2025cot} framework and demonstrate the effectiveness of these findings for fine-tuning LLMs. 
% In particular, we propose~\coticlnew, which introduces (1) special abstract tokens to delineate reasoning and answer segments in CoT examples and (2) \texttt{CoT-Recipe}, a formal approach to modulate the mix of CoT and non-CoT examples in training sequences. To mitigate such behavior, we leverage the \texttt{CoT-Recipe} approach and demonstrate that careful balancing of CoT and non-CoT examples can increase the accuracy of transformers by up to $300\%$ even when there are no CoT examples available in-context. We confirm the broader effectiveness of these meta-training techniques by applying them to pretrained LLMs (Qwen2.5 series) for symbolic reasoning tasks and observing gains of up to $130\%$ in accuracy. Code is available at: \href{https://github.com/kvignesh1420/cot-icl-lab}{https://github.com/kvignesh1420/cot-icl-lab}
\end{abstract}

% \begin{abstract}
% Chain-of-thought (CoT) prompting combined with few-shot in-context learning (ICL) has unlocked significant reasoning capabilities in large language models (LLMs). However, when the pre-training knowledge is insufficient for the tasks at hand, there are no principled fine-tuning data recipes that facilitate reasoning with limited in-context CoT supervision. We address this gap in a controlled setting using the \coticl~\citep{kothapalli2025cot} framework and demonstrate the effectiveness of these findings for supervised fine-tuning of LLMs. In particular, we propose~\coticlnew, which introduces (1) special abstract tokens that delineate reasoning and answer segments and (2) \texttt{CoT-Recipe}, a formal approach to modulate the mix of CoT and non-CoT examples in training sequences. Although CoT examples facilitate reasoning, we noticed that their excessive inclusion in few-shot prompts during training degrades model performance at inference when CoT supervision is limited. To mitigate such behavior, we leverage the \texttt{CoT-Recipe} approach and demonstrate that careful balancing of CoT and non-CoT examples during training can increase accuracy by up to $300\%$. On symbolic reasoning tasks using pretrained LLMs (Qwen2.5 series), we observe gains of up to $130\%$ in accuracy, confirming the broader effectiveness. Code is available at: \href{https://github.com/kvignesh1420/cot-icl-lab}{https://github.com/kvignesh1420/cot-icl-lab}
% \end{abstract}

\section{Introduction}

\begin{figure*}[h]
    \centering
    \includegraphics[width=0.9\linewidth]{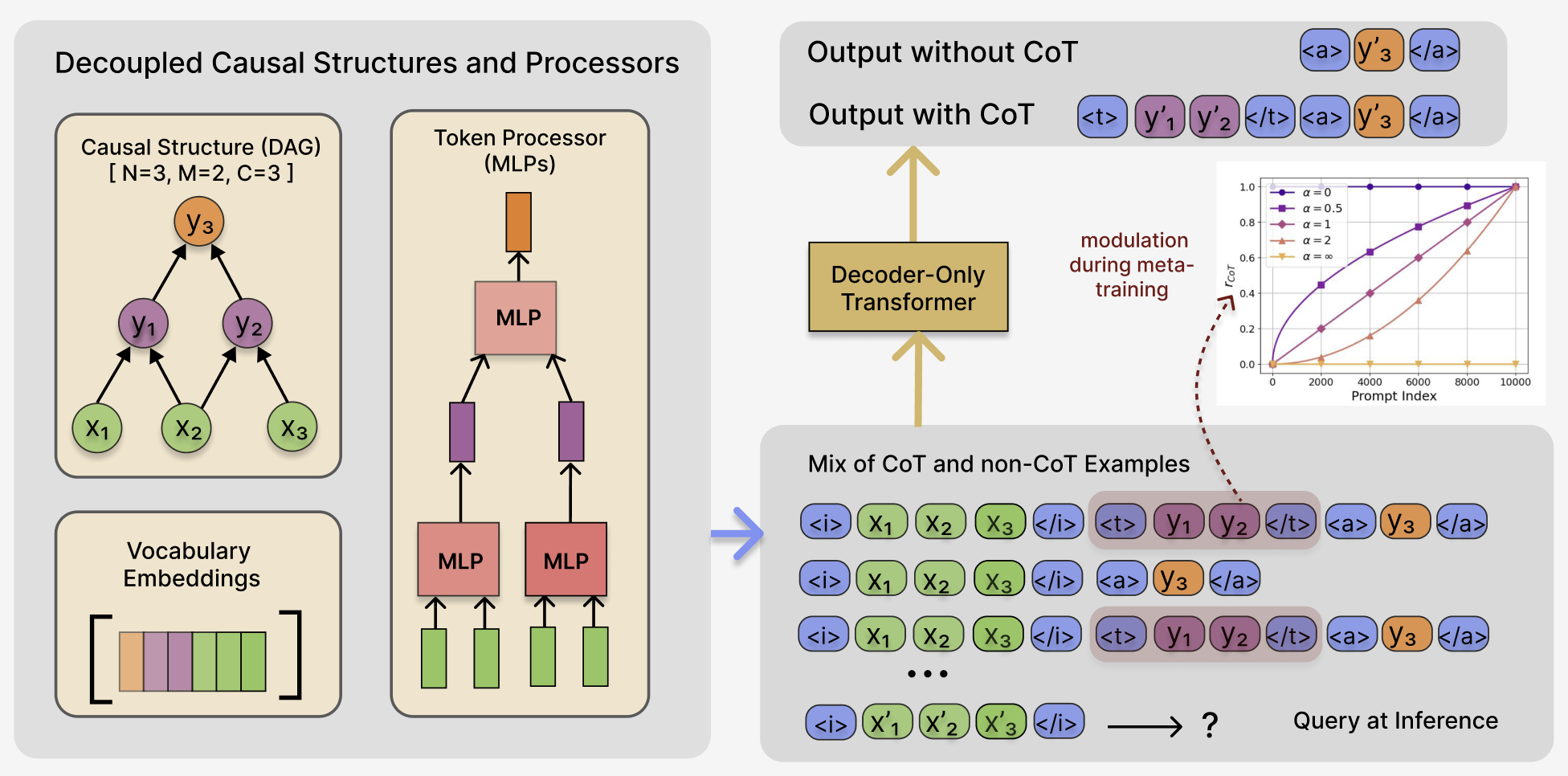}
   \caption{The \coticlnew~framework. (1) We incorporate special tokens (marked in blue) to act as delimiter tokens between input, intermediate/thinking, and answer tokens. (2) Each sequence is constructed using a specific DAG, which determines the number of input and chain tokens per example. (3) The choice of $r_{CoT}$ is varied across sequences as per the \texttt{CoT-Recipe} and modulates the mix of \textit{CoT/standard examples} for meta-training.}
    \label{fig:cot_icl_lab_extended}
\end{figure*}

% \begin{figure}
%     \centering
%     \includegraphics[width=\linewidth]{images/intro.png}
%     \caption{Design philosophy of~\coticl. (Left) Vocabularies corresponding to natural language, image patches and audio snippets are generalized using an abstract vocabulary, that is devoid of semantics. (Right) The abstract tokens are randomly sampled to create a sequence of input tokens ($2$ in this case) for each in-context example. The underlying causal structure and CoT token generation (for the $3$ CoT tokens in this case) is determined by a randomly sampled DAG and MLP token embedding processor respectively. \maz{why do we need to emphasize multimodal?} }
%     \label{fig:intro}
%     \vspace{-4mm}
% \end{figure}

%  \begin{figure}
%     \centering
%          \includegraphics[width=\linewidth]{images/power_law_recipes.jpg}
%     \caption{Systematic control over the probability of \textit{CoT examples} in few-shot training prompts. Here $r_{CoT}$, determined by the \texttt{CoT-Recipe} parameter $\alpha$, represents the probability that an in-context example is a \textit{CoT example} instead of a non-CoT or \textit{standard example}. }
%     \label{fig:power_law_recipes}
%     \vspace{-3mm}
% \end{figure}

Recent advances in large language models (LLMs) have demonstrated remarkable reasoning abilities when prompted to generate step-by-step solutions to problems. A prime example is \emph{chain-of-thought} (CoT) prompting~\citep{Wei2022, Kojima2022}, where appending a prompt with \textit{``Let's think step by step"} can induce an LLM to generate intermediate ``thought'' steps~\cite{nye2022show}, and enable them to tackle multi-step problems. CoT prompting, specially when combined with \textit{in-context learning} (ICL)~\citep{brown2020language} has yielded impressive gains on arithmetic, commonsense, and symbolic reasoning benchmarks \citep{Wei2022, Kojima2022}, and has become a key technique for eliciting reasoning in LLMs.

Despite its successes, CoT in-context learning (CoT-ICL) faces several limitations. First, models often require carefully chosen exemplars~\cite{min2022rethinking, zhao2021calibrate} for effective ICL. Next, few-shot CoT prompting uses handcrafted demonstrations of question-answer pairs with rationales, which can be labor-intensive to create for each task \citep{Wei2022, kim2023the}.
% Zero-shot CoT prompting \citep{Kojima2022} offers a shortcut by using a generic trigger phrase, but not all models or scenarios respond equally well to this prompt\maz{citation?}.
Moreover, the benefits of CoT prompting tend to emerge with model scale; while smaller models struggle to produce answers with good reasoning unless fine-tuned to do so~\citep{Li2023, huang2023large, ho2023large, kim2023the}. 
%Second, always producing a detailed chain-of-thought can be inefficient or unnecessary for simpler queries, and it may introduce opportunities for the model to go off-track. In standard CoT prompting, the model decides implicitly when to stop reasoning and provide an answer, which could lead to overly long or even irrelevant explanations. 
% Moreover, if a model is exclusively trained or used with CoT-style prompts, it might become overly reliant on them, failing to solve problems directly when no rationale is present in exemplar. This over-reliance and lack of reasoning generalization is undesirable and could degrade performance in settings where generating exemplars with a rationale is infeasible \maz{citation? if this is our result, we should mention it.}. 
%Finally, while having the model think aloud is useful, it is not guaranteed to be \emph{correct} or optimal on the first attempt. Different chains of thought can lead to different answers, so harnessing multiple reasoning paths could be beneficial.
% \maz{Rest of Introduction needed}

These issues are exacerbated when the tasks are entirely novel and the pre-training knowledge of LLMs is insufficient to generate the correct responses. For example, prompting LLMs to answer domain-specific queries whose background knowledge is not included in the pre-training data. In such scenarios, the models have to rely solely on the (possibly limited) task descriptions and the in-context examples to generate a response. Having CoT examples aids in revealing more information about the task, but their availability might be limited due to data curation constraints. 

While previous works have explored meta-training approaches~\cite{min2022metaicl, chen2022meta} with ICL as an objective, the role of CoT exemplars in the data recipes and inference prompts has been largely overlooked. By addressing this gap, our work aims to understand if models can be meta-trained to effectively leverage the (limited) CoT examples at inference for solving novel tasks. In particular, we study this problem in a controlled setting using the~\coticl~ framework~\cite{kothapalli2025cot} for abstract reasoning with transformers. Although CoT exemplars can aid in learning about the task, we find that their excessive inclusion during meta-training can be detrimental to the model's performance when such supervision is limited (during inference). We propose principled data curation recipes to modulate the mix of CoT and non-CoT examples in sequences to address this issue (Figure~\ref{fig:cot_icl_lab_extended}). We also create a novel symbolic reasoning dataset called \texttt{CIL-LangSym} and meta-train LLMs (Qwen-2.5 series) with our data recipes to show that they can reason effectively on these domain-specific queries (1) even in the absence of CoT exemplars and (2) limited task descriptions. In summary, our key contributions are as follows:

% While there has been significant interest in curating diverse reasoning datasets, the focus has largely been on zero-shot CoT prompts via notions of problem complexity and reasoning length~\cite{lambert2024tulu, guo2025deepseek,abdin2025phi, muennighoff2025s}. 
% However, the few-shot CoT prompts have been largely overlooked during training. In this work, we show that excessive inclusion of CoT examples in few-shot prompts during training is detrimental to the model's reasoning capability during inference.
% In particular, they fail to reason when supervision via CoT examples is limited. Our work systematically studies this behavior for abstract reasoning tasks and proposes techniques to carefully curate the training data and mitigate such issues. 

% by leveraging and enhancing the recently proposed~\coticl~framework~\cite{kothapalli2025cot} for abstract reasoning tasks.
% Figure~\ref{fig:intro} illustrates the design philosophy, where vocabularies corresponding to different modalities can be generalized to an `abstract vocabulary' simply by disassociating semantics from the token ids. This flexibility allows us to generate ICL examples using only the token indices where input tokens are sampled randomly and chain tokens are generated using pre-determined token processing functions such as MLPs. Furthermore, the causal dependencies between tokens can be controlled via DAGs, thus facilitating a mechanistically interpretable analysis of transformer learning dynamics.

\begin{enumerate}
    \item We introduce~\coticlnew, an extension of~\coticl~by~\citet{kothapalli2025cot} for meta-training on abstract reasoning tasks. It incorporates special tokens to isolate the `input', `thinking', and `answer' segments of examples, and allows \textit{reasoning control} in trained transformers by enabling dynamic invocation of multi-step reasoning or direct final answers as needed. 
    % Thus, allowing investigation on the role of task diversity and CoT for facilitating ICL.
    % a systematic study on (a) modulating the mix of CoT and non-CoT examples during training and (b) inference strategies that force a transformer to either `Think' (i.e, generate CoT tokens) or directly `Answer' (i.e, generate only the answer token) for a given abstract reasoning query.
    \item We introduce \texttt{CoT-Recipe}, a formal approach to modulate the mix of CoT and non-CoT examples in sequences during meta-training. In essence, datasets curated with this approach allow the models to reason even in the absence of CoT exemplars. 
    % \maz{this is better wording. let's use something similar in abstract too.}
    % \item To validate the effectiveness of \texttt{CoT-Recipe}s, we employ the design principles of~\coticlnew~and create a new NLP based symbolic reasoning dataset called \texttt{CIL-LangSym}. Our results show that the pretrained Qwen-2.5 series of LLMs can indeed be fine-tuning based on \texttt{CoT-Recipe}'s for improving reasoning control. 
    \item We leverage the insights from systematic experiments with~\coticlnew~to improve the CoT-ICL capabilities of real-world LLMs (\texttt{Qwen-2.5} series) on symbolic reasoning tasks. Especially when their pre-trained knowledge is insufficient for reasoning.
\end{enumerate}

% Especially, we formalize the diversification of few-shot prompts (i.e modulating a mix of CoT and non-CoT examples), and study its impact on the transformer model's over-reliance on CoT examples for ICL.

\section{Related Work}

% in-context learning has garnered extensive attention for its surprising ability to generalize with just a few example prompts. Many investigations center on how transformers might implicitly perform gradient descent or implement other adaptation mechanisms in their hidden activations \citep{garg2022can,akyurek2024incontext,von2023transformers, bai2023transformers}. See \cite{} for surveys on the topic. However, these analyses often assume real-valued examples and very simple data distributions, leaving room to explore richer compositional structures that can align with natural language tasks.

\paragraph{Chain-of-Thought Prompting.} CoT prompting with~\citep{Wei2022} and without~\citep{Kojima2022} ICL has been an effective strategy to improve model performance in complex reasoning tasks. However, CoT prompting's effectiveness is strongly dependent on model scale and the quality of exemplars~\cite{Li2023}. 
% As noted by \citet{Li2023}, CoT reasoning reliably emerges only in models above a certain parameter threshold.
% (often cited around 50B or more).
% Smaller models often ignore the reasoning prompt or produce incoherent chains, resulting in little to no gain. 
% This has driven research into techniques for teaching smaller models to utilize CoT.
Additionally, designing good CoT exemplars for few-shot prompting can be non-trivial~\cite{min2022rethinking, zhao2021calibrate} as the exemplars need to be representative and align with the task at hand~\cite{wang2023towards, chu2024navigate}. This highlights the brittleness of the current models in utilizing the CoT exemplar\footnote{For instance: The Tulu-3 report~\cite{lambert2024tulu} describes the effectiveness of zero-shot CoT over 1-shot CoT prompts for evaluations on MATH~\cite{hendrycks2021measuring}.}. 
% Considering these limitations, \citet{Zhang2022} proposed \emph{Auto-CoT}, an approach to automatically generate diverse and effective CoT exemplars using a large model itself: they sample multiple questions and have the model answer them with CoT (using a zero-shot prompt like~\cite{Kojima2022}), then filter or select exemplars to create a few-shot prompt for a given task. While our work does not focus on exemplar selection, such techniques are complementary and could be integrated with our approach in the future. 
% This automated method often matches the performance of manual exemplars. 
Beyond the basic paradigm of CoT prompting, the `ReAct' framework by \citet{Yao2023} blends CoT prompting with actions that interface with an external environment. In this framework, the model is prompted in a format where it alternates between generating a thought (reflection on the task state) and an action (like a command to a tool or an environment). This is a departure from pure CoT, but it underscores a prompt design principle: by structuring the model's output into labeled sections (thought vs. action), one can guide it to perform complex interactive tasks. Our design of the~\coticlnew~framework by introducing `special' tokens has a similar flavor in that we isolate different parts of the output (reasoning vs. answer), which in turn allows us to modulate the mix of CoT and non-CoT examples in-context. 
% \maz{want to mention ``wait" work?}
% The success of ReAct on interactive tasks suggests that clearly delineating the model's internal reasoning from external outputs can lead to more controllable and interpretable behavior. 

\paragraph{In-Context Learning and Meta-Training.} The ICL capability of LLMs~\cite{brown2020language} to recognize and generalize to unseen patterns during inference has been extensively studied and explored in many practical and theoretical settings \cite{achiam2023gpt, team2024gemini, firooz2025360brew, dong2024survey, zhou2024mystery}. Beyond the theoretical exploration of this behaviour using single-layer transformers~\cite{oko2024pretrained,chang2025provable} and real-valued inputs~\cite{garg2022can,bai2023transformers,li2023dissecting}, a recently developed framework called~\coticl~\cite{kothapalli2025cot} leveraged a vocabulary of abstract tokens to train transformer models from scratch and shed light on the emergent ICL capabilities.
% In particular, \cite{kothapalli2025cot} studied the role of vocabulary size, model size, chain length and data embeddings for learning the underlying causal dependencies between input and chain tokens. 
The prompt design employed by~\citet{kothapalli2025cot} for the abstract reasoning tasks is similar to the meta-training approaches~\cite{min2022metaicl, chen2022meta, min2022rethinking} used in the broader literature with NLP prompts. However, as with the design of any few-shot prompt~\cite{kim2023the, biderman2024lessons}, the sequences are usually limited to either including/excluding the `chain' tokens (i.e, the thinking tokens) in \textit{all} the in-context examples. Thus, the current understanding of the proportion of CoT vs non-CoT examples needed to meta-train these models and facilitate ICL is rather limited. Our work sheds light on this underexplored aspect and demonstrates the effectiveness of such modulation when the tasks (at inference) have limited CoT examples.

\section{Preliminaries and Setup}
\label{sec:prelim_and_setup}

\paragraph{Notation.} Let $[K] = \{1, \cdots, K\}$. We consider a vocabulary $\gV$ of `abstract' tokens that is associated with a data embedding matrix $\mE_{\texttt{data}} \in \sR^{|\gV| \times d}$. Here $d$ denotes the data embedding dimension and the entries of $\mE_{\texttt{data}}$ are sampled i.i.d from $\gN(0,1)$. Let $\gG, \gH$ denote the causal structure and token processing function classes respectively. Formally, $\gG$ represents functions that select $M$ tokens from an arbitrary number of input tokens. $\gH$ represents functions that process the data embeddings (i.e, rows of $\mE_{\texttt{data}}$) of $M$ tokens and output a single token. In this setup, we are interested in learning a class of functions $f \in \gF$ that are compositional $f = f_C \circ f_{C-1} \cdots \circ f_1$ and can be formulated as $f_c = h_c \circ g_c$, where $g_c \in \gG, h_c \in \gH$. Given $N$ input tokens $\vx = (x_1, \cdots, x_N)$,  $f$ recursively generates $C$ \textit{chain tokens} $\vy = (y_1, \cdots, y_C)$ as:
\begin{align}
\begin{split}
\label{eq:comp_f_dag}
    y_c = h_c\left(g_c(\vx, y_1, \cdots, y_{c-1})\right), \forall c \in [C].
\end{split}
\end{align}
Here $\vy_{:C-1} = (y_1, \cdots, y_{C-1})$ denote the \textit{intermediate/thinking tokens}, and $y_C$ represents the final \textit{answer token}.

\paragraph{Special tokens.} We reserve a subset of tokens ($\gV_{\texttt{special}} \in \gV$) to serve as delimiters. The rest of the vocab is denoted by $\gV_{\texttt{normal}}$. These special tokens comprise of $t_{\texttt{bos}}$, $t_{\texttt{eos}}$, $t_{\texttt{pad}}$, $t_{\texttt{inp\_start}}$, $t_{\texttt{inp\_end}}$, $t_{\texttt{think\_start}}$, $t_{\texttt{think\_end}}$, $t_{\texttt{ans\_start}}$, and $t_{\texttt{ans\_end}}$. The role of these tokens are similar to the ones used in NLP. For example: $t_{\texttt{think\_start}}$, $t_{\texttt{think\_end}}$ here correspond to the \texttt{<think>, </think>} tokens in NLP (see Table~\ref{tab:notations}).

% \subsection{Multi-Input ICL and CoT with Special Tokens}

\paragraph{In-context example design.} We design a \textit{standard} \textit{example} $\ve(f)$  using tokens from $\gV_{\texttt{special}}$, the $N$ \textit{input tokens} $\vx \in \gV_{\texttt{normal}}^N$ and the \textit{answer token} $y_C \in \gV_{\texttt{normal}}$ from \eqref{eq:comp_f_dag} as:
\begin{align}
\begin{split}
\label{eq:standard_example}
\vx' &= \left( t_{\texttt{inp\_start}}, \vx, t_{\texttt{inp\_end}} \right), \\
\va' &= \left( t_{\texttt{ans\_start}}, y_{C}, t_{\texttt{ans\_end}} \right), \\
    \ve(f) &= \left( \vx', \va', t_{\texttt{eos}} \right).
\end{split}
\end{align}
Similarly, we include tokens $t_{\texttt{think\_start}}, t_{\texttt{think\_end}}$ along with the $C-1$ \textit{intermediate tokens} $\vy_{:C-1} \in \gV_{\texttt{normal}}^{C-1}$ in a \textit{standard example} to obtain a \textit{CoT example} $\ve_{\text{CoT}}(f)$ as follows:
\begin{align}
\begin{split}
\label{eq:cot_example}
\vx' &= \left( t_{\texttt{inp\_start}}, \vx, t_{\texttt{inp\_end}} \right), \\
\vt' &= \left( t_{\texttt{think\_start}}, \vy_{:C-1},  t_{\texttt{think\_end}} \right), \\
\va' &= \left( t_{\texttt{ans\_start}}, y_{C}, t_{\texttt{ans\_end}} \right), \\
    \ve_{\text{CoT}}(f) &= \left( \vx', \vt', \va', t_{\texttt{eos}} \right).
\end{split}
\end{align}

\paragraph{Tokenized sequence design.} Instead of designing a sequence with only \textit{standard} or \textit{CoT examples}, we employ a CoT probability parameter $r_{\text{CoT}}$ and diversify the sequences with a mix of both types (Figure~\ref{fig:cot_icl_lab_extended}). Formally, a sequence $\vp^K(f, r_{\text{CoT}})$ with $K$ examples is given by:
\begin{align}
\begin{split}
\label{eq:seq_design}
    \vp^K(f, r_{\text{CoT}}) &= \left(t_{\texttt{bos}}, \left( \vz^{(i)} \right)_{i=1}^K \right), \quad \\
    \text{where } 
    \vz^{(i)} &= 
    \begin{cases}
        \ve^{(i)}_{\text{CoT}}(f) & \text{if } r_{\text{CoT}} \ge u^{(i)}, \\
        \ve^{(i)}(f) & \text{otherwise}.
    \end{cases}
\end{split}
\end{align}
Here $u^{(i)} \sim \mathcal{U}(0,1)$ denotes a scalar sampled from the uniform distribution $\mathcal{U}(0,1)$ for deciding the $i^{th}$ example design. Note that choosing $r_{\text{CoT}}=1$ gives us the special case of all \textit{CoT examples} in-context, whereas $r_{CoT}=0$ gives a sequence with all \textit{standard examples}.

% \paragraph{Remark.} Here $t_{\texttt{think\_start}}$ and $t_{\texttt{ans\_start}}$ act as `trigger' tokens which conditions the transformer model (during training) to either start thinking or just provide the final answer. As a result, this design allows us to `force' a model to (1) think before answering by appending $t_{\texttt{think\_start}}$ or (2) directly provide an answer by appending the $t_{\texttt{ans\_start}}$ token to the evaluation query respectively.

\subsection{Model Training and Evaluation}

\paragraph{Models.} We follow the setup of~\citet{kothapalli2025cot} and train $3$ custom transformer models based on the Llama-3~\cite{dubey2024llama} architecture with depth $l\in\{4,8,12\}$ (Table~\ref{tab:model_card}). For notational consistency, we denote a model with $L$ layers as $\texttt{TF-}L$ throughout the paper. Appendix~\ref{app:hardware_hyperparams} presents details about the experiment settings.

\paragraph{Training objective.} We employ the Cross-Entropy ($\tCE$) based next-token prediction loss with masking for training the $\tTF$ models~\cite{kothapalli2025cot}. The $\tCE$ loss is computed only on the $4$ tokens: $t_{\texttt{ans\_start}}$, $t_{\texttt{ans\_end}}$,$y_C^{(i)}$ and $t_{\texttt{eos}}$ per \textit{standard example}. Similarly, the loss is computed only on the $C+5$ tokens per \textit{CoT example}, i.e, $t_{\texttt{think\_start}}$, $t_{\texttt{think\_end}}$, $t_{\texttt{ans\_start}}$, all the \textit{chain tokens} $\vy^{(i)}$, $t_{\texttt{ans\_end}}$, and $t_{\texttt{eos}}$.

\paragraph{Evaluation prompts.} 

Considering a test function $\tilde{f} \in \gF$ and the query input tokens $\tilde{\vx} = (\tilde{x}_1, \cdots, \tilde{x}_N) \in \gV_{\texttt{normal}}^N$, the evaluation prompt $\tilde{\vp}^{K}(\tilde{f}, r_{CoT})$ is defined as:
\begin{align}
\begin{split}
\tilde{\vx}' &= \left( t_{\texttt{inp\_start}}, \tilde{\vx}, t_{\texttt{inp\_end}}\right), \\
     \tilde{\vp}^{K}(\tilde{f}, r_{CoT}) &:= \left( \vp^{K-1}(\tilde{f}, r_{CoT}), \tilde{\vx}' \right).
\end{split}
\end{align}

% \maz{too complicated. needs simplification. Do we really need it?}
% Considering a test function $\tilde{f} \in \gF$, we first generate $K-1$ in-context examples and prepare the prefix sequence as  $\vp^{K-1}(\tilde{f},r_{\text{CoT}})$. Next, we append the special token $t_{\texttt{inp\_start}}$, the query input tokens $\tilde{\vx} = (\tilde{x}_1, \cdots, \tilde{x}_N) \in \gV_{\texttt{normal}}^N$, and the special token $t_{\texttt{inp\_end}}$. We denote the ground truth chain tokens for $\tilde{\vx}$ as $\tilde{\vy} = (\tilde{y}_1, \cdots, \tilde{y}_C) \in \gV_{\texttt{normal}}^C$, which are generated by the formulation given in \eqref{eq:comp_f_dag} using $\tilde{f} \in \gF$. The evaluation prompt $\tilde{\vp}^{K}(\tilde{f}, r_{CoT})$ is now defined as:
% \begin{align}
% \begin{split}
% \tilde{\vx}' &= \left( t_{\texttt{inp\_start}}, \tilde{\vx}, t_{\texttt{inp\_end}}\right), \\
%      \tilde{\vp}^{K}(\tilde{f}, r_{CoT}) &:= \left( \vp^{K-1}(\tilde{f}, r_{CoT}), \tilde{\vx}' \right).
% \end{split}
% \end{align}

% \begin{align}
% \begin{split}
%     \hat{y}_{pred} &:= \tTF\left( \vp^{K-1}(\tilde{f}), \tilde{\vx} \right) \hspace{10pt} \textit{ w/o CoT,} \\
%     \hat{y}_{pred} &:= \tTF^{\circ C}\left( \vp_{CoT}^{K-1}(\tilde{f}), \tilde{\vx} \right) \hspace{10pt} \textit{ w/ CoT}. 
% \end{split}
% \end{align}
\paragraph{Forcing strategies.} Incorporating special tokens in the prompt design allows us to measure the performance of the trained \texttt{TF} models using the following $3$ approaches: (1) `Force Think', (2) `Force Answer' and (3) `No Forcing'. Consider $C_{eos}$ as the number of tokens generated by the model, and $\tTF(\cdot)$ as greedy auto-regressive single token generation.
% Based on the above-mentioned forcing strategies, we tabulate the output generation formulations in Table~\ref{tab:eval_output_variants}.
The recursive formulation $\hat{y_o}, \forall o \le C_{eos}$ with the `Force Think' strategy is as follows:
\begin{align}
\label{eq:force_strategy}
    \tTF\left( \underbrace{\tilde{\vp}^{K}(\tilde{f}, r_{CoT})}_{\textit{evaluation seq}},\underbrace{t_{\texttt{think\_start}}}_{\textit{force token}}, \underbrace{\hat{y}_1, \cdots, \hat{y}_{o-1}}_{\textit{previous step outputs}} \right).
\end{align}
By replacing $t_{\texttt{think\_start}}$ with  $t_{\texttt{ans\_start}}$ in \eqref{eq:force_strategy}, we condition the model to directly provide the final answer, whereas `No Forcing' completely removes the \textit{force token} suffix and allows the model to choose between thinking and no thinking modes.

\paragraph{Measuring \tacc.} We denote the ground truth chain tokens for $\tilde{\vx}$ in the evaluation prompt as $\tilde{\vy} = (\tilde{y}_1, \cdots, \tilde{y}_C) \in \gV_{\texttt{normal}}^C$, which are generated by the formulation given in \eqref{eq:comp_f_dag} using $\tilde{f} \in \gF$. Considering the model's output sequence as: $\hat{\mathbf{y}} = ( \hat{y}_1, \hat{y}_2, \dots, \hat{y}_{C_{eos}})$, we search for the following pattern: $(t_{\texttt{ans\_start}}, \hat{y}_{k}, t_{\texttt{ans\_end}})$ in $\hat{\mathbf{y}}$ and treat $\hat{y}_{k}$ as the predicted answer $\hat{y}_{pred}$. If the pattern does not exist then we set $\sI_{\hat{y}_{pred}=\tilde{y}_C} = 0$ since the model failed to follow the output format\footnote{This is similar to measuring the instruction following capability of LLMs where the output should be formatted in a certain fashion based on the in-context examples.}. Given $\widetilde{T}$ evaluation sequences, we measure the overall $\tacc$ as $\frac{1}{\widetilde{T}} \sum_{t=1}^{\widetilde{T}} \sI_{\hat{y}_{pred}=\tilde{y}_C}$.
% \begin{align}
% \begin{split}
% \hat{\mathbf{y}} &= \left( \hat{y}_1, \hat{y}_2, \dots, \hat{y}_{C_{eos}} \right) \\
% \hat{s} &= \max \left\{ i \mid \hat{y}_i = t_{\texttt{ans\_start}} \right\}\text{\maz{why max?}} \\
% \hat{e} &= \max \left\{ j > \hat{s} \mid \hat{y}_j = t_{\texttt{ans\_end}} \right\}.
% \end{split}
% \end{align}

% Then the token $\hat{y}_{\hat{s} + 1}$ is considered as the predicted answer $\hat{y}_{\text{pred}}$ if $\hat{e} == \hat{s} + 2$. Else, it is treated as None or Invalid and we assign $\sI_{\hat{y}_{pred}=\tilde{y}_C} = 0$ since the model failed to follow the output format\footnote{This is similar to measuring the instruction following capability of LLMs where the output should be formatted in a certain fashion based on the in-context examples.}. Given $\widetilde{T}$ evaluation sequences, we measure the overall $\tacc$ as $\frac{1}{\widetilde{T}} \sum_{t=1}^{\widetilde{T}} \sI_{\hat{y}_{pred}=\tilde{y}_C}$. \maz{can we simplify with words instead of math?}

% \begin{table*}[t]
%     \centering
%     \begin{tabular}{|c|c|c|c|}
%     \hline
%       \textbf{Strat} & Force Think  & Force Answer \\
%       \hline
%        \textbf{Pred} & $\tTF^{\circ C_{eos}}\left( \tilde{\vp}^{K}(\tilde{f},r_{CoT}), t_{\texttt{think\_start}}  \right)$  & $\tTF^{\circ C_{eos}}\left( \tilde{\vp}^{K}(\tilde{f}, r_{CoT}), t_{\texttt{ans\_start}}  \right)$ \\
%          \hline
%     \end{tabular}
%     \caption{Variations in model (eval) inputs based on the `forcing' strategies using special tokens.}
%     \label{tab:eval_output_variants}
% \end{table*}

\subsection{Choices of $\gG, \gH$ in~\coticl}

The~\coticl~framework is designed with the above formalization and employs the following function classes $\gG, \gH$ to generate chain tokens \eqref{eq:comp_f_dag}.

% \paragraph{Data embeddings for abstract tokens.} A vocabulary $\gV$ of abstract tokens (i.e not grounded in any semantic information) is associated with a data embedding matrix $\mE_{\texttt{data}} \in \sR^{|\gV| \times d}$, whose entries are sampled from $\gN(0,1)$.

\paragraph{DAG representation of causal dependencies.}  $\gG$ is a class of topologically sorted DAGs whose structure is determined by the choice of $(N,M,C)$. See Figure~\ref{fig:cot_icl_lab_extended}, which represents a DAG sampled from $\gG(N=3, M=2, C=3)$.

\paragraph{Process data embeddings with MLPs.} $\gH$ is a class of MLPs of depth $l$ with the activation $\phi$; denoted as $\gH(l, \phi)$\footnote{We choose $l=1$, $\phi=\texttt{LeakyReLU}$ based on the token distribution analysis of sequences by~\citet{kothapalli2025cot}.}.
% Since the weights of these MLPs are randomly initialized, the cardinality of $\gH(1,\texttt{LeakyReLU})$ tends to be unbounded. Although such a complex setting was considered by \citet{kothapalli2025cot}, we focus on scenarios where the cardinality is finite yet sufficient for creating challenging datasets for the $\tTF$ models.
% Although~\cite{kothapalli2025cot} sampled $h \sim \gH(l, \phi)$ for every sequence, 
We maintain a \texttt{TokenProcessorCache} of finite MLPs with random weights and sample from it accordingly. For a given value of $C$, we sample $C$ MLPs from this cache (one for each chain token) and use them to generate the chain tokens of all $K$ examples within that sequence.

\paragraph{Remark.} As the DAG and MLPs are unique to every sequence, one can intuitively think of meta-training as teaching the model to figure out the underlying DAG and approximate the MLP transformations solely from the in-context examples.

% \paragraph{Process data embeddings with MLPs.} $\gH$ is a class of MLPs of depth $1$ with the \texttt{LeakyReLU} activation; denoted as $\gH(1,\texttt{LeakyReLU})$\footnote{NOTE: We do not consider MLPs of varying depth and activations since \citet{kothapalli2025cot} has shown that the choice of $l=1$, $\phi=\texttt{LeakyReLU}$ tends to provide a rich enough diversity in the token distribution of sequences. }.
% % Since the weights of these MLPs are randomly initialized, the cardinality of $\gH(1,\texttt{LeakyReLU})$ tends to be unbounded. Although such a complex setting was considered by \citet{kothapalli2025cot}, we focus on scenarios where the cardinality is finite yet sufficient for creating challenging datasets for the $\tTF$ models.
% We maintain a \texttt{TokenProcessorCache} of a finite count of MLPs with random weights and sample from it accordingly. For a given value of $C$, we sample $C$ MLPs from this cache (one for each chain token) and use them to generate the chain tokens of all $K$ examples within that sequence.

% \section{Data Generation}
\section{The \coticlnew~}

Following the formalization in the above section and the design choices for $\gG, \gH$ as per~\coticl, we (1) incorporate special tokens, (2) diversify the sequences in a dataset by randomly choosing $N,M,C$ per sequence from a list of choices, and (3) modulating the mix of \textit{CoT/standard examples} based on \texttt{CoT-Recipes}. The algorithms to generate the entire dataset based on these techniques are formalized in Appendix~\ref{app:datagen_algo}.

% For the reader's convenience, we revisit these key design aspects introduced in \citet{kothapalli2025cot}, and follow up with the details of our $3$ extensions.

% \subsection{Diversifying Synthetic Datasets}

% \begin{figure}
%     \centering
%     \includegraphics[width=\linewidth]{images/cot_icl_lab_2.png}
%     \caption{The \coticlnew~design. (1) We incorporate special tokens (marked in gray) to act as delimiter tokens between input, intermediate/thinking, and answer tokens. (2) Each individual sequence is diversified based on different choices of $N,M,C$ ($K$ is omitted in the figure). (3) We drop the intermediate tokens per in-context example based on $r_{CoT}$. The choice of $r_{CoT}$ is varied across sequences as per the \texttt{CoT-Recipe}.  }
%     \label{fig:cot_icl_lab_extended}
%     \vspace{-4mm}
% \end{figure}

% Based on the existing design of \coticl~described above, we present an overview of our key extensions in Figure~\ref{fig:cot_icl_lab_extended} and discuss them in detail in the following sections. 

\subsection{Ext 1: Special Tokens}

We extend the in-context examples with special tokens as formulated by \eqref{eq:standard_example} for a \textit{standard example} and \eqref{eq:cot_example} for a \textit{CoT example}. Without delimiters, a \textit{CoT example} is formulated as $\ve_{CoT}(f) = \left(\vx, \vy\right)$, and a \textit{standard example} as $\ve(f) = \left(\vx, y_C\right)$, where the loss is computed only on $\vy$ or $y_C$ respectively. Such a design has the following limitation that if one were to mix \textit{CoT} and \textit{standard examples} in a single sequence, the model cannot differentiate between the first intermediate token of a \textit{CoT example} from the answer token of a \textit{standard example}.
% \begin{enumerate}
    % \item Any single tokenized sequence is restricted to employ either $K$ \textit{CoT} or $K$ \textit{standard examples} in-context. If one were to mix \textit{CoT} and \textit{standard examples} with this design, the model cannot differentiate between the first intermediate token of a \textit{CoT example} from the answer token of a \textit{standard example}.
    % \item Due to a lack of delimiter tokens, we cannot condition the model to forcefully think or directly provide the final answer during inference. Thus, preventing us from gaining richer insights on the proportion of \textit{CoT examples} needed to facilitate ICL.
% \end{enumerate}

% To better understand the sequence design with special tokens, notice the $6$ extra tokens that are added to the tokenized sequences in Figure~\ref{fig:cot_icl_lab_extended}. With such explicit demarcation of input, intermediate and answer tokens, we can train the model to differentiate between the different phases of a response and use the forcing strategies to elicit the thinking process on-demand during inference.

\subsection{Ext 2: Diversification with $N,M,C$}

To enhance the diversity of the dataset comprising $T$ sequences, we introduce variability through randomized sampling of the parameters $N$, $M$, $C$. Specifically, we define discrete sets of available choices: $\mN$, $\mM$, and $\mC$ for each sequence, and sample one value from each set: $N \sim \mN$, $M \sim \mM$, $C \sim \mC$. These values are used to construct a DAG using $\gG(N, M, C)$ for all the $K$ examples in the sequence. 
% A visualization of such diversity is presented in Figure~\ref{fig:cot_icl_lab_extended}, where ``Sequence 1'' is generated using $N=2,M=1,C=4$, whereas ``Sequence 2'' employs $N=3,M=1,C=5$.
Such a sequence design is not possible in the older~\coticl~design as the same $(N,M,C)$ are used for all $T$ sequences.
% In this case, one can consider $\mN=\{2,3\}, \mM=\{1\}, \mC=\{4,5\}$, based on which the two different $N,M,C$ tuples were sampled. We omitted the illustration with randomly sampled $K$ for brevity. On the contrary, observe that the older design of \coticl~relies on a fixed tuple of $N=2,M=1,C=4$ for all sequences.

\subsection{Ext 3: Diversification with \texttt{CoT-Recipe}}

% \begin{figure}
%     \centering
%          \includegraphics[width=\linewidth]{images/power_law_recipes.jpg}
%     \caption{Systematic control over the probability of \textit{CoT examples} in few-shot training prompts. Here $r_{CoT}$, determined by the \texttt{CoT-Recipe} parameter $\alpha$, represents the probability that an in-context example is a \textit{CoT example} instead of a non-CoT or \textit{standard example}. }
%     \label{fig:power_law_recipes}
%     \vspace{-3mm}
% \end{figure}

By leveraging the richer sequence design of \eqref{eq:seq_design}, we define \texttt{CoT-Recipe}, a parameterized approach to systematically assign $r_{CoT}^{(j)}$ based on the sequence index $j \in [0, T{-}1] $. This formulation allows us to control the expected proportion of \textit{CoT} versus \textit{standard examples} in context. Formally, a \texttt{CoT-Recipe} is a partial power-law function with offset as:
\begin{align}
\label{eq:cot_recipe}
\begin{split}
    \texttt{CoT-Recipe}(\alpha, a, b)(u) &= a \cdot u^{\alpha} + b,
\end{split}
\end{align}
where \( \alpha \in \mathbb{R}_{\ge 0} \) governs the shape (e.g., linear, sublinear), while \( a, b \in \mathbb{R} \) scale and shift the curve respectively (see the illustration in Figure~\ref{fig:cot_icl_lab_extended} and Appendix~\ref{app:expected_token_count_alpha} for calculations on the expected token counts). Using this formalization, we set $r_{CoT}^{(j)}$ for a sequence/prompt with index $j \in [0, T{-}1] $ as $ r_{CoT}^{(j)} = \texttt{CoT-Recipe}(\alpha, a, b)\left( j/T \right)$, which gives:
\begin{align}
\begin{split}
    r_{CoT}^{(j)} &= a \cdot \left( j/T \right)^{\alpha} + b.
\end{split}
\end{align}
\paragraph{Design Rationale.} Although $r_{\text{CoT}}^{(j)}$ is deterministically computed based on sequence index  $j \in [0, T-1]$, we note that the sequences are randomly shuffled prior to training. The formulation in \eqref{eq:cot_recipe} allows us to have a clear mental model for designing the \texttt{CoT-Recipe} partial function, while explicitly avoiding a curriculum-like effect during training.

\section{To Think or Not To Think?}
\label{sec:role_of_cot_recipe}

In this section, we systematically meta-train the models with varying \texttt{CoT-Recipe} parameter $\alpha$, and evaluate them on datasets with varying fractions of \textit{CoT examples}. In essence, we aim to understand which choices of $\alpha$ can lead to effective meta-training and allow the models to solve novel tasks even with limited \textit{CoT examples}.

\paragraph{Meta-training Setup.} We choose $|
\gV|=1024$, $d=10, \mN=\mM=\mC=\{4\}, K=40$ and vary $\alpha \in \{0, 0.5, 1, 2, \infty\}$ with $a=1,b=0$, for creating the training ($T=64 \times 10^5$) datasets.

\paragraph{Evaluation Setup.} We create a separate evaluation dataset $\widetilde{\gD}$ with $|\widetilde{\gD}| = \widetilde{T}=10^4$ sequences using $\widetilde{\mN}=\widetilde{\mM}=\widetilde{\mC}=\{4\}$, $\widetilde{K}=40$ and $r_{CoT}^{(j)} = 1, \forall j \in [0, \widetilde{T}-1]$. Since $r_{CoT}^{(j)} = 1$, all the evaluation sequences contain $K-1$ \textit{CoT examples} along with the query input tokens. Next, we transform $\widetilde{\gD}$ by randomly choosing $K'$ \textit{CoT examples} (i.i.d) per sequence and dropping the intermediate tokens along with the special tokens $t_{\texttt{think\_start}}$ and $t_{\texttt{think\_end}}$. Thus converting them into $K'$ \textit{standard examples}. 

\paragraph{Limiting CoT supervision by increasing $K'$.} By varying $K'\in\{0, 10, 20, 30, 39\}$ and applying different forcing strategies, we evaluate the $\tTF$ model's performance when CoT supervision is limited, i.e as $K'$ increases, the \textit{standard examples} outnumber the \textit{CoT examples}.
% We denote the resulting dataset as $\widetilde{D}^{\downarrow K'}. By varying $K'\in\{0, 10, 20, 30, 39\}$ and applying different forcing strategies, we evaluate the $\tTF$ model's reliance on \textit{CoT/standard examples}.

\begin{figure*}[t!]
    \centering
    % First subfigure
    \begin{subfigure}[b]{0.32\textwidth}
        \centering
        \includegraphics[width=\textwidth]{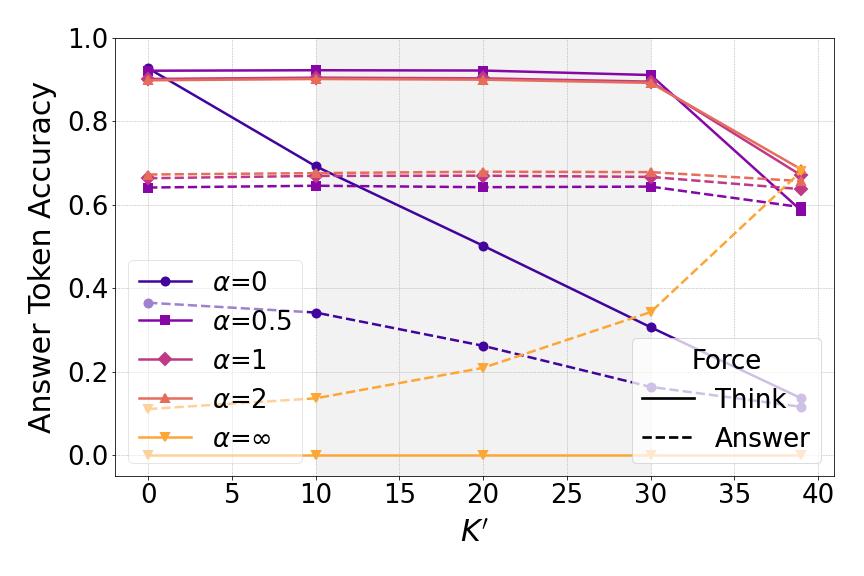}
        \caption{\texttt{TF-4}}
        \label{fig:bulk_eval_cot_recipes_N_4_M_4_C_4_L_4}
    \end{subfigure}
    \hfill
    % Second subfigure
    \begin{subfigure}[b]{0.32\textwidth}
        \centering
        \includegraphics[width=\textwidth]{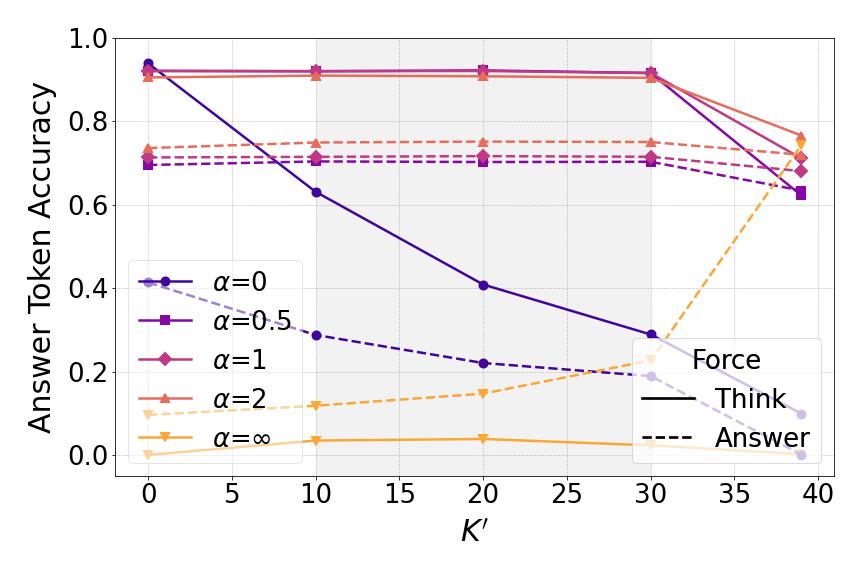}
        \caption{\texttt{TF-8}}
        \label{fig:bulk_eval_cot_recipes_N_4_M_4_C_4_L_8}
    \end{subfigure}
    \hfill
    % \vskip\baselineskip
    % Third subfigure
    \begin{subfigure}[b]{0.32\textwidth}
        \centering
        \includegraphics[width=\textwidth]{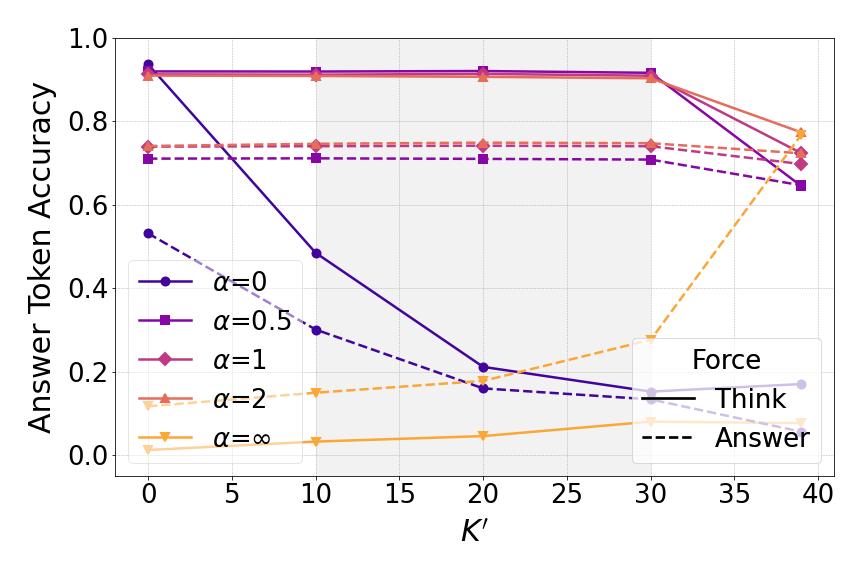}
        \caption{\texttt{TF-12}}
        \label{fig:bulk_eval_cot_recipes_N_4_M_4_C_4_L12}
    \end{subfigure}
    \caption{$\tacc$ of models trained with varying $\alpha$, $\mN=\mM=\mC=\{4\}$ and evaluated on datasets with  $\widetilde{\mN}=\widetilde{\mM}=\widetilde{\mC}=\{4\}$. Here $K'$ indicates the number of \textit{standard examples} in test prompts. }
    \label{fig:bulk_eval_cot_recipes_N_4_M_4_C_4}
    \vspace{-4mm}
\end{figure*}

\subsection{Over-Reliance on \textit{CoT/standard Examples}}
\label{subsec:diversity_overreliance}

\paragraph{Lack of modulation during meta-training leads to over-reliance on \textit{CoT/standard examples}.} As the $\tTF$ models trained using $\alpha=0$ have never encountered \textit{standard examples} in a training sequence, they tend to rely on the \textit{CoT examples} for generating the outputs. Observe from Figure~\ref{fig:bulk_eval_cot_recipes_N_4_M_4_C_4} that such a recipe leads to a gradual reduction of $\tacc$ across all model sizes and forcing strategies as $K'$ increases. In particular, the (largest) \texttt{TF-12} model evaluated with the `Force Think' strategy exhibits a reduction in $\tacc$ from $\approx 0.93$ (when $K'=0$) to $\approx 0.18$ (when $K'=39$). On the contrary, models trained using $\alpha=\infty$ are incapable of producing the intermediate chain tokens. Thus, the `Force Think' strategy leads to $\approx 0$ $\tacc$ across all models. However, the `Force Answer' strategy results in a gradual increase in $\tacc$ as $K'$ increases and the fraction of \textit{CoT examples} reduces. Thus indicating an over-reliance on \textit{standard examples}. Such behavior is not desirable since we want the model to generalize the thinking process.
% \maz{We should mention that this is not a good generalization behavior since we want the model to be able to generalize its thinking strategy.}

% Nonetheless, notice that for any model size, the $\tacc$ with `Force Think' strategy and $\alpha=0, K'=0$ is always higher than `Force Answer' strategy with $\alpha=\infty, K'=39$, with the gap being relatively larger in the smaller \texttt{TF-4} model.

\paragraph{Careful selection of $\alpha$ facilitates thinking even without \textit{CoT examples}.} Training datasets created with $\alpha=\{0.5, 1, 2\}$ modulate the proportion of \textit{CoT examples} in the sequences and thus prevent the model's over-reliance as seen with $\alpha=\{0, \infty\}$. Especially, observe from Figure~\ref{fig:bulk_eval_cot_recipes_N_4_M_4_C_4} that the `Force Think' strategy can be applied to all models for maintaining a high $\tacc$ even as the proportion of \textit{standard examples} increases (i.e, $K'$ increases). Thus, avoiding the gradual collapse as observed with $\alpha=0$. Surprisingly, notice that the models trained using $\alpha=2$ can leverage the `Force Think' strategy even when there are no \textit{CoT examples} in-context (i.e, $K'=39$), and perform on-par with the counterparts that were trained using $\alpha=\infty$ and evaluated with the `Force Answer' strategy. This observation signifies that with careful selection of the recipes, one can train models to be good at thinking even when there are no \textit{CoT examples} available in-context. In particular, the $\tacc$ improvements over $\alpha=0$ tend to be larger as $K'$ increases and can be greater than $300\%$ with \texttt{TF-12} for $K'=39$ i.e, a increase from $\approx 0.18$ with $\alpha=0$ to $\approx 0.78$ with $\alpha=2$.
% \maz{we should highlight this more. even mention it in the abstract and contributions.}
% Similarly, when we apply the `Force Answer' strategy with all \textit{CoT examples} (i.e, $K'=0$), the models do not exhibit an over-reliance on the \textit{standard examples} and significantly outperform the counterparts trained using $\alpha=\infty$. .

\section{Task Diversity and Length Generalization}
\label{subsec:diversity_len_gen}
In the previous section, we have seen that the \texttt{CoT-Recipe} formalization allows us to modulate the mix of CoT examples and facilitate reasoning control in the models. As a next step, we show that increasing the diversity of training data via $\mN,\mM,\mC$ can aid the models to generalize to out-of-distribution (OOD) settings with longer inputs.

% This showcases that directly answering the query can be as effective as elaborate reasoning in OOD settings.

\begin{figure}
    \centering
        \includegraphics[width=\linewidth]{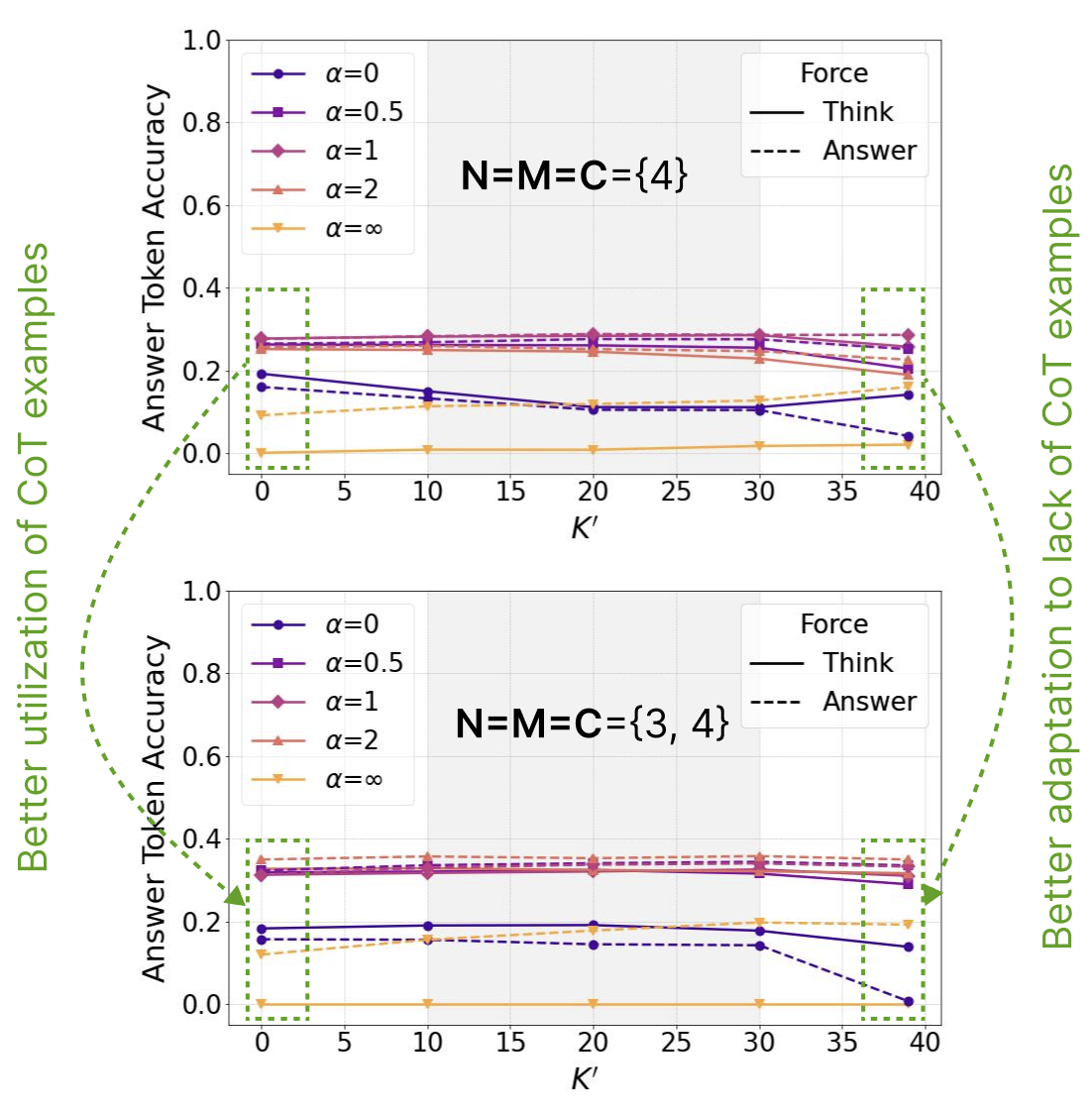}
% \label{fig:cot_recipes_train_N_4_M_4_C_4_K_40_bulk_eval_N_5_M_4_C_4_K_40_L12}
    \caption{Input length generalization of \texttt{TF-12} models when tested with $\widetilde{\mN}=\{5\}, \widetilde{\mM}=\widetilde{\mC}=\{4\}$.}
\label{fig:cot_recipes_len_gen_bulk_eval_N_5_M_4_C_4_K_40_L_12}
    \vspace{-4mm}
\end{figure}

\paragraph{Setup.} We train additional \texttt{TF} models on new diverse datasets created using $\mN=\mM=\mC=\{3, 4\}, K=40$ (with rest of the parameters for \texttt{CoT-Recipe}, $\gG, \gH$ being the same as Section~\ref{sec:role_of_cot_recipe}). For evaluation, we consider $\widetilde{\mN}=\{5\}, \widetilde{\mM}=\widetilde{\mC}=\{4\}, \widetilde{K}=40$ and employ the same process as Section~\ref{sec:role_of_cot_recipe} to create the evaluation datasets.

\paragraph{Diversity with $\mN, \mM, \mC$ improves length generalization.} Figure~\ref{fig:cot_recipes_len_gen_bulk_eval_N_5_M_4_C_4_K_40_L_12} illustrates that the \texttt{TF-12} models leverage the diversity of $\mN=\mM=\mC=\{3, 4\}$ and attain a peak evaluation $\tacc$ of $\approx 0.38$, when compared to $\approx 0.28$ with models trained using $\mN=\mM=\mC=\{4\}$. Furthermore, models trained with $\alpha=\{0.5, 1, 2\}$ consistently outperform the $\alpha=\{0, \infty\}$ cases across all sizes (see Figure~\ref{fig:cot_recipes_len_gen_bulk_eval_N_5_M_4_C_4_K_40_L_4} for \texttt{TF-4} and Figure~\ref{fig:cot_recipes_len_gen_bulk_eval_N_5_M_4_C_4_K_40_L_8} for \texttt{TF-8}). More importantly, the `Force Think' strategy with $\alpha=0$ models on $K'=0$ is not as effective as $\alpha=\{0.5, 1, 2\}$, unlike the in-domain generalization setting (see Figure~\ref{fig:bulk_eval_cot_recipes_N_4_M_4_C_4}). 

\paragraph{Forcing strategies and OOD tasks.} Surprisingly, Figure~\ref{fig:cot_recipes_len_gen_bulk_eval_N_5_M_4_C_4_K_40_L_12} also shows that for any $K'$, the peak $\tacc$ with `Force Answer' strategy across all $\alpha$ is comparable and sometimes even higher than the `Force Think' strategy. This is a failure mode where the model is unable to generate the required number of thinking steps ($5$ in this case) even with the `Force Think' strategy and fails to arrive at the right answers.
% \maz{Can we expand this? Try to show why this is happening? Is the model misunderstanding how many intermediate tokens are needed? or something else?}
In terms of recipes, $\alpha=2$ tends to be the best choice for \texttt{TF-12}. While we observe that the gap in $\tacc$ for both the strategies is quite narrow across model sizes and $K'$ in Figure~\ref{fig:cot_recipes_len_gen_bulk_eval_N_5_M_4_C_4_K_40_L_12}, Figure~\ref{fig:cot_recipes_len_gen_bulk_eval_N_5_M_4_C_4_K_40_L_4}, and Figure~\ref{fig:cot_recipes_len_gen_bulk_eval_N_5_M_4_C_4_K_40_L_8}, the gap was observed to be relatively wider for the in-domain generalization setting (see Figure~\ref{fig:bulk_eval_cot_recipes_N_4_M_4_C_4}). Thus, it sheds light on the limitations of enforcing thinking behavior and the brittleness of OOD generalization.
% \maz{Do we see something similar in the symbolic dataset in the later sections? If so, we should highlight this.}

% \paragraph{A note on evaluating reasoning in LLMs.} Our experiments surface an overlooked aspect in evaluating the effectiveness of reasoning in LLMs. As shown above, the `Force Answer' strategy can be more effective than the `Force Think' strategy for OOD tasks when the models are trained with diverse tokenized sequences. Since the open-source pre-trained LLMs generally satisfy the diverse training criteria, we conjecture that the ID/OOD nature of tasks play a major role in determining if lengthy CoT reasoning is desirable or not~\cite{hassid2025don}.

\begin{figure*}[h]
    \centering
    \includegraphics[width=\linewidth]{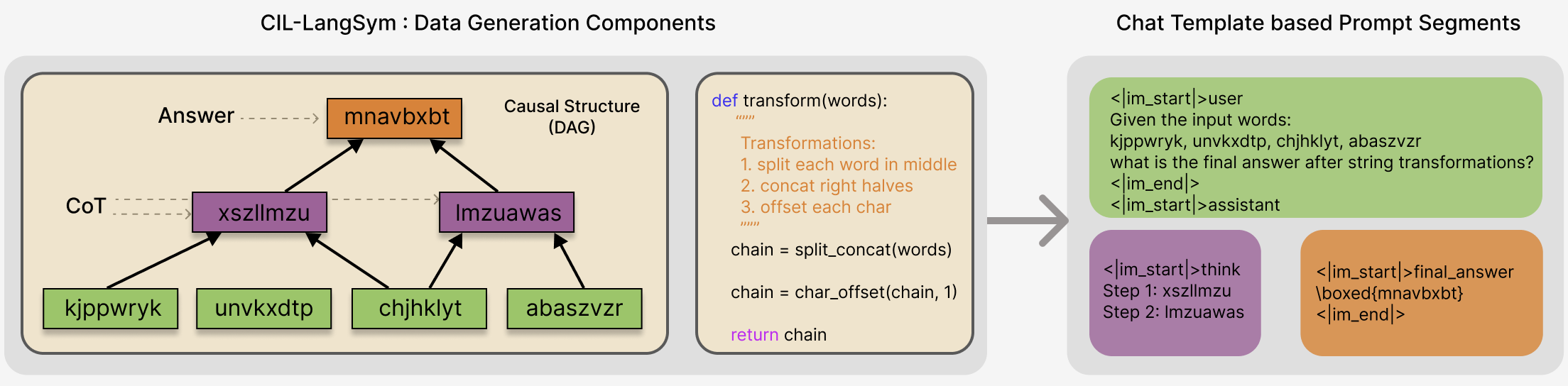}
    \caption{Chat template of a \textit{CoT/standard example} in \texttt{CIL-LangSym} based on the \texttt{Qwen-2.5-1.5B-Instruct} tokenizer. Given $N=4, M=2, C=3$ and word length $W=8$, the DAG determines the ground truth causal dependencies, and the \texttt{transform} function illustrates the string processing of the $M$ parent words. We apply the chat template to differentiate the question, thinking, and final answer segments of the examples and also ensure that the task description does not reveal the underlying string transformation in natural language. }
    \label{fig:cil_symbolic}
\end{figure*}

\section{Symbolic Reasoning with LLMs}

Our analysis above highlighted the importance of \texttt{CoT-Recipe} for reasoning with abstract tokens. To verify if these insights can be transferred to pretrained LLMs, we leverage the design patterns of~\coticlnew~to create a fully interpretable symbolic reasoning dataset called \texttt{CIL-LangSym}. In particular, we aim to understand (1) if \texttt{CoT-Recipe} parameters obtained in the previous sections can still be effective when each intermediate step and answer spans multiple tokens and (2) how forcing strategies affect length generalization when the model is pre-trained on a much larger and diverse natural language data.

% Unlike existing work on symbolic reasoning~\cite{boix-adsera2024when}, \texttt{CIL-LangSym} leverages the design patterns of~\coticlnew~to generate complex, yet fully interpretable intermediate reasoning steps for each in-context example.

\begin{figure*}[t!]
    \centering
    \begin{subfigure}[b]{0.32\textwidth}
        \centering
        \includegraphics[width=\textwidth]{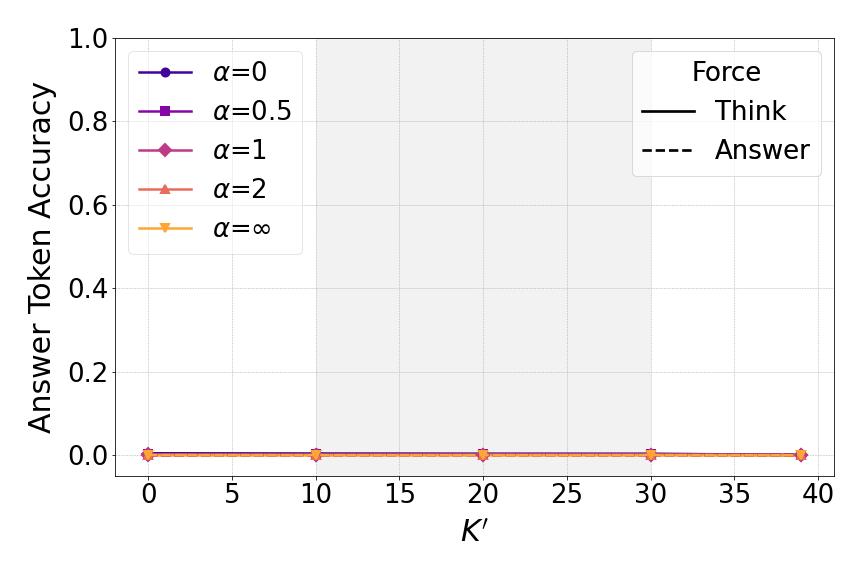}
        \caption{\texttt{Qwen-2.5-0.5B-Instruct}}
        \label{fig:cil_langsym_qwen_0.5B_instruct_cot_recipes_N_4_M_2_C_3}
    \end{subfigure}
    \hfill
    \begin{subfigure}[b]{0.32\textwidth}
        \centering
        \includegraphics[width=\textwidth]{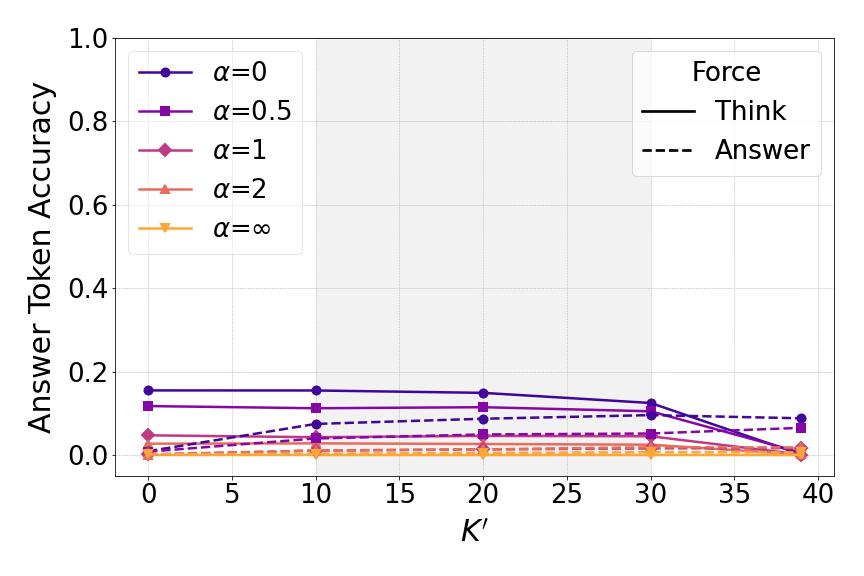}
        \caption{\texttt{Qwen-2.5-1.5B-Instruct}}
        \label{fig:cil_langsym_qwen_1.5B_instruct_cot_recipes_N_4_M_2_C_3}
    \end{subfigure}
    \hfill
    \begin{subfigure}[b]{0.32\textwidth}
        \centering
        \includegraphics[width=\textwidth]{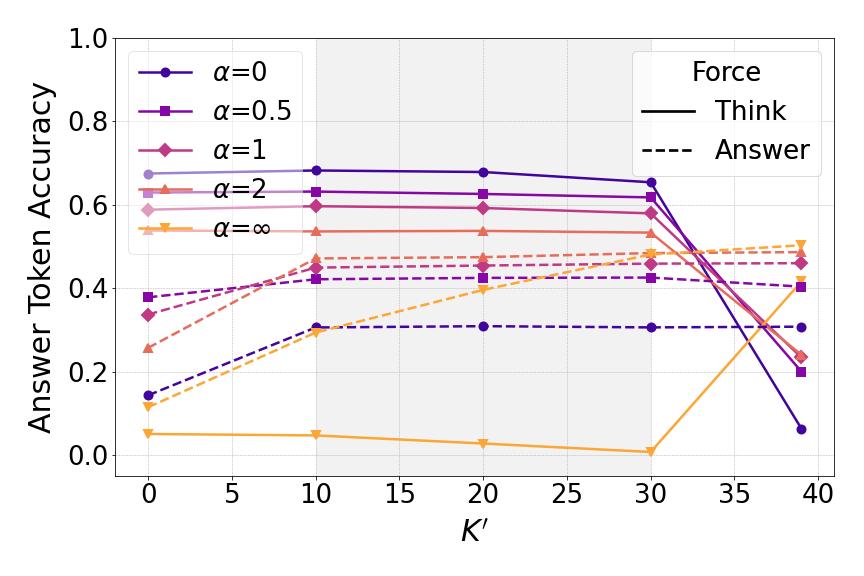}
        \caption{\texttt{Qwen-2.5-7B-Instruct}}
        \label{fig:cil_langsym_qwen_7B_instruct_cot_recipes_N_4_M_2_C_3}
    \end{subfigure}
    \caption{$\tacc$ of models trained with varying $\alpha$, $\mN=\{4\}, \mM=\{2\}, \mC=\{3\}$ and evaluated on datasets with $\widetilde{\mN}=\{4\}, \widetilde{\mM}=\{2\}, \widetilde{\mC}=\{3\}$. }
    \label{fig:cil_langsym_qwen_cot_recipes_N_4_M_2_C_3}
    \vspace{-4mm}
\end{figure*}

\subsection{Data Generation}

% \subsubsection{Design Patterns Based on~\coticlnew}
\paragraph{Random words with ASCII lowercase.} We adhere to the setup in Section~\ref{sec:prelim_and_setup} and generate $N$ \textit{input words} per in-context example, each comprising of $W$ ASCII lowercase characters (a-z). In essence, we transition from $N$ abstract tokens in~\coticlnew~to $N$ words per in-context example.

\paragraph{Causal structure via DAGs.} Similar to~\coticlnew~, we consider the function class $\gG$ of topologically sorted DAGs to implant the causal dependencies between words.

\paragraph{String processing function.} Unlike the token processing function class $\gH$ in~\coticlnew~that relied on the data embeddings $\mE_{\texttt{data}}$ and MLPs to generate the abstract chain tokens, we employ a string processing function $s$ based on string slicing and character offset operations for \texttt{CIL-LangSym}. Formally, the function $s$ takes the $M$ filtered words (from $\gG$) as inputs and outputs a single \textit{intermediate word} (Figure~\ref{fig:cil_symbolic}). 
% The mathematical formulation is exactly the same as~\eqref{eq:comp_f_dag} with $h_c \in \gH, \forall c \in [C]$ replaced by $s$ and abstract tokens replaced by random words 
See Appendix~\ref{app:cil_langsym_details} for the algorithmic description of these design aspects and Figure~\ref{fig:cil_langsym_example} for an example prompt.

\paragraph{Remark.} Notice that the task description does not reveal the underlying string transformation, and the model is required to figure out the task solely from the in-context \textit{CoT/standard examples}.
% \maz{what additional information does this experiment give us?}
% \maz{are intermediate stuff and final answer guranteed to be one token?}

\subsection{Experiments}
\label{subsec:cil_langsym_exps}

\paragraph{Setup.} The \texttt{CIL-LangSym} dataset is created with the following parameters: $\mN=\widetilde{\mN}=\{4\}, \mM=\widetilde{\mM}=\{2\}, \mC=\widetilde{\mC}=\{3\}, K=\widetilde{K}=40$ for training ($T=1000$), and evaluation ($\widetilde{T}=10,000$). We consider the \texttt{Qwen-2.5} series models to highlight the difficulty of the tasks and to study the role of \texttt{CoT-Recipes} and forcing strategies. We allocate a budget of $1000$ and $100$ tokens for the `Force Think' and `Force Answer' strategies respectively.

\paragraph{Insufficiency of pre-trained knowledge.} We evaluate $8$ \texttt{Qwen-2.5} series models~\cite{Yang2024Qwen25TR}: \texttt{0.5B-Instruct}, \texttt{1.5B}, \texttt{1.5B-Instruct}, \texttt{Math-1.5B-Instruct}, \texttt{7B}, \texttt{7B-Instruct} along with the DeepSeek distilled reasoning models: \texttt{DeepSeek-R1-Distill-Qwen-1.5B, DeepSeek-R1-Distill-Qwen-7B} on evaluation datasets for $K'=\{0,10,20,30,39\}$ and \texttt{temperature=0.6}. The evaluation datasets are prepared using the same approach as Section~\ref{sec:role_of_cot_recipe}. By relying solely on CoT-ICL prompting, we noticed that all the $8$ models scored an $\tacc$ of $0$ across all $K'$ and both the forcing strategies. This baseline establishes that the models cannot leverage their pre-trained knowledge for this symbolic reasoning task and ensures that the performance improvements are largely attributed to meta-training based on the \texttt{CoT-Recipe}\footnote{Results were also consistent across other temperature values ranging from $0$ to $1$ with $0.1$ increments.}.

\paragraph{\texttt{CoT-Recipe} facilitates reasoning control even in pretrained LLMs.} We use \texttt{CoT-Recipe} with $\alpha=\{0, 0.5,1,2,\infty\}$ for SFT (meta-training) of \texttt{0.5B-Instruct, 1.5B-Instruct}, and \texttt{7B-Instruct} models to focus solely on the model size. Given that we use only $1000$ prompts, Figure~\ref{fig:cil_langsym_qwen_cot_recipes_N_4_M_2_C_3} shows a clear benefit of model size as the peak $\tacc$ attained by \texttt{0.5B-Instruct} is only $0.004$ whereas \texttt{1.5B-Instruct, 7B-Instruct} can reach up-to $0.15$, and $0.67$ respectively. Given a large enough model such as \texttt{7B-Instruct}, notice from Figure~\ref{fig:cil_langsym_qwen_7B_instruct_cot_recipes_N_4_M_2_C_3} that as long as there are CoT examples available in-context, the \texttt{CoT-Recipes} with $\alpha \neq \infty$ facilitate the model to leverage the `Force Think' strategy. Similar to the observations with~\coticlnew, these models showcase an over-reliance on the \textit{CoT examples} and exhibit a reversal in trend of $\tacc$ as $\alpha$ increases from $0$ to $\infty$ for $K'=39$. In these scenarios with limited CoT supervision, if one were to employ the `Force Think' strategy, then choosing $\alpha=2$ can result in $\tacc$ improvements of over $130\%$ when $K'=39$ (consistent with the observations from~\coticlnew~in Section~\ref{sec:role_of_cot_recipe}). We provide guidance on choosing $\alpha$ in Section~\ref{sec:guidance_alpha}, an example prompt in Figure~\ref{fig:cil_langsym_example} and model outputs in Table~\ref{tab:cil_ft_fa}.

\paragraph{Remark.} Since these LLMs are pre-trained on a large corpus on natural language data, notice from Figure~\ref{fig:cil_langsym_qwen_7B_instruct_cot_recipes_N_4_M_2_C_3} that the $\tacc$ of the $\alpha=0$ model with `Force Answer' strategy exhibits an increasing trend rather than a deteriorating one as observed with ~\coticlnew. This behavior is unique to the pre-trained models and understanding the role of model size and pre-training datasets can be an interesting avenue for future research.
% For the `Force Answer' strategy, notice that $\alpha=\{0.5, 1,2\}$ allow the model to exhibit better reasoning performance compared to the extreme cases of $\alpha=\{0, \infty\}$ when $K'=0$ (i.e all examples are \textit{CoT examples}). This behaviour is unique to the pre-trained LLMs and was not observed with transformers trained from scratch on~\coticlnew. 
% We present an example prompt in Figure~\ref{fig:cil_langsym_example} and model outputs in Table~\ref{tab:cil_ft_fa}. 
% Overall, highlighting that the behavior of LLMs is consistent with that of models in Section~\ref{subsec:diversity_overreliance} with~\coticlnew.
% to gain transferrable insights into the ICL capabilities of LLMs when pretrained knowledge is insufficient.

\paragraph{Correct thinking steps do not necessarily imply a correct final answer.} We break down the $\tacc$ of the \texttt{Qwen-2.5-7B-Instruct} model trained with $\alpha=0$ by analyzing the correctness of the intermediate reasoning steps for the evaluation prompts. Since $\widetilde{\mC}=\{3\}$, there will be $2$ intermediate steps. Figure~\ref{fig:7b_instruct_alpha_0_step_analysis} considers the scenario with `Force Think' strategy and $K'=0$ (i.e, all examples have CoT) to highlight that: out of the $6752$ correctly predicted prompts, around $80\%$ of them have both the intermediate steps to be correct. Whereas, out of the $3428$ prompts with wrong final answers, $\approx 30\%$ of them have both the steps to be correct and $\approx 60\%$ of them have at least one step to be correct. We also present breakdowns for other models and $\alpha$, based on the inclusion of these intermediate steps in the ground truth DAG in Appendix~\ref{app:cot_reliance}. In summary, our results indicate that the output of the second thinking step (denoted as Step 2) is relatively more important than Step 1 because of the underlying causal structure. Thus, incorrect Step 2 predictions lead to relatively more errors in the final answer than Step 1. 
% \maz{What is the message here? I did not get it. Step 1 and Step 2 are not clear to me.}

\begin{figure}[h]
\centering
\begin{subfigure}{0.45\linewidth}
\centering
\begin{tikzpicture}[scale=1.2]
    % Left heatmap - Correct Final Answers
    \fill[low] (0,1) rectangle (1,2);    % 81.5%
    \fill[low] (1,1) rectangle (2,2);     % 7.8%
    \fill[high] (0,0) rectangle (1,1);     % 5.0%
    \fill[low] (1,0) rectangle (2,1);     % 5.6%

    \draw[thick] (0,0) rectangle (2,2);
    \draw[thick] (0,1) -- (2,1);
    \draw[thick] (1,0) -- (1,2);

\node at (0.5,0.5) {\textbf{82.8\%}}; % cc  
\node at (1.5,1.5) {\textbf{4.8\%}};  % ii  
\node at (1.5,0.5) {\textbf{4.9\%}};  % ic  
\node at (0.5,1.5) {\textbf{7.5\%}};  % ci

    % Axis labels
    \node[anchor=center] at (0.5,-0.25) {\footnotesize Correct};
    \node[anchor=center] at (1.5,-0.25) {\footnotesize Incorrect};
    \node[anchor=center, rotate=90] at (-0.25,0.5) {\footnotesize Correct};
    \node[anchor=center, rotate=90] at (-0.25,1.5) {\footnotesize Incorrect};

    % Step labels
    \node at (1,2.3) {\small Step1};
    \node[rotate=90] at (-0.6,1) {\small Step2};
\end{tikzpicture}
\caption{\checkmark Answer $(6752)$}
\end{subfigure}
\begin{subfigure}{0.45\linewidth}
\centering
\begin{tikzpicture}[scale=1.2]
    % Right heatmap - Incorrect Final Answers
    \fill[med] (0,1) rectangle (1,2);     % 30.5%
    \fill[high] (1,1) rectangle (2,2);     % 22.4%
    \fill[med] (0,0) rectangle (1,1);     % 11.6%
    \fill[med] (1,0) rectangle (2,1);    % 35.4%

    \draw[thick] (0,0) rectangle (2,2);
    \draw[thick] (0,1) -- (2,1);
    \draw[thick] (1,0) -- (1,2);

\node at (0.5,0.5) {\textbf{30.3\%}}; % cc  
\node at (1.5,1.5) {\textbf{36.1\%}}; % ii  
\node at (1.5,0.5) {\textbf{12.3\%}}; % ic  
\node at (0.5,1.5) {\textbf{21.3\%}}; % ci

 % Axis labels
    \node[anchor=center] at (0.5,-0.25) {\footnotesize Correct};
    \node[anchor=center] at (1.5,-0.25) {\footnotesize Incorrect};
    \node[anchor=center, rotate=90] at (-0.25,0.5) {\footnotesize Correct};
    \node[anchor=center, rotate=90] at (-0.25,1.5) {\footnotesize Incorrect};

    % Step labels
    \node at (1,2.3) {\small Step1};
    \node[rotate=90] at (-0.6,1) {\small Step2};
\end{tikzpicture}
\caption{$\times$ Answer $(3248)$}
\end{subfigure}
\caption{\texttt{Qwen-2.5-7B-Instruct} trained with $\alpha=0$ achieves $0.6752$ final answer $\tacc$ with the `Force Think' strategy for $K'=0$. Each grid in the plots indicate the percentage of prompts (out of $6752$ in (a) and $3248$ in (b)) for which the Step 1 and Step 2 predictions by the model were correct/incorrect.}
\label{fig:7b_instruct_alpha_0_step_analysis}
\vspace{-4mm}
\end{figure}
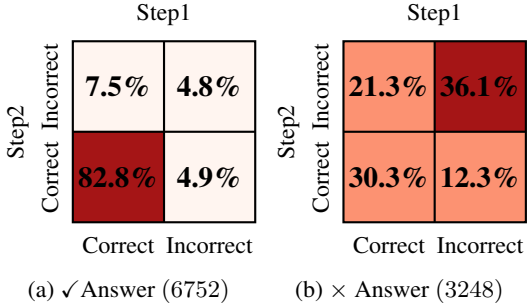

\paragraph{Input length generalization.} Figure~\ref{fig:cil_langsym_cot_recipes_input_len_gen_7B_instruct} illustrates that the $\tacc$ of the SFT'ed \texttt{Qwen-2.5-7B-Instruct} model tends to drop as the number of input words per example in the evaluation sequences is increased. This observation is consistent with the results of~\coticlnew. However, a key difference is that the $\tacc$ values on \texttt{CIL-LangSym} are relatively closer to the performance on the in-domain task (Figure~\ref{fig:cil_langsym_qwen_7B_instruct_cot_recipes_N_4_M_2_C_3}), with `Force Think' strategy outperforming `Force Answer' for attaining the peak $\tacc$ across $\alpha$ and $K'$. A similar behavior can be observed for the chain length generalization experiments in Figure~\ref{fig:cil_langsym_cot_recipes_chain_len_gen_7B_instruct}. This indicates that although pre-training is not sufficient for solving the tasks, the diversity of the data aids in OOD generalization. Thus, highlighting the key difference between training from scratch for~\coticlnew~tasks and pre-training on large natural language corpora.

\begin{figure}[h!]
    \centering
    % First subfigure
    \begin{subfigure}[b]{0.48\linewidth}
        \centering
        \includegraphics[width=\linewidth]{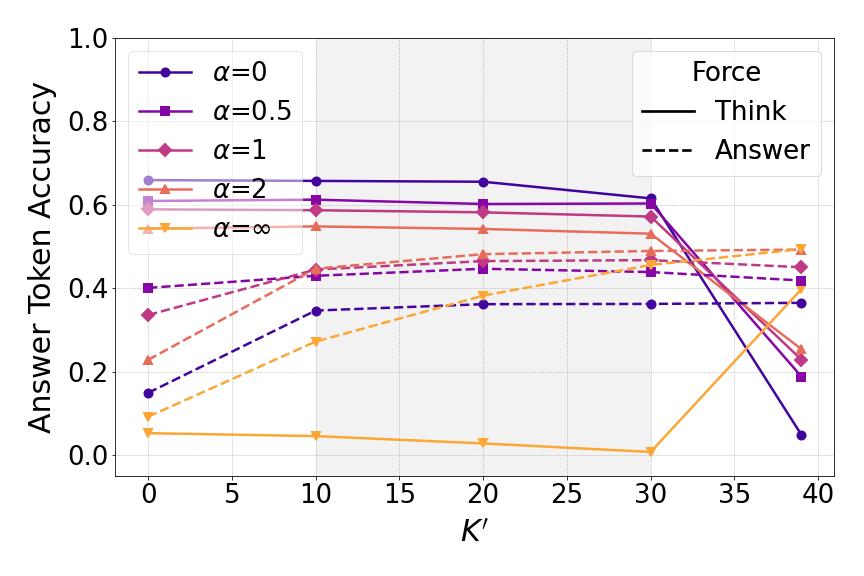}
        \caption{$\widetilde{\mN}=\{5\}$}
        \label{fig:cil_symbolic_cot_recipes_len_gen_7B_instruct_N_5}
    \end{subfigure}
    % \hfill
    % Second subfigure
    \begin{subfigure}[b]{0.48\linewidth}
        \centering
        \includegraphics[width=\linewidth]{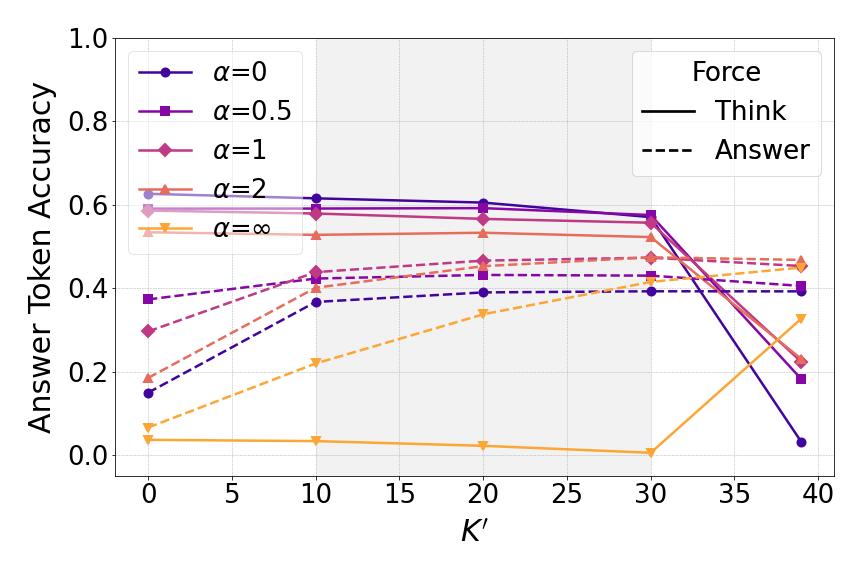}
        \caption{$\widetilde{\mN}=\{6\}$}
        \label{fig:cil_symbolic_cot_recipes_len_gen_7B_instruct_N_6}
    \end{subfigure}
    \caption{\texttt{Qwen-2.5-7B-Instruct}: Input length generalization with train: $\mN=\{4\}, \mM=\{2\}, \mC=\{3\}$ and test: $\widetilde{\mM}=\{2\}, \widetilde{\mC}=\{3\}$ by varying $\widetilde{\mN}$. }
    \label{fig:cil_langsym_cot_recipes_input_len_gen_7B_instruct}
    % \vspace{-4mm}
\end{figure}

\section{Guidance on Choosing $\alpha$}
\label{sec:guidance_alpha}

Throughout the paper, $\alpha \in \{0, 0.5, 1, 2, \infty\}$ was varied to simulate constant, linear, sub-linear, and super-linear growth of $r_{\texttt{Cot}}$ in the dataset. Although our coverage of $\alpha$ is limited to these $5$ values, we observe clear patterns in our experiments that guide their choice for meta-training.

\paragraph{CoT rich scenarios.} When the evaluation prompts are expected to have a sufficient number of CoT examples, then it is recommended to choose $\alpha = \{0, 0.5\}$ and employ the `Force Think' strategy to maximize the performance.

\paragraph{CoT poor scenarios.} If minimal CoT examples are expected to be available during inference, then it is recommended to choose $\alpha = \{2, \infty\}$ and employ the `Force Answer' strategy so that the model is not meta-trained on sequences that causes it to over-rely on CoT supervision.

\paragraph{A single model that can reason and answer directly.} Finally, if a single model is expected to perform well in all kinds of evaluation scenarios, i.e, with and without CoT examples as well as being forced to think or answer directly, then it is recommended to meta-train with $\alpha \in \{0.5, 2\}$ to achieve the best tradeoffs. We underscore that this is not a universal selection criterion and only aims to narrow down the optimal value selection. We also analyze the role of $\alpha$ on computational overheads in Appendix~\ref{app:expected_token_count_alpha}.

% \vspace{-4mm}
\section{Conclusion}

This work introduced~\coticlnew, a data generation framework for abstract reasoning tasks to design effective meta-training techniques for transformer models. By incorporating special abstract tokens and modulating the mix of \textit{CoT/standard examples} via \texttt{CoT-Recipe}, we systematically showcased the importance of training task diversity and forcing strategies for reasoning control in transformers. By verifying the effectiveness of these insights on practical LLMs for novel symbolic reasoning tasks, we hope to encourage the formalization of the data-mixing aspects of meta-training with broader domain-specific tasks.

\section{Limitations}

The~\coticlnew~framework is designed to generate diverse sequences of abstract tokens that are devoid of semantics. In this context, we note that eliciting chain-of-thought does not exactly resemble the case with NLP datasets (especially in the zero-shot settings), since the token distribution, input and chain lengths may vary significantly with answer tokens exceeding just a single token. Although the \texttt{CIL-LangSym} addresses some of these concerns, it can be treated as a domain-specific dataset, and one should carefully consider the scenarios and training stages in which these insights can apply to real-world tasks with math/code, etc.

\bibliography{references}

\appendix

\clearpage
\section{The~\coticlnew~Dataset Generation Algorithms}
\label{app:datagen_algo}

We formalize the design aspects of~\coticlnew~and present the dataset generation process in Algorithm~\ref{alg:cot_icl_lab_2_dataset}. A single tokenized sequence in this dataset is generated using Algorithm~\ref{alg:cot_icl_lab_2_single_sequence}, especially by isolating the input and answering parts of an in-context example using special tokens. Algorithm~\ref{alg:cot_icl_lab_2_single_chain_token} uses the data embeddings $\mE_{\texttt{data}}$ to generate a single chain token for every in-context example.

% \subsection{Generate the Dataset of Sequences}

\begin{algorithm}
\caption{Generate dataset with $T$ sequences}
\begin{algorithmic}[1]
\label{alg:cot_icl_lab_2_dataset}
\REQUIRE Parameter choices $\mN, \mM, \mC, K$, the \texttt{CoT-Recipe} parameters $\alpha, a, b$, and size $T$.
\STATE Initialize empty dataset $\mD = []$
\FOR{$j = 1$ to $T$}
\STATE $r_{CoT}^{(j)} = \texttt{CoT-Recipe}(\alpha, a, b)(j/T)$
\STATE $\vp$ = Algorithm~\ref{alg:cot_icl_lab_2_single_sequence}$(\mN, \mM, \mC, K, r_{CoT}^{(j)})$.
\STATE $\mD.\texttt{append}(\vp)$
\ENDFOR
\RETURN $\mD$
\end{algorithmic}
\end{algorithm}

% \subsection{Generate Single Sequence}

\begin{algorithm}
\caption{Single sequence generation with index $j$ in the dataset.}
\begin{algorithmic}[1]
\label{alg:cot_icl_lab_2_single_sequence}
\REQUIRE Parameter choices $\mN, \mM, \mC, K$, and the CoT probability parameter $r_{CoT}^{(j)}$.

\STATE sample $N \sim \mN, M \sim \mM, C \sim \mC$.
\STATE Limit $M = \min(M, N)$.
\STATE Initialize the sequence $\vp = [t_{\texttt{bos}}]$ 
\FOR{$k = 1$ to $K$}
\STATE Initialize empty input token sequence $\vx$, chain token sequence $\vy$.
\FOR{$i = 1$ to $N$}
    \STATE $\vx[i] \overset{\text{i.i.d.}}{\sim} \gV_{\texttt{normal}}$
\ENDFOR
\STATE $\vt = \vx.\texttt{clone()}$
\FOR{$c = 1$ to $C$}
    \STATE $\texttt{parent\_tokens}=\texttt{rand.choice}(\vt, M)$
    \STATE $\vy[c]$ = Algorithm~\ref{alg:cot_icl_lab_2_single_chain_token}(\texttt{parent\_tokens})
    \STATE $\vt.\texttt{append}(\vy[c])$
\ENDFOR
\STATE $\vp.\texttt{extend}([t_{\texttt{inp\_start}} ,\vx, t_{\texttt{inp\_end}}])$
\IF{$r^{(j)}_{CoT} \ge \gU(0,1)$}
\STATE $\vp.\texttt{extend}([t_{\texttt{think\_start}} ,\vy_{:C-1}, t_{\texttt{think\_end}}])$
\ENDIF
\STATE $\vp.\texttt{extend}([t_{\texttt{ans\_start}} ,y_C, t_{\texttt{ans\_end}}, t_{\texttt{eos}}])$
\ENDFOR
\RETURN $\vp$
\end{algorithmic}
\end{algorithm}

% \subsection{Generate Single Chain Token}
\begin{algorithm}
\caption{Single \textit{chain token} $y_c$ generation}
\begin{algorithmic}[1]
\label{alg:cot_icl_lab_2_single_chain_token}
\REQUIRE $M$ row indices of $\mE_{\texttt{data}}$ corresponding to \texttt{parent\_tokens}.

\STATE MLP $h_c \in \texttt{TokenProcessorCache}(\mathcal{H}(l=1,\phi=\texttt{LeakyReLU}))$
\FOR{$i = 1$ to $M$}
    \STATE $\mathbf{h}^i \leftarrow h_c(\mE_{\texttt{data}}[\texttt{parent\_tokens}[i]])$
\ENDFOR
\STATE $\mathbf{h}_{\text{act}} \leftarrow \phi\left(\frac{1}{M} \sum_{i=1}^{M} \mathbf{h}^i\right)$
\STATE $y_c \leftarrow \texttt{argmax}(\mathbf{E}_{\texttt{data}} \mathbf{h}_{\text{act}})$
\RETURN $y_c$
\end{algorithmic}
\end{algorithm}

\section{Task Difficulty with \texttt{CoT-Recipe} and Training Token Estimates}

\subsection{Fine-grained control of task difficulty}

Considering the meta-training setup of Section~\ref{sec:role_of_cot_recipe}, we use $\widetilde{\mN}=\widetilde{\mM}=\widetilde{\mC}=\{4\}, \widetilde{K}=40$ for the evaluation ($\widetilde{T} = 10^4$) datasets corresponding to the same $\alpha$ and apply the `No Forcing' strategy to measure $\tacc$. As $\alpha$ increases, the sequences tend to contain fewer proportion of \textit{CoT examples} in-context (see Figure~\ref{fig:cot_icl_lab_extended}) and in-turn lead to consistently lower $\tacc$ values across model sizes (see Figure~\ref{fig:train_eval_cot_recipes_N_4_M_4_C_4}). In particular, when $\alpha=0$, the $\tTF$ models leverage the intermediate/thinking tokens to achieve higher $\tacc$, whereas $\alpha=\infty$ presents no such information and the model is forced to answer directly. Although such extreme cases were already studied in \citet{kothapalli2025cot}, our results highlight a fine-grained control over the difficulty of such tasks by carefully selecting the shape parameter $\alpha$.

\begin{figure*}[h!]
    \centering
    % First subfigure
    \begin{subfigure}[b]{0.32\linewidth}
        \centering
        \includegraphics[width=\textwidth]{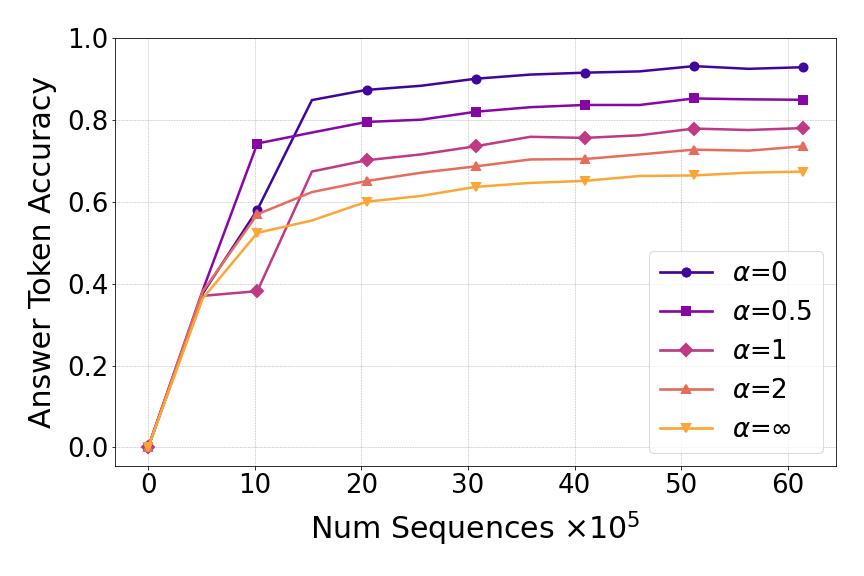}
        \caption{\texttt{TF-4}}
        \label{fig:train_eval_cot_recipes_N_4_M_4_C_4_L_4}
    \end{subfigure}
    \hfill
    % Second subfigure
    \begin{subfigure}[b]{0.32\linewidth}
        \centering
        \includegraphics[width=\textwidth]{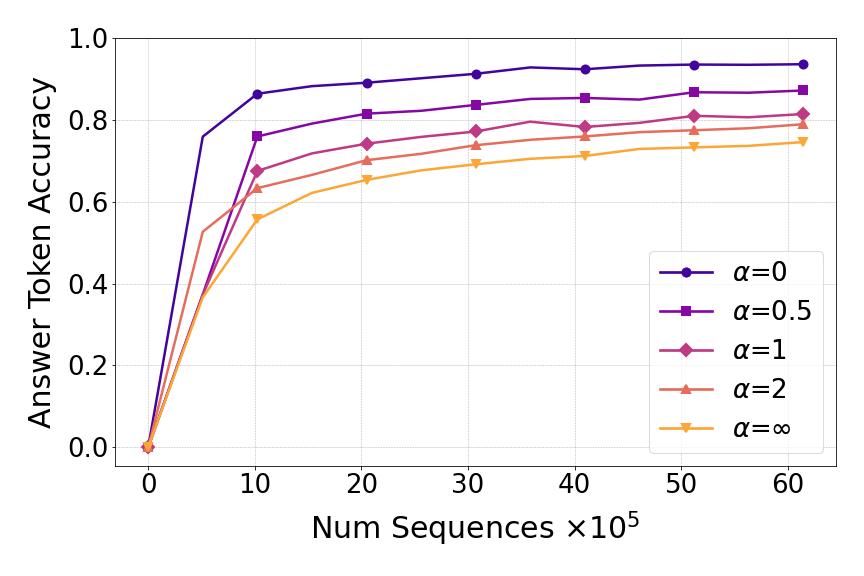}
        \caption{\texttt{TF-8}}
        \label{fig:train_eval_cot_recipes_N_4_M_4_C_4_L_8}
    \end{subfigure}
    % \hfill
    % \vskip\baselineskip
    % Third subfigure
    \begin{subfigure}[b]{0.32\linewidth}
        % \centering
        \includegraphics[width=\textwidth]{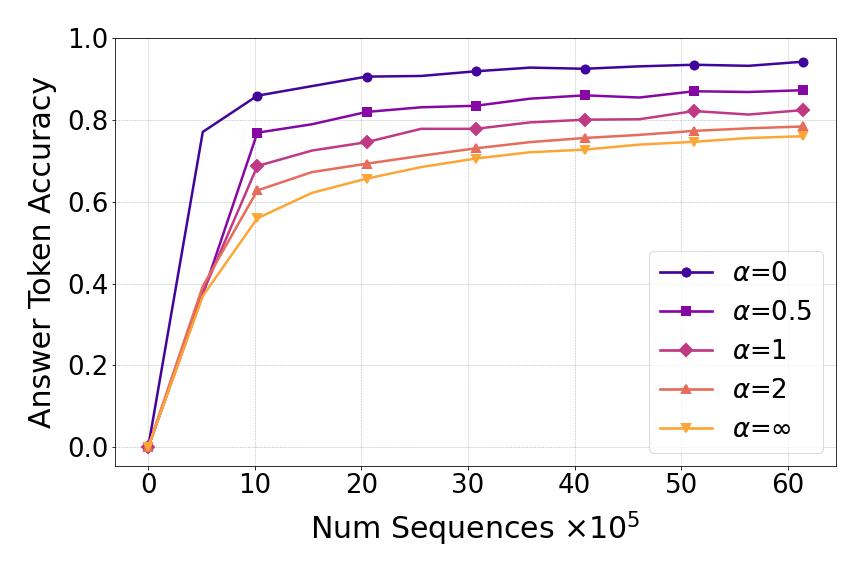}
        \caption{\texttt{TF-12}}
        \label{fig:train_eval_cot_recipes_N_4_M_4_C_4_L12}
    \end{subfigure}
    \caption{$\tacc$ with varying $\alpha$, $\mN=\mM=\mC=\{4\}$, and $\widetilde{\mN}=\widetilde{\mM}=\widetilde{\mC}=\{4\}$ .}
    \label{fig:train_eval_cot_recipes_N_4_M_4_C_4}
    \vspace{-4mm}
\end{figure*}

\subsection{Expected Token Count in Datasets}
\label{app:expected_token_count_alpha}
\begin{theorem}
\label{thm:expected_token_count}
 Consider a dataset $\gD$ of $T$ sequences created using the tuple $N,M,C,K$. Let the \texttt{CoT-Recipe}~\eqref{eq:cot_recipe} with $a=1,b=0$ determine the CoT probability parameter $r^{(j)}_{\text{CoT}}, \forall j \in [0, T-1]$ as follows: $r^{(j)}_{\text{CoT}} = \left(\frac{j}{T}\right)^{\alpha}$. Then the expected number of tokens $\mathbb{E}\left[|\gD|_{Tokens}\right]$ is:
    \begin{align*}
    \begin{split}
       \mathbb{E}\left[|\gD|_{Tokens}\right] &= T + \mathbb{E}\left[|\gD|_{CoT-ex}\right](N+C+7) \\
        &+ (KT - \mathbb{E}\left[|\gD|_{CoT-ex}\right])(N+6) 
    \end{split}
    \end{align*},
    where $\mathbb{E}\left[|\gD|_{CoT-ex}\right]$ is given by:
    \begin{align*}
    \frac{K}{T^\alpha}\left(\frac{(T-1)^{\alpha+1} - 1}{\alpha + 1} + \frac{(T-1)^\alpha + 1}{2} \right)
    \end{align*}
\end{theorem}
\begin{proof}
Notice that a \textit{standard example}~\eqref{eq:standard_example} consists of $N + 6$ tokens, and a \textit{CoT example}~\eqref{eq:cot_example} consists of $N + C + 7$ tokens.

For a sequence $\vp$ with $K$ examples and CoT probability parameter $r^{(j)}_{\text{CoT}}$, the expected number of \textit{CoT/standard examples} are:
\begin{align}
\begin{split}
    \mathbb{E}\left[|\vp|_{CoT-ex}\right] &= K\times r^{(j)}_{\text{CoT}} \\
    \mathbb{E}\left[|\vp|_{S-ex}\right] &= K\times (1-r^{(j)}_{\text{CoT}})
\end{split}
\end{align}

By the linearity of expectations, the expected number of \textit{CoT/standard examples} in the entire dataset is given by:
\begin{align}
\begin{split}
    \mathbb{E}\left[|\gD|_{CoT-ex}\right] &= \sum_{j=0}^{T-1} K\left(\frac{j}{T}\right)^{\alpha}\\
    \mathbb{E}\left[|\gD|_{S-ex}\right] &= \sum_{j=0}^{T-1} K\left(1 - \left(\frac{j}{T}\right)^{\alpha}\right).
\end{split}
\end{align}

We use the Euler-Maclaurin approximation of $\mathbb{E}\left[|\gD|_{CoT-ex}\right]$ to obtain:
\begin{align*}
    &\frac{K}{T^\alpha}\sum_{j=0}^{T-1} j^{\alpha} = \frac{K}{T^\alpha}\sum_{j=1}^{T-1}j^{\alpha} \\
    &\approx \frac{K}{T^\alpha}\left(\int_1^{T-1} x^\alpha dx + \frac{(T-1)^\alpha + 1}{2} \right) \\
    &= \frac{K}{T^\alpha}\left(\frac{(T-1)^{\alpha+1} - 1}{\alpha + 1} + \frac{(T-1)^\alpha + 1}{2} \right)
\end{align*}

Since $\mathbb{E}\left[|\gD|_{S-ex}\right] = KT - \mathbb{E}\left[|\gD|_{CoT-ex}\right]$, we consider $1$ $t_{\texttt{bos}}$ token per sequence and calculate the expected tokens (excluding $t_{\texttt{pad}}$) in $\gD$ as:
\begin{align*}
\begin{split}
    \mathbb{E}\left[|\gD|_{Tokens}\right] &= T + \mathbb{E}\left[|\gD|_{CoT-ex}\right](N+C+7) \\
    &+ \mathbb{E}\left[|\gD|_{S-ex}\right](N + 6).
\end{split}
\end{align*}
\end{proof}

\paragraph{Remark.} Based on Theorem~\ref{thm:expected_token_count} and $\alpha=0$, we get $\mathbb{E}\left[|\gD|_{CoT-ex}\right] = KT$ and $\mathbb{E}\left[|\gD|_{Tokens}\right] = T + KT(N+C+7)$. Whereas $\alpha=\infty$ results in $\mathbb{E}\left[|\gD|_{CoT-ex}\right] = 0$ and $\mathbb{E}\left[|\gD|_{Tokens}\right] = T + KT(N+6)$. By substituting: $N=C=4, K=40$ and $T=64 \times 10^5$ as per training experiments in Section~\ref{sec:role_of_cot_recipe}, the ratio $\frac{1 + K(N+C+7)}{1 + K(N+6)}$ turns out to be $1.5$. A numerical simulation of the ratio of expected tokens with $\alpha=2$ and $\alpha=\infty$ turns out to be $\approx 1.16$. Thus indicating that with a slight increase in token budget to modulate the mix of \textit{CoT/standard examples} in the training dataset, we can achieve significant improvements in reasoning control and length generalization (Section~\ref{sec:guidance_alpha}).

\section{Resources and Hyper-Parameters for Training and Inference}
\label{app:hardware_hyperparams}

\begin{table}[ht!]
\centering
\begin{tabular}{|c|c|}
\hline
\textbf{Model} & \textbf{Params w/o Embedding Layer} \\ \hline
\texttt{TF-4}      & $243,288,064$\\ \hline
\texttt{TF-8}      &  $486,574,080$ \\ \hline
\texttt{TF-12}      & $729,860,096$ \\ \hline
\end{tabular}
\caption{Model Card for the \texttt{TF} models.}
\label{tab:model_card}
\vspace{-2mm}
\end{table}

We use $8$ H100 NVIDIA GPUs for all the training experiments and employ Liger-Kernels \cite{hsu2025ligerkernel} to speed up training and vLLM \cite{kwon2023efficient} for bulk inference experiments.
\paragraph{\coticlnew.} We use a training batch size of $16$ per rank and the \texttt{AdamW} optimizer with $\eta=5\times 10^{-5}$. Training runs with the larger \texttt{TF-12} model take $\approx 18$ hours to finish. The evaluation with vLLM for $K'=\{0,10,20,30,39\}$ and both the forcing strategies take up to $2$ hours for the \texttt{TF-12} model on a single H100 GPU.

\paragraph{\texttt{CIL-LangSym}.} The SFT experiments with the \texttt{Qwen-2.5} series LLMs use a learning rate of $\eta=10^{-5}$, with a warm-up ratio of $0.05$, a cosine learning rate scheduler and a weight decay of $10^{-4}$. All experiments typically finish under $5$ minutes since the training set consists only of $1000$ prompts. Considering an inference token budget of $1000$ for `Force Think' and $100$ for the `Force Answer' strategy, the evaluation with vLLM for $K'=\{0,10,20,30,39\}$ and both the strategies take up to $5$ hours for the \texttt{7B-Instruct} model on a single H100 GPU.

\section{Length Generalization Experiments with~\coticlnew}

\begin{figure}[h!]
    \centering
    % First subfigure
    \begin{subfigure}[b]{\linewidth}
        \centering
        \includegraphics[width=\linewidth]{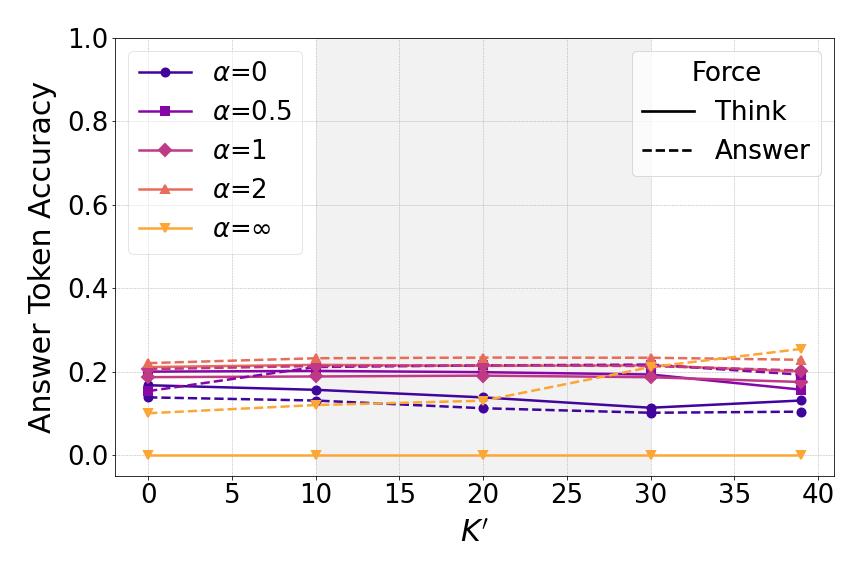}
        \caption{$\mN=\mM=\mC=\{4\}$}
        \label{fig:cot_recipes_train_N_4_M_4_C_4_K_40_bulk_eval_N_5_M_4_C_4_K_40_L_4}
    \end{subfigure}
    % \hfill
    % Second subfigure
    \begin{subfigure}[b]{\linewidth}
        \centering
        \includegraphics[width=\linewidth]{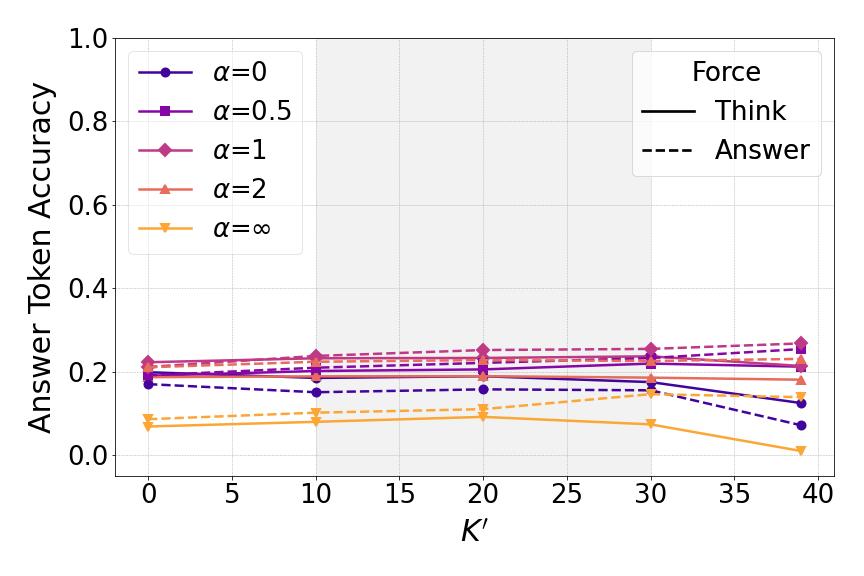}
        \caption{$\mN=\mM=\mC=\{3, 4\}$}
        \label{fig:cot_recipes_train_N_3-4_M_3-4_C_3-4_K_40_bulk_eval_N_5_M_4_C_4_K_40_L_4}
    \end{subfigure}
    \caption{Input length generalization: $\tacc$ of \texttt{TF-4} models trained with varying $\alpha$ and $K=40$ on evaluation datasets with longer inputs $\widetilde{\mN}=\{5\},\widetilde{\mM}=\widetilde{\mC}=\{4\}$ and $\widetilde{K}=40$. }
    \label{fig:cot_recipes_len_gen_bulk_eval_N_5_M_4_C_4_K_40_L_4}
    % \vspace{-4mm}
\end{figure}

\begin{figure}[h!]
    \centering
    % First subfigure
    \begin{subfigure}[b]{\linewidth}
        \centering
        \includegraphics[width=\linewidth]{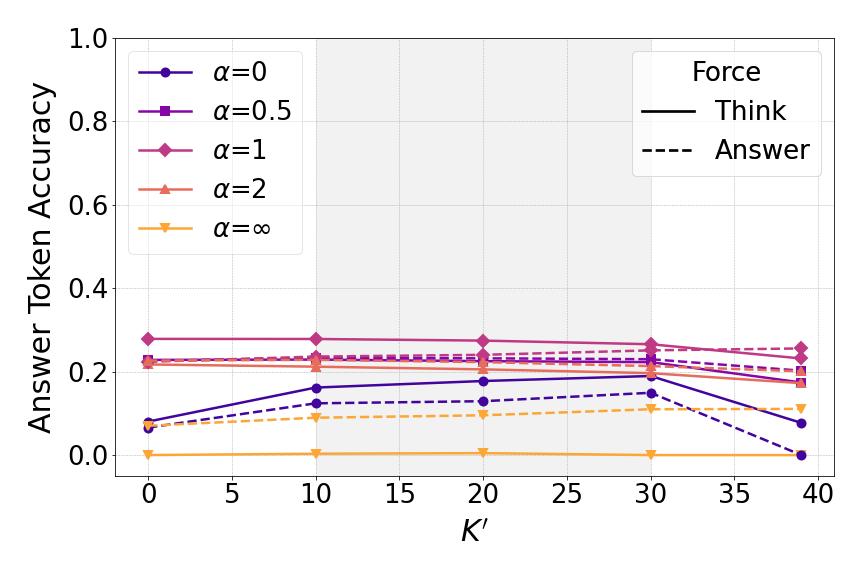}
        \caption{$\mN=\mM=\mC=\{4\}$}
        \label{fig:cot_recipes_train_N_4_M_4_C_4_K_40_bulk_eval_N_5_M_4_C_4_K_40_L_8}
    \end{subfigure}
    \hfill
    % Second subfigure
    \begin{subfigure}[b]{\linewidth}
        \centering
        \includegraphics[width=\linewidth]{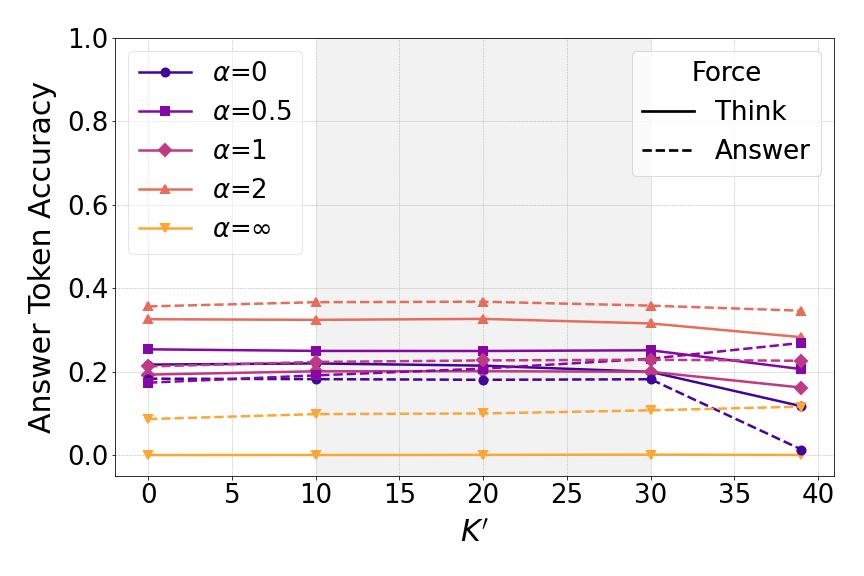}
        \caption{$\mN=\mM=\mC=\{3, 4\}$}
        \label{fig:cot_recipes_train_N_3-4_M_3-4_C_3-4_K_40_bulk_eval_N_5_M_4_C_4_K_40_L_8}
    \end{subfigure}
    \caption{Input length generalization: $\tacc$ of \texttt{TF-8} models trained with varying $\alpha$ and $K=40$ on evaluation datasets with longer inputs $\widetilde{\mN}=\{5\},\widetilde{\mM}=\widetilde{\mC}=\{4\}$ and $\widetilde{K}=40$.  }
    \label{fig:cot_recipes_len_gen_bulk_eval_N_5_M_4_C_4_K_40_L_8}
    % \vspace{-4mm}
\end{figure}

\begin{figure*}[h!]
    \centering
    % First subfigure
    \begin{subfigure}[b]{0.32\textwidth}
        \centering
        \includegraphics[width=\textwidth]{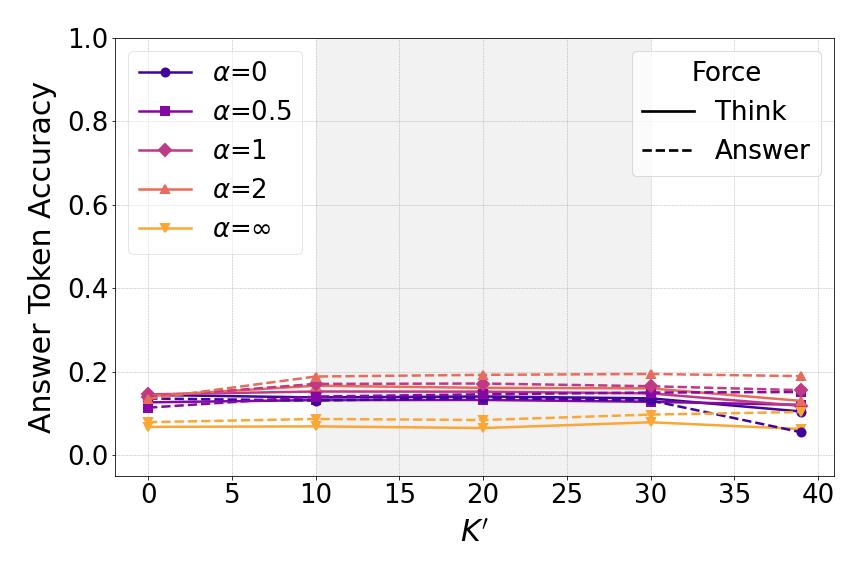}
        \caption{\texttt{TF-4}}
        \label{fig:cot_recipes_train_N_4_M_4_C_4_K_40_bulk_eval_N_6_M_4_C_4_K_40_L_4}
    \end{subfigure}
    \hfill
    % Second subfigure
    \begin{subfigure}[b]{0.32\textwidth}
        \centering
        \includegraphics[width=\textwidth]{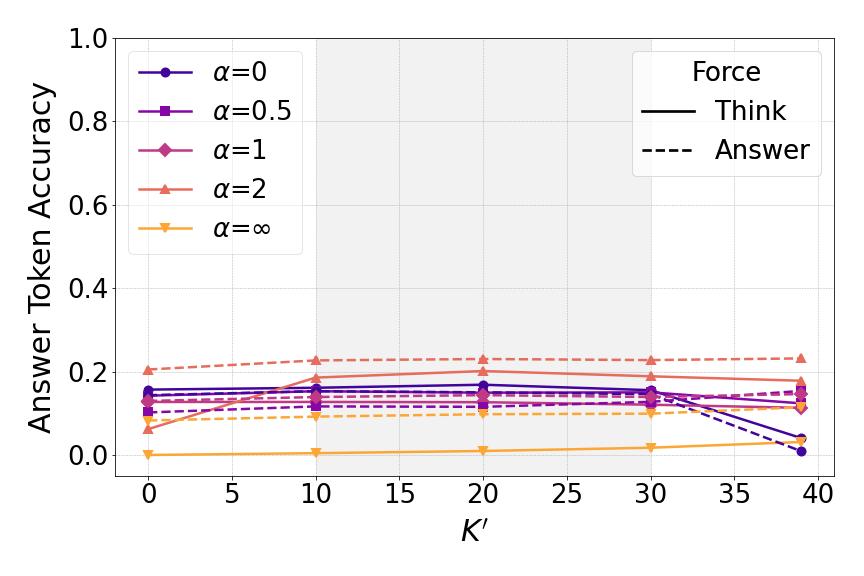}
        \caption{\texttt{TF-8}}
        \label{fig:cot_recipes_train_N_4_M_4_C_4_K_40_bulk_eval_N_6_M_4_C_4_K_40_L_8}
    \end{subfigure}
    \hfill
    % \vskip\baselineskip
    % Third subfigure
    \begin{subfigure}[b]{0.32\textwidth}
        \centering
        \includegraphics[width=\textwidth]{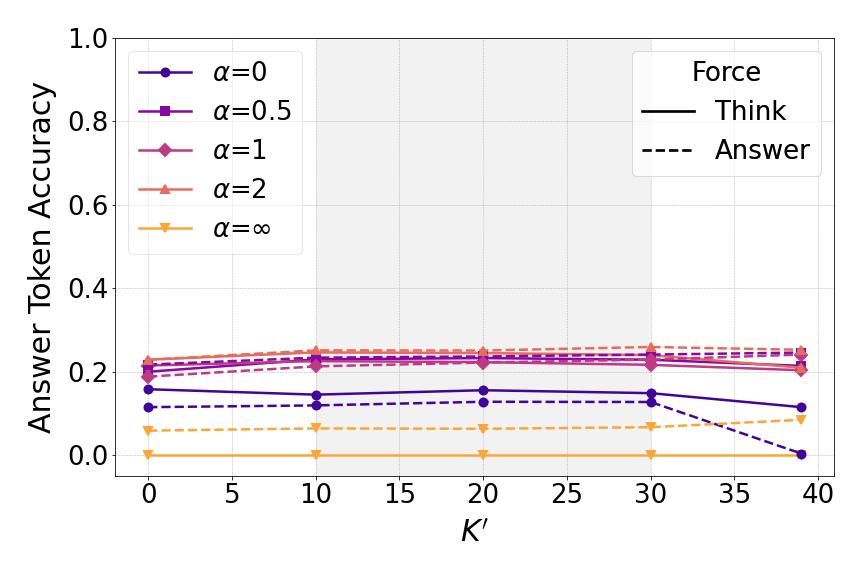}
        \caption{\texttt{TF-12}}
        \label{fig:cot_recipes_train_N_4_M_4_C_4_K_40_bulk_eval_N_6_M_4_C_4_K_40_L_12}
    \end{subfigure}
    \caption{Input length generalization: $\tacc$ of models trained with varying \texttt{CoT-Recipe}$(\alpha)$ and $\mN=\mM=\mC=\{4\}, K=40$ on evaluation datasets with $\widetilde{\mN}=\{6\}, \widetilde{\mM}=\widetilde{\mC}=\{4\}$ and $\widetilde{K}=40$. }
    \label{fig:cot_recipes_train_N_4_M_4_C_4_K_40_bulk_eval_N_6_M_4_C_4_K_40}
    % \vspace{-4mm}
\end{figure*}

\begin{figure*}[h!]
    \centering
    % First subfigure
    \begin{subfigure}[b]{0.32\textwidth}
        \centering
        \includegraphics[width=\textwidth]{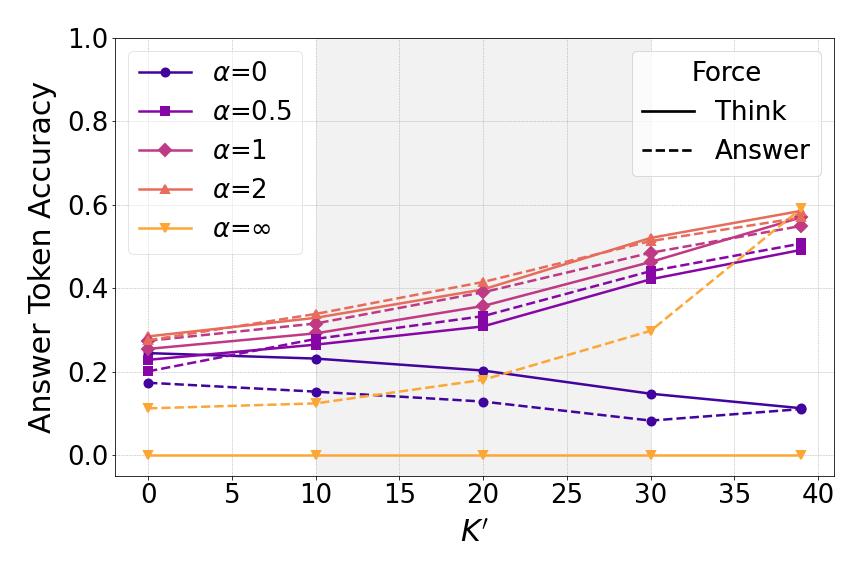}
        \caption{\texttt{TF-4}}
        \label{fig:cot_recipes_train_N_4_M_4_C_4_K_40_bulk_eval_N_4_M_4_C_5_K_40_L_4}
    \end{subfigure}
    \hfill
    % Second subfigure
    \begin{subfigure}[b]{0.32\textwidth}
        \centering
        \includegraphics[width=\textwidth]{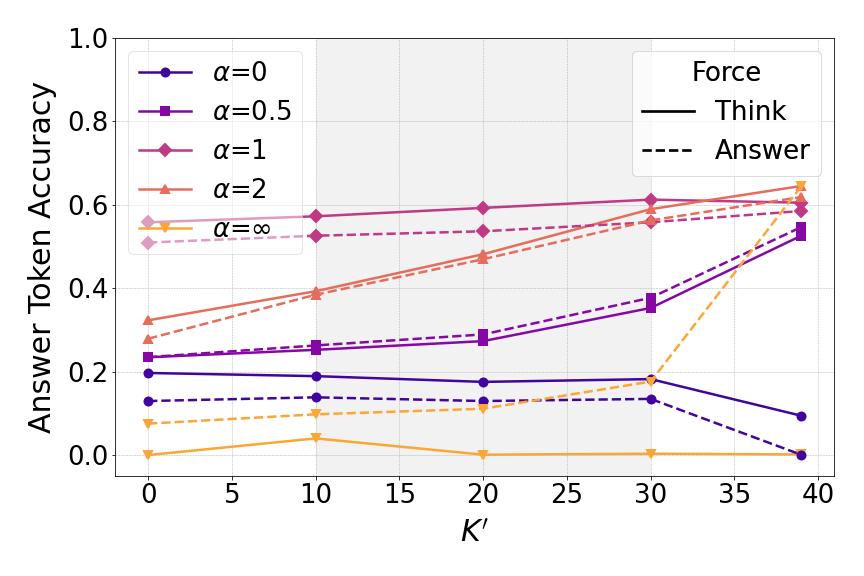}
        \caption{\texttt{TF-8}}
        \label{fig:cot_recipes_train_N_4_M_4_C_4_K_40_bulk_eval_N_4_M_4_C_5_K_40_L_8}
    \end{subfigure}
    \hfill
    % \vskip\baselineskip
    % Third subfigure
    \begin{subfigure}[b]{0.32\textwidth}
        \centering
        \includegraphics[width=\textwidth]{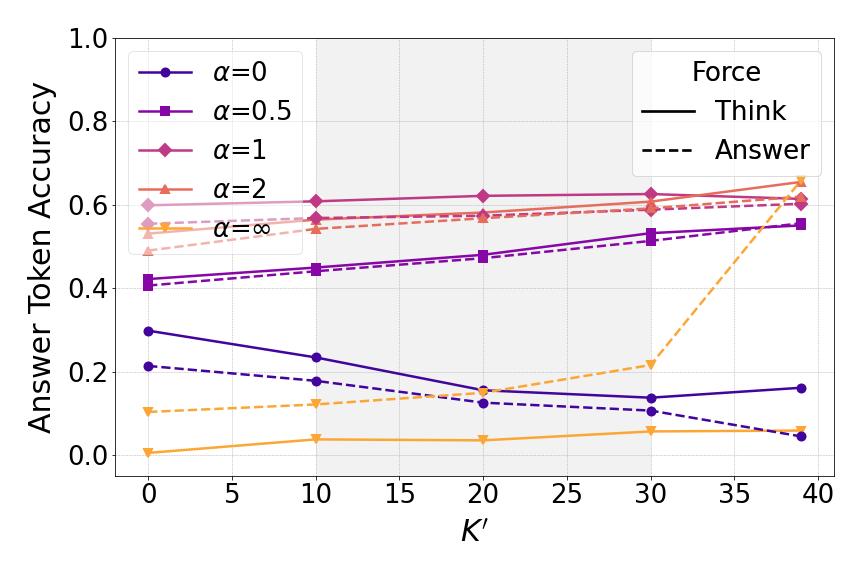}
        \caption{\texttt{TF-12}}
        \label{fig:cot_recipes_train_N_4_M_4_C_4_K_40_bulk_eval_N_4_M_4_C_5_K_40_L12}
    \end{subfigure}
    \caption{Chain length generalization: $\tacc$ of models trained with varying \texttt{CoT-Recipe}$(\alpha)$ and $\mN=\mM=\mC=\{4\}, K=40$ on evaluation datasets with $\widetilde{\mN}=\widetilde{\mM}=4, \widetilde{\mC}=\{5\}$ and $\widetilde{K}=40$. }
\label{fig:cot_recipes_train_N_4_M_4_C_4_K_40_bulk_eval_N_4_M_4_C_5_K_40}
    % \vspace{-4mm}
\end{figure*}

\begin{figure*}[h!]
    \centering
    % First subfigure
    \begin{subfigure}[b]{0.32\textwidth}
        \centering
        \includegraphics[width=\textwidth]{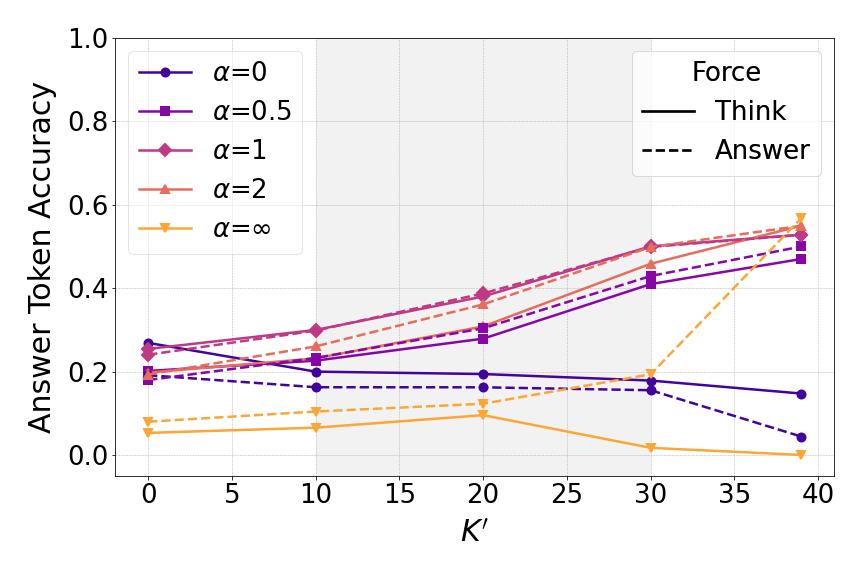}
        \caption{\texttt{TF-4}}
        \label{fig:cot_recipes_train_N_3-4_M_3-4_C_3-4_K_40_bulk_eval_N_4_M_4_C_5_K_40_L_4}
    \end{subfigure}
    \hfill
    % Second subfigure
    \begin{subfigure}[b]{0.32\textwidth}
        \centering
        \includegraphics[width=\textwidth]{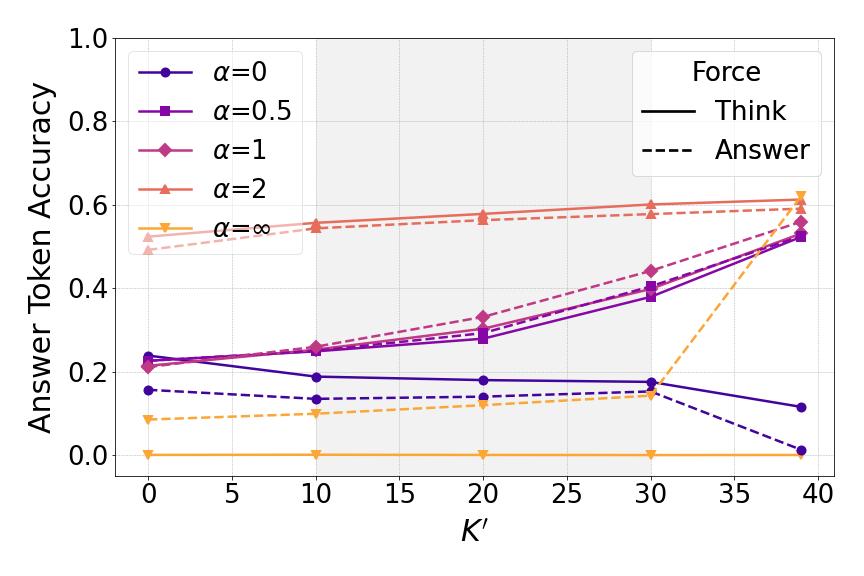}
        \caption{\texttt{TF-8}}
        \label{fig:cot_recipes_train_N_3-4_M_3-4_C_3-4_K_40_bulk_eval_N_4_M_4_C_5_K_40_L_8}
    \end{subfigure}
    \hfill
    % \vskip\baselineskip
    % Third subfigure
    \begin{subfigure}[b]{0.32\textwidth}
        \centering
        \includegraphics[width=\textwidth]{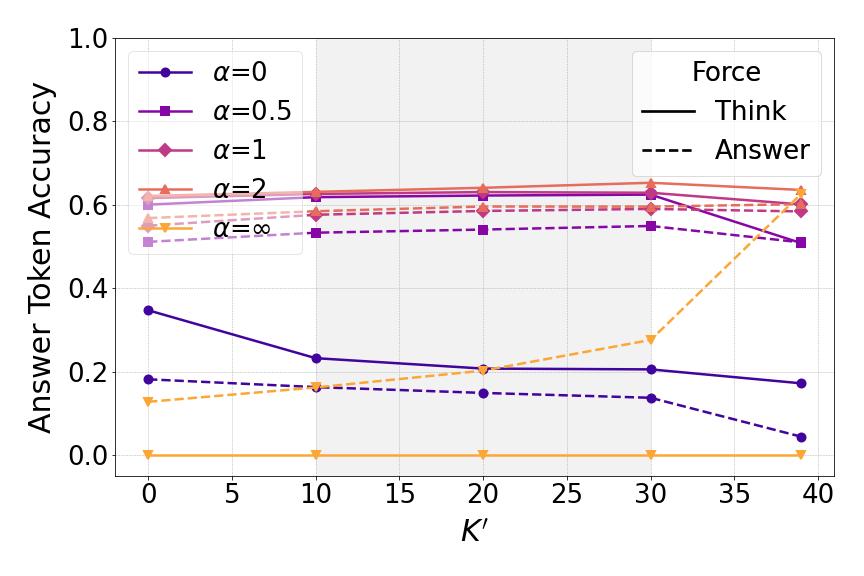}
        \caption{\texttt{TF-12}}
        \label{fig:cot_recipes_train_N_3-4_M_3-4_C_3-4_K_40_bulk_eval_N_4_M_4_C_5_K_40_L_12}
    \end{subfigure}
    \caption{Chain length generalization: $\tacc$ of models trained with varying \texttt{CoT-Recipe}$(\alpha)$ and $\mN=\mM=\mC=\{3, 4\}, K=40$ on evaluation datasets with $\widetilde{\mN}=\widetilde{\mM}=4, \widetilde{\mC}=\{5\}$ and $\widetilde{K}=40$. }
    \label{fig:cot_recipes_train_N_3-4_M_3-4_C_3-4_K_40_bulk_eval_N_4_M_4_C_5_K_40}
    % \vspace{-4mm}
\end{figure*}

\begin{figure*}[h!]
    \centering
    % First subfigure
    \begin{subfigure}[b]{0.32\textwidth}
        \centering
        \includegraphics[width=\textwidth]{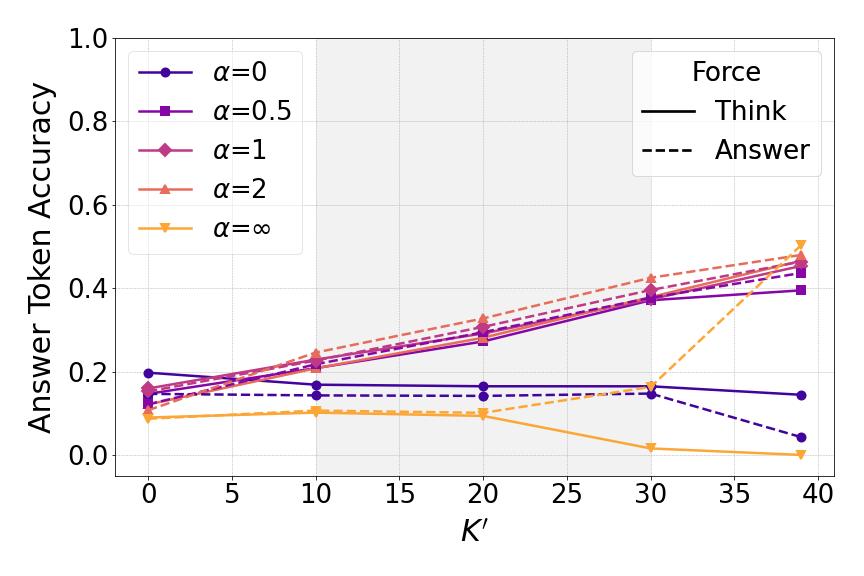}
        \caption{\texttt{TF-4}}
        \label{fig:cot_recipes_train_N_3-4_M_3-4_C_3-4_K_40_bulk_eval_N_4_M_4_C_6_K_40_L_4}
    \end{subfigure}
    \hfill
    % Second subfigure
    \begin{subfigure}[b]{0.32\textwidth}
        \centering
        \includegraphics[width=\textwidth]{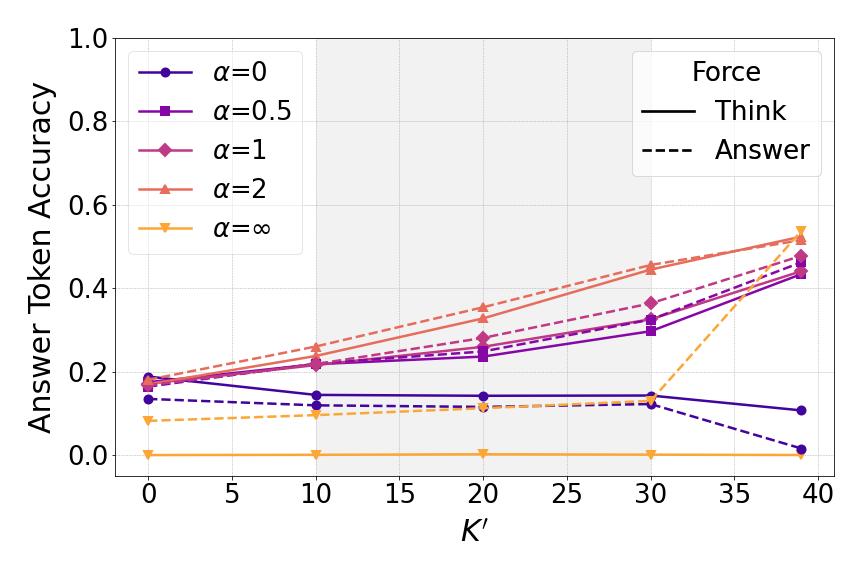}
        \caption{\texttt{TF-8}}
        \label{fig:cot_recipes_train_N_3-4_M_3-4_C_3-4_K_40_bulk_eval_N_4_M_4_C_6_K_40_L_8}
    \end{subfigure}
    \hfill
    % \vskip\baselineskip
    % Third subfigure
    \begin{subfigure}[b]{0.32\textwidth}
        \centering
        \includegraphics[width=\textwidth]{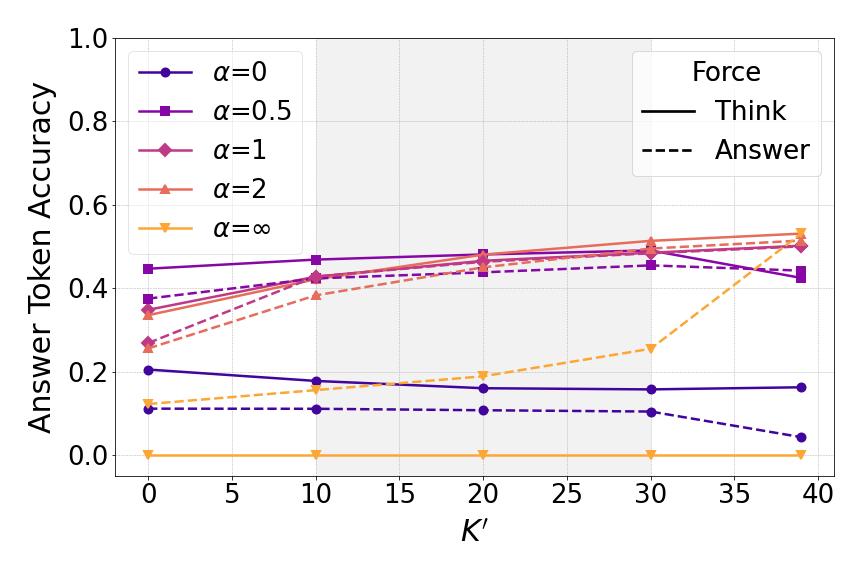}
        \caption{\texttt{TF-12}}
        \label{fig:cot_recipes_train_N_3-4_M_3-4_C_3-4_K_40_bulk_eval_N_4_M_4_C_6_K_40_L_12}
    \end{subfigure}
    \caption{Chain length generalization: $\tacc$ of models trained with varying \texttt{CoT-Recipe}$(\alpha)$ and $\mN=\mM=\mC=\{3, 4\}, K=40$ on evaluation datasets with $\widetilde{\mN}=\widetilde{\mM}=4, \widetilde{\mC}=\{6\}$ and $\widetilde{K}=40$. }
    \label{fig:cot_recipes_train_N_3-4_M_3-4_C_3-4_K_40_bulk_eval_N_4_M_4_C_6_K_40}
    % \vspace{-4mm}
\end{figure*}

\paragraph{Input Length.} In Section~\ref{subsec:diversity_len_gen}, we have noticed that the \texttt{TF-12} model was able to leverage the diversity of $\mN=\mM=\mC=\{3,4\}$ to achieve good length generalization when $\widetilde{\mN}=\{5\}$. However, such improvements are not evident in the smaller \texttt{TF-4} model as shown in Figure~\ref{fig:cot_recipes_len_gen_bulk_eval_N_5_M_4_C_4_K_40_L_4}. On the other hand, the \texttt{TF-8} model trained with $\alpha=2$ shows a significant lift with both the forcing strategies (see Figure~\ref{fig:cot_recipes_len_gen_bulk_eval_N_5_M_4_C_4_K_40_L_8}), with `Force Answer' being the suitable strategy over `Force Think'. When the input length is increased to $\widetilde{\mN}=\{6\}$, the $\tacc$ goes down even further across model sizes (see Figure~\ref{fig:cot_recipes_train_N_4_M_4_C_4_K_40_bulk_eval_N_6_M_4_C_4_K_40}), but we still observe that `Force Answer' is a preferred strategy over `Force Think'.

\paragraph{Chain Length.} For chain length generalization experiments, we create the evaluation datasets by setting $\widetilde{\mN}=\{4\},\widetilde{\mC}=\{5\}$ and keeping the rest of the parameter choices the same as Section~\ref{subsec:diversity_len_gen}. For models trained with $\mN=\mM=\mC=\{4\}$, we observe from Figure~\ref{fig:cot_recipes_train_N_4_M_4_C_4_K_40_bulk_eval_N_4_M_4_C_5_K_40} that preparing evaluation prompts with $K'=39$ results in the best $\tacc$ across model sizes and forcing strategies. In particular, the gap between peak $\tacc$ at $K'=0$ and $L'=39$ reduces as model size increases (see also Figure~\ref{fig:cot_recipes_train_N_3-4_M_3-4_C_3-4_K_40_bulk_eval_N_4_M_4_C_6_K_40} for the case with $\widetilde{\mC}=\{6\}$). Furthermore, the robustness of the trained models to varying $K'$ with $\alpha=\{0.5,1,2\}$ increases with model size as well. Similar observations can be made for models trained with $\mN=\mM=\mC=\{3, 4\}$ in Figure~\ref{fig:cot_recipes_train_N_3-4_M_3-4_C_3-4_K_40_bulk_eval_N_4_M_4_C_5_K_40}. However, unlike the length generalization case, the diversity via $\mN,\mM,\mC$ does not increase the peak $\tacc$ across $\alpha$ but profoundly impacts the choice of $\alpha$. For instance, $\alpha=2$ is significantly better than other choices for the \texttt{TF-8} model (Figure~\ref{fig:cot_recipes_train_N_3-4_M_3-4_C_3-4_K_40_bulk_eval_N_4_M_4_C_5_K_40_L_8}). Nonetheless, such gaps tend to minimize as model size increases (Figure~\ref{fig:cot_recipes_train_N_3-4_M_3-4_C_3-4_K_40_bulk_eval_N_4_M_4_C_5_K_40_L_12}).

\section{Symbolic Reasoning with \texttt{CIL-LangSym}}
\label{app:cil_langsym_details}

% \begin{figure*}[h!]
%     \centering
%     \includegraphics[width=\linewidth]{images/cil_langsym.png}
%     \caption{Chat template of a CoT/standard example in \texttt{CIL-LangSym} based on the \texttt{Qwen-2.5-1.5B-Instruct} tokenizer. Given $N=4, M=2, C=3$ and word length $W=8$, the DAG determines the ground truth causal dependencies and the \texttt{transform} function illustrates the string processing of the $M$ parent words. We apply the chat template to differentiate the question, thinking and final answer segments of the CoT example. The standard example does not contain intermediate steps. }
%     \label{fig:cil_symbolic}
% \end{figure*}

Data generation algorithms for \texttt{CIL-LangSym} are similar in nature to that of~\coticlnew, but instead are applied to string inputs rather than token embeddings. Algorithm~\ref{alg:cil_langsym_dataset} presents the pseudo-code for generating the entire dataset, Algorithm~\ref{alg:cil_langsym_single_prompt} generates a single prompt formatted using the chat template, and Algorithm~\ref{alg:cil_langsym_single_chain_word} generates a single chain word for an in-context example. The \texttt{system\_prompt} along with the formatted questions after applying the \texttt{Qwen-2.5-7B-Instruct} tokenizer chat template is shown in Figure~\ref{fig:cil_langsym_example}. We also illustrate the underlying DAG structure while creating an in-context example in Figure~\ref{fig:cil_symbolic}.

% \subsection{Generate the Dataset of Sequences}

\begin{algorithm}
\caption{Generate dataset with $T$ prompts}
\begin{algorithmic}[1]
\label{alg:cil_langsym_dataset}
\REQUIRE Parameter choices $\mN, \mM, \mC, K$, word length $W$, the \texttt{CoT-Recipe} parameters $\alpha, a, b$, size $T$, \texttt{system\_prompt},  \texttt{question\_template}.
\STATE Store all input params in \texttt{args}.
\STATE Initialize empty dataset $\mD = []$
\FOR{$j = 1$ to $T$}
\STATE $r_{CoT}^{(j)} = \texttt{CoT-Recipe}(\alpha, a, b)(j/T)$
\STATE $\vp$ = Algorithm~\ref{alg:cil_langsym_single_prompt}$(\texttt{*args})$.
\STATE $\mD.\texttt{append}(\vp)$
\ENDFOR
\RETURN $\mD$
\end{algorithmic}
\end{algorithm}

% \subsection{Generate Single Sequence}

\begin{algorithm}
\caption{Single prompt generation with index $j$ in the dataset.}
\begin{algorithmic}[1]
\label{alg:cil_langsym_single_prompt}
\REQUIRE Parameter choices $\mN, \mM, \mC, K$, word length $W$, and the CoT probability $r_{CoT}^{(j)}$, \texttt{system\_prompt},  \texttt{question\_template}, 

\STATE sample $N \sim \mN, M \sim \mM, C \sim \mC$
\STATE Limit $M = \min(M, N)$.
\STATE Messages $\vm = [\texttt{``system"}, \texttt{system\_prompt}]$ 
\STATE \texttt{CHARS = list(string.ascii\_lowercase)}
\FOR{$k = 1$ to $K$}

\STATE Initialize empty input words list $\vx$, chain words list $\vy$.
\FOR{$i = 1$ to $N$}
    \STATE $\vx[i] \leftarrow \texttt{``".join(sample(CHARS, W))}$
\ENDFOR
\STATE $\vt = \vx.\texttt{clone()}$
\FOR{$c = 1$ to $C$}
    \STATE $\texttt{parent\_words}=\texttt{rand.choice}(\vt, M)$
    \STATE $\vy[c]$ = Algorithm~\ref{alg:cil_langsym_single_chain_word}(\texttt{parent\_tokens})
    \STATE $\vt.\texttt{append}(\vy[c])$
\ENDFOR
\STATE \texttt{q = fmt\_q(question\_template,$\vx$)}
\STATE \texttt{m.extend([`user', q])}
\IF{$r^{(j)}_{CoT} \ge \gU(0,1)$}
\STATE \texttt{a = fmt\_a($\vy$, CoT=True)}
\ENDIF
\STATE \texttt{a = fmt\_a($\vy$, CoT=False)}
\STATE \texttt{m.extend([`assistant', a])}
\ENDFOR
\STATE \texttt{p = chat\_template(m)}
\RETURN \texttt{p}
\end{algorithmic}
\end{algorithm}

% \subsection{Generate Single Chain Token}
\begin{algorithm}
\caption{Single \textit{chain word} generation}
\begin{algorithmic}[1]
\label{alg:cil_langsym_single_chain_word}
\REQUIRE $M$ \texttt{parent\_words}.
\STATE Initialize temporary list of slices $\vh$
\FOR{$i = 1$ to $M$}
    \STATE $w \leftarrow \texttt{parent\_words}[i]$
    \STATE $\vh[i] \leftarrow \texttt{w[len(w) // 2 :]}$
\ENDFOR
\STATE $o \leftarrow \texttt{concat}(\vh)$
\STATE $y_c \leftarrow \texttt{char\_offset}(o, 1)$
\RETURN $y_c$
\end{algorithmic}
\end{algorithm}

\begin{figure}[h!]
    \centering
    % First subfigure
    \begin{subfigure}[b]{\linewidth}
        \centering
        \includegraphics[width=\linewidth]{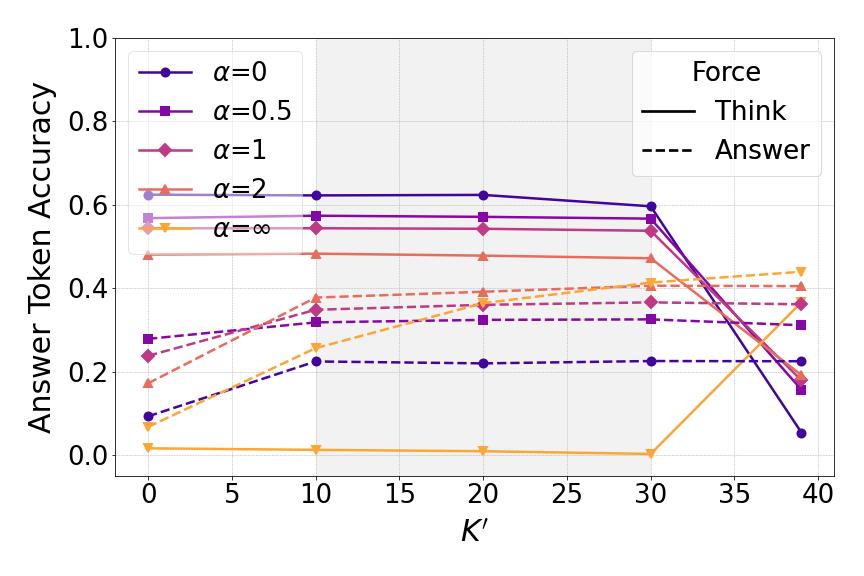}
        % \caption{$\widetilde{\mC}=\{4\}$}
        \caption{$\widetilde{\mN}=\{4\}, \widetilde{\mM}=\{2\}, \widetilde{\mC}=\{4\}$}
        \label{fig:cil_symbolic_cot_recipes_len_gen_7B_instruct_C_4}
    \end{subfigure}
    \hfill
    % Second subfigure
    \begin{subfigure}[b]{\linewidth}
        \centering
        \includegraphics[width=\linewidth]{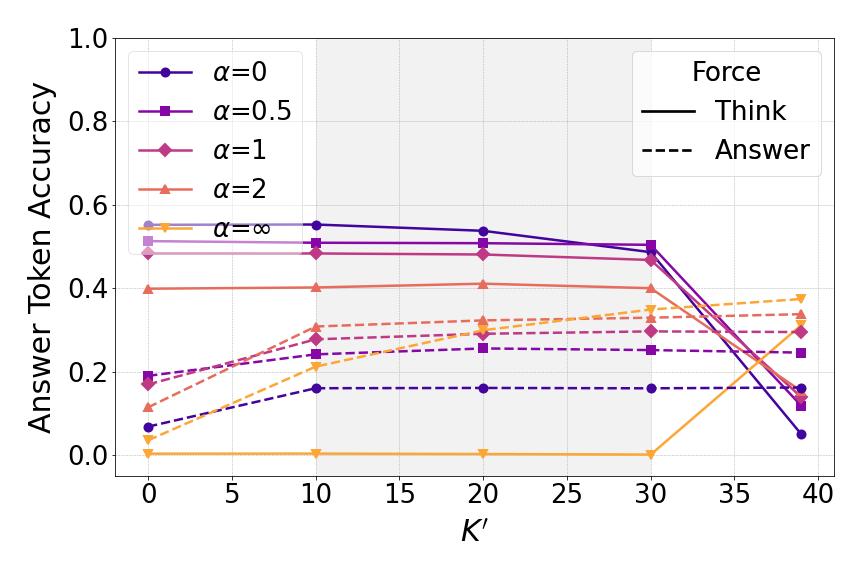}
        % \caption{$\widetilde{\mC}=\{5\}$}
        \caption{$\widetilde{\mN}=\{4\}, \widetilde{\mM}=\{2\}, \widetilde{\mC}=\{5\}$}
        \label{fig:cil_symbolic_cot_recipes_len_gen_7B_instruct_C_5}
    \end{subfigure}
    \caption{Chain length generalization: $\tacc$ of \texttt{Qwen-2.5-7B-Instruct} model trained with varying $\alpha$, $\mN=\{4\}, \mM=\{2\}, \mC=\{3\}$ and $K=40$ on evaluation datasets with longer chains. }
    \label{fig:cil_langsym_cot_recipes_chain_len_gen_7B_instruct}
    % \vspace{-4mm}
\end{figure}

\subsection{Measuring the Reliance on Intermediate `Thinking' Steps}
\label{app:cot_reliance}

As shown in Section~\ref{subsec:cil_langsym_exps}, there is a clear indication of the role of model size in attaining high $\tacc$ across various $\alpha$ and $K'$. In this section, we consider the model predictions for $K'=0$ with the `Force Think' strategy  and present a breakdown of the evaluation prompts based on the correctness of the intermediate `thinking' steps. Since \texttt{Qwen-2.5-7B-Instruct} with $\alpha=0$ was already analyzed in Section~\ref{subsec:cil_langsym_exps}, we focus on the remaining models and $\alpha$ below.

\paragraph{Qwen-2.5-0.5B-Instruct.} Based on the values plotted in Figure~\ref{fig:cil_langsym_qwen_0.5B_instruct_cot_recipes_N_4_M_2_C_3}, the $\tacc$ of the smaller SFT'ed \texttt{Qwen-2.5-0.5B-Instruct} model with $\alpha=0$ is $0.0044$. Out of the $44$ prompts for which the final answer was predicted correctly, Figure~\ref{fig:0.5B_instruct_alpha_0_step_breakdown} shows the model was able to correctly predict both the intermediate step outputs for only $4.5\%$ of those prompts. Surprisingly, in $72.7\%$ of the $44$ model predictions, both the intermediate step outputs were wrong. Similar observations can be made for $\alpha=0.5$ in Figure~\ref{fig:0.5B_instruct_alpha_0.5_step_breakdown}, $\alpha=1$ in Figure~\ref{fig:0.5B_instruct_alpha_1_step_breakdown}, $\alpha=2$ in Figure~\ref{fig:0.5B_instruct_alpha_2_step_breakdown}, and $\alpha=\infty$ in Figure~\ref{fig:0.5B_instruct_alpha_inf_step_breakdown}.

%=====0.5 B-=========

\begin{figure}[h!]
\centering
\begin{subfigure}{0.45\linewidth}
\centering
\begin{tikzpicture}[scale=1.2]
    % Left heatmap - Correct Final Answers
    \fill[low] (0,1) rectangle (1,2);    % 81.5%
    \fill[high] (1,1) rectangle (2,2);     % 7.8%
    \fill[low] (0,0) rectangle (1,1);     % 5.0%
    \fill[med] (1,0) rectangle (2,1);     % 5.6%

    \draw[thick] (0,0) rectangle (2,2);
    \draw[thick] (0,1) -- (2,1);
    \draw[thick] (1,0) -- (1,2);

   % fill here
\node at (0.5,0.5) {\textbf{4.5\%}}; % cc  
\node at (1.5,1.5) {\textbf{72.7\%}}; % ii  
\node at (1.5,0.5) {\textbf{13.6\%}}; % ic  
\node at (0.5,1.5) {\textbf{9.1\%}};  % ci

    % Axis labels
    \node[anchor=center] at (0.5,-0.25) {\footnotesize Correct};
    \node[anchor=center] at (1.5,-0.25) {\footnotesize Incorrect};
    \node[anchor=center, rotate=90] at (-0.25,0.5) {\footnotesize Correct};
    \node[anchor=center, rotate=90] at (-0.25,1.5) {\footnotesize Incorrect};

    % Step labels
    \node at (1,2.3) {\small Step1};
    \node[rotate=90] at (-0.6,1) {\small Step2};
\end{tikzpicture}
\caption{\checkmark Answer}
\end{subfigure}
\begin{subfigure}{0.45\linewidth}
\centering
\begin{tikzpicture}[scale=1.2]
    % Right heatmap - Incorrect Final Answers
    \fill[low] (0,1) rectangle (1,2);     % 30.5%
    \fill[high] (1,1) rectangle (2,2);     % 22.4%
    \fill[low] (0,0) rectangle (1,1);     % 11.6%
    \fill[low] (1,0) rectangle (2,1);    % 35.4%

    \draw[thick] (0,0) rectangle (2,2);
    \draw[thick] (0,1) -- (2,1);
    \draw[thick] (1,0) -- (1,2);

   %fill here    
\node at (0.5,0.5) {\textbf{0.0\%}};  % cc  
\node at (1.5,1.5) {\textbf{99.4\%}}; % ii  
\node at (1.5,0.5) {\textbf{0.4\%}};  % ic  
\node at (0.5,1.5) {\textbf{0.2\%}};  % ci

 % Axis labels
    \node[anchor=center] at (0.5,-0.25) {\footnotesize Correct};
    \node[anchor=center] at (1.5,-0.25) {\footnotesize Incorrect};
    \node[anchor=center, rotate=90] at (-0.25,0.5) {\footnotesize Correct};
    \node[anchor=center, rotate=90] at (-0.25,1.5) {\footnotesize Incorrect};

    % Step labels
    \node at (1,2.3) {\small Step1};
    \node[rotate=90] at (-0.6,1) {\small Step2};
\end{tikzpicture}
\caption{$\times$ Answer}
\end{subfigure}
\caption{Trained Qwen2.5 0.5 with $\alpha=0$ achieves only $0.44\%$ final answer accuracy. 
% For both correct and incorrect final answers, the model achieves as low as $4.5\%$ and $0\%$ accuracy for the intermediate tokens over both steps respectively.
}
\label{fig:0.5B_instruct_alpha_0_step_breakdown}
\end{figure}
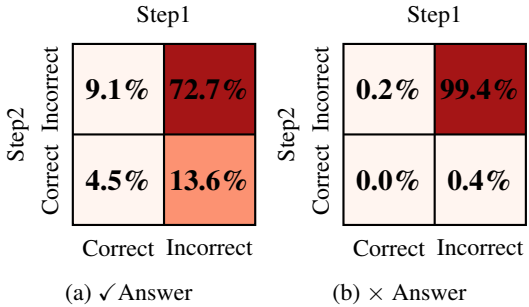

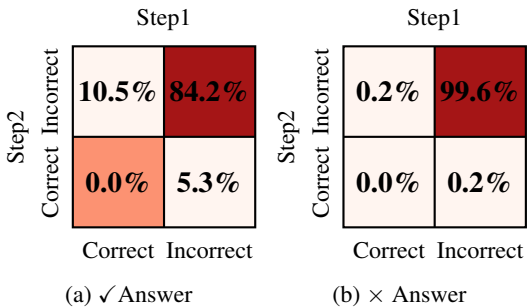
\begin{figure}[h!]
\centering
\begin{subfigure}{0.45\linewidth}
\centering
\begin{tikzpicture}[scale=1.2]
    % Left heatmap - Correct Final Answers
    \fill[low] (0,1) rectangle (1,2);    % 81.5%
    \fill[high] (1,1) rectangle (2,2);     % 7.8%
    \fill[med] (0,0) rectangle (1,1);     % 5.0%
    \fill[low] (1,0) rectangle (2,1);     % 5.6%

    \draw[thick] (0,0) rectangle (2,2);
    \draw[thick] (0,1) -- (2,1);
    \draw[thick] (1,0) -- (1,2);

   % fill here
\node at (0.5,0.5) {\textbf{0.0\%}};  % cc  
\node at (1.5,1.5) {\textbf{84.2\%}}; % ii  
\node at (1.5,0.5) {\textbf{5.3\%}};  % ic  
\node at (0.5,1.5) {\textbf{10.5\%}}; % ci

    % Axis labels
    \node[anchor=center] at (0.5,-0.25) {\footnotesize Correct};
    \node[anchor=center] at (1.5,-0.25) {\footnotesize Incorrect};
    \node[anchor=center, rotate=90] at (-0.25,0.5) {\footnotesize Correct};
    \node[anchor=center, rotate=90] at (-0.25,1.5) {\footnotesize Incorrect};

    % Step labels
    \node at (1,2.3) {\small Step1};
    \node[rotate=90] at (-0.6,1) {\small Step2};
\end{tikzpicture}
\caption{\checkmark Answer}
\end{subfigure}
\begin{subfigure}{0.45\linewidth}
\centering
\begin{tikzpicture}[scale=1.2]
    % Right heatmap - Incorrect Final Answers
    \fill[low] (0,1) rectangle (1,2);     % 30.5%
    \fill[high] (1,1) rectangle (2,2);     % 22.4%
    \fill[low] (0,0) rectangle (1,1);     % 11.6%
    \fill[low] (1,0) rectangle (2,1);    % 35.4%

    \draw[thick] (0,0) rectangle (2,2);
    \draw[thick] (0,1) -- (2,1);
    \draw[thick] (1,0) -- (1,2);

   %fill here    
\node at (0.5,0.5) {\textbf{0.0\%}};  % cc  
\node at (1.5,1.5) {\textbf{99.6\%}}; % ii  
\node at (1.5,0.5) {\textbf{0.2\%}};  % ic  
\node at (0.5,1.5) {\textbf{0.2\%}};  % ci

 % Axis labels
    \node[anchor=center] at (0.5,-0.25) {\footnotesize Correct};
    \node[anchor=center] at (1.5,-0.25) {\footnotesize Incorrect};
    \node[anchor=center, rotate=90] at (-0.25,0.5) {\footnotesize Correct};
    \node[anchor=center, rotate=90] at (-0.25,1.5) {\footnotesize Incorrect};

    % Step labels
    \node at (1,2.3) {\small Step1};
    \node[rotate=90] at (-0.6,1) {\small Step2};
\end{tikzpicture}
\caption{$\times$ Answer}
\end{subfigure}
\caption{Trained Qwen2.5 0.5 with $\alpha=0.5$ achieves only $0.19\%$ final answer accuracy. 
% For both correct and incorrect final answers, the model achieves $0.0\%$ accuracy for the intermediate tokens over both steps.
}
\label{fig:0.5B_instruct_alpha_0.5_step_breakdown}
\end{figure}

\begin{figure}[h!]
\centering
\begin{subfigure}{0.45\linewidth}
\centering
\begin{tikzpicture}[scale=1.2]
    % Left heatmap - Correct Final Answers
    \fill[med] (0,1) rectangle (1,2);    % 81.5%
    \fill[high] (1,1) rectangle (2,2);     % 7.8%
    \fill[low] (0,0) rectangle (1,1);     % 5.0%
    \fill[low] (1,0) rectangle (2,1);     % 5.6%

    \draw[thick] (0,0) rectangle (2,2);
    \draw[thick] (0,1) -- (2,1);
    \draw[thick] (1,0) -- (1,2);

   % fill here
\node at (0.5,0.5) {\textbf{5.9\%}};  % cc  
\node at (1.5,1.5) {\textbf{58.8\%}}; % ii  
\node at (1.5,0.5) {\textbf{29.4\%}}; % ic  
\node at (0.5,1.5) {\textbf{5.9\%}};  % ci
    % Axis labels
    \node[anchor=center] at (0.5,-0.25) {\footnotesize Correct};
    \node[anchor=center] at (1.5,-0.25) {\footnotesize Incorrect};
    \node[anchor=center, rotate=90] at (-0.25,0.5) {\footnotesize Correct};
    \node[anchor=center, rotate=90] at (-0.25,1.5) {\footnotesize Incorrect};

    % Step labels
    \node at (1,2.3) {\small Step1};
    \node[rotate=90] at (-0.6,1) {\small Step2};
\end{tikzpicture}
\caption{\checkmark Answer}
\end{subfigure}
\begin{subfigure}{0.45\linewidth}
\centering
\begin{tikzpicture}[scale=1.2]
    % Right heatmap - Incorrect Final Answers
    \fill[low] (0,1) rectangle (1,2);     % 30.5%
    \fill[high] (1,1) rectangle (2,2);     % 22.4%
    \fill[low] (0,0) rectangle (1,1);     % 11.6%
    \fill[low] (1,0) rectangle (2,1);    % 35.4%

    \draw[thick] (0,0) rectangle (2,2);
    \draw[thick] (0,1) -- (2,1);
    \draw[thick] (1,0) -- (1,2);

   %fill here    
\node at (0.5,0.5) {\textbf{0.0\%}};  % cc  
\node at (1.5,1.5) {\textbf{99.9\%}}; % ii  
\node at (1.5,0.5) {\textbf{0.1\%}};  % ic  
\node at (0.5,1.5) {\textbf{0.0\%}};  % ci

 % Axis labels
    \node[anchor=center] at (0.5,-0.25) {\footnotesize Correct};
    \node[anchor=center] at (1.5,-0.25) {\footnotesize Incorrect};
    \node[anchor=center, rotate=90] at (-0.25,0.5) {\footnotesize Correct};
    \node[anchor=center, rotate=90] at (-0.25,1.5) {\footnotesize Incorrect};

    % Step labels
    \node at (1,2.3) {\small Step1};
    \node[rotate=90] at (-0.6,1) {\small Step2};
\end{tikzpicture}
\caption{$\times$ Answer}
\end{subfigure}
\caption{Trained Qwen2.5 0.5 with $\alpha=1$ achieves only $0.17\%$ final answer accuracy. 
% For both correct and incorrect final answers, the model achieves as low as $5.9\%$ and $0.0\%$ accuracy for the intermediate tokens over both steps respectively.
}
\label{fig:0.5B_instruct_alpha_1_step_breakdown}
\end{figure}
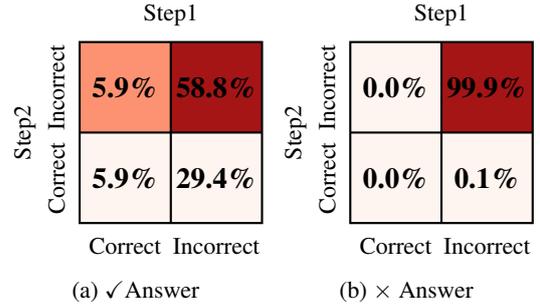

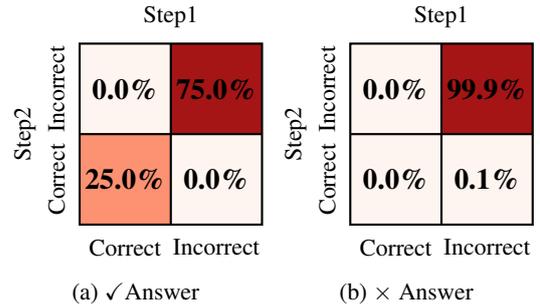
\begin{figure}[h!]
\centering
\begin{subfigure}{0.45\linewidth}
\centering
\begin{tikzpicture}[scale=1.2]
    % Left heatmap - Correct Final Answers
    \fill[low] (0,1) rectangle (1,2);    % 81.5%
    \fill[high] (1,1) rectangle (2,2);     % 7.8%
    \fill[med] (0,0) rectangle (1,1);     % 5.0%
    \fill[low] (1,0) rectangle (2,1);     % 5.6%

    \draw[thick] (0,0) rectangle (2,2);
    \draw[thick] (0,1) -- (2,1);
    \draw[thick] (1,0) -- (1,2);

   % fill here
\node at (0.5,0.5) {\textbf{25.0\%}}; % cc  
\node at (1.5,1.5) {\textbf{75.0\%}}; % ii  
\node at (1.5,0.5) {\textbf{0.0\%}};  % ic  
\node at (0.5,1.5) {\textbf{0.0\%}};  % ci

    % Axis labels
    \node[anchor=center] at (0.5,-0.25) {\footnotesize Correct};
    \node[anchor=center] at (1.5,-0.25) {\footnotesize Incorrect};
    \node[anchor=center, rotate=90] at (-0.25,0.5) {\footnotesize Correct};
    \node[anchor=center, rotate=90] at (-0.25,1.5) {\footnotesize Incorrect};

    % Step labels
    \node at (1,2.3) {\small Step1};
    \node[rotate=90] at (-0.6,1) {\small Step2};
\end{tikzpicture}
\caption{\checkmark Answer}
\end{subfigure}
\begin{subfigure}{0.45\linewidth}
\centering
\begin{tikzpicture}[scale=1.2]
    % Right heatmap - Incorrect Final Answers
    \fill[low] (0,1) rectangle (1,2);     % 30.5%
    \fill[high] (1,1) rectangle (2,2);     % 22.4%
    \fill[low] (0,0) rectangle (1,1);     % 11.6%
    \fill[low] (1,0) rectangle (2,1);    % 35.4%

    \draw[thick] (0,0) rectangle (2,2);
    \draw[thick] (0,1) -- (2,1);
    \draw[thick] (1,0) -- (1,2);

   %fill here 
\node at (0.5,0.5) {\textbf{0.0\%}};  % cc  
\node at (1.5,1.5) {\textbf{99.9\%}}; % ii  
\node at (1.5,0.5) {\textbf{0.1\%}};  % ic  
\node at (0.5,1.5) {\textbf{0.0\%}};  % ci

 % Axis labels
    \node[anchor=center] at (0.5,-0.25) {\footnotesize Correct};
    \node[anchor=center] at (1.5,-0.25) {\footnotesize Incorrect};
    \node[anchor=center, rotate=90] at (-0.25,0.5) {\footnotesize Correct};
    \node[anchor=center, rotate=90] at (-0.25,1.5) {\footnotesize Incorrect};

    % Step labels
    \node at (1,2.3) {\small Step1};
    \node[rotate=90] at (-0.6,1) {\small Step2};
\end{tikzpicture}
\caption{$\times$ Answer}
\end{subfigure}
\caption{Trained Qwen2.5 0.5 with $\alpha=2$ achieves only $0.04\%$ final answer accuracy. 
% For both correct and incorrect final answers, the model achieves $25\%$ and $0\%$ accuracy for the intermediate tokens over both steps.
}
\label{fig:0.5B_instruct_alpha_2_step_breakdown}
\end{figure}

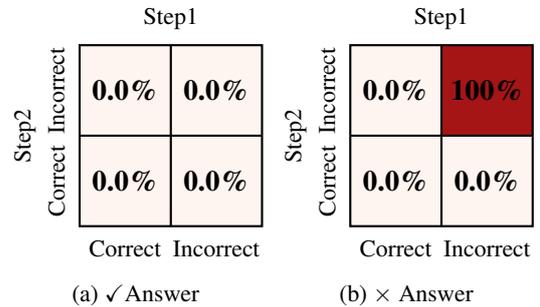
\begin{figure}[h!]
\centering
\begin{subfigure}{0.45\linewidth}
\centering
\begin{tikzpicture}[scale=1.2]
    % Left heatmap - Correct Final Answers
    \fill[low] (0,1) rectangle (1,2);    % 81.5%
    \fill[low] (1,1) rectangle (2,2);     % 7.8%
    \fill[low] (0,0) rectangle (1,1);     % 5.0%
    \fill[low] (1,0) rectangle (2,1);     % 5.6%

    \draw[thick] (0,0) rectangle (2,2);
    \draw[thick] (0,1) -- (2,1);
    \draw[thick] (1,0) -- (1,2);

   % fill here
\node at (0.5,0.5) {\textbf{0.0\%}}; % cc  
\node at (1.5,1.5) {\textbf{0.0\%}}; % ii  
\node at (1.5,0.5) {\textbf{0.0\%}}; % ic  
\node at (0.5,1.5) {\textbf{0.0\%}}; % ci

    % Axis labels
    \node[anchor=center] at (0.5,-0.25) {\footnotesize Correct};
    \node[anchor=center] at (1.5,-0.25) {\footnotesize Incorrect};
    \node[anchor=center, rotate=90] at (-0.25,0.5) {\footnotesize Correct};
    \node[anchor=center, rotate=90] at (-0.25,1.5) {\footnotesize Incorrect};

    % Step labels
    \node at (1,2.3) {\small Step1};
    \node[rotate=90] at (-0.6,1) {\small Step2};
\end{tikzpicture}
\caption{\checkmark Answer}
\end{subfigure}
\begin{subfigure}{0.45\linewidth}
\centering
\begin{tikzpicture}[scale=1.2]
    % Right heatmap - Incorrect Final Answers
    \fill[low] (0,1) rectangle (1,2);     % 30.5%
    \fill[high] (1,1) rectangle (2,2);     % 22.4%
    \fill[low] (0,0) rectangle (1,1);     % 11.6%
    \fill[low] (1,0) rectangle (2,1);    % 35.4%

    \draw[thick] (0,0) rectangle (2,2);
    \draw[thick] (0,1) -- (2,1);
    \draw[thick] (1,0) -- (1,2);

   %fill here 
\node at (0.5,0.5) {\textbf{0.0\%}};  % cc  
\node at (1.5,1.5) {\textbf{100\%}}; % ii  
\node at (1.5,0.5) {\textbf{0.0\%}};  % ic  
\node at (0.5,1.5) {\textbf{0.0\%}};  % ci

 % Axis labels
    \node[anchor=center] at (0.5,-0.25) {\footnotesize Correct};
    \node[anchor=center] at (1.5,-0.25) {\footnotesize Incorrect};
    \node[anchor=center, rotate=90] at (-0.25,0.5) {\footnotesize Correct};
    \node[anchor=center, rotate=90] at (-0.25,1.5) {\footnotesize Incorrect};

    % Step labels
    \node at (1,2.3) {\small Step1};
    \node[rotate=90] at (-0.6,1) {\small Step2};
\end{tikzpicture}
\caption{$\times$ Answer}
\end{subfigure}
\caption{Trained Qwen2.5 0.5B with $\alpha=\infty$ achieves  $0.0\%$ final answer accuracy.
% For both correct and incorrect final answers, the model fails at predicting any intermediate steps correctly.
}
\label{fig:0.5B_instruct_alpha_inf_step_breakdown}
\end{figure}

\paragraph{Qwen-2.5-1.5B-Instruct.} As the model size increases, Figure~\ref{fig:cil_langsym_qwen_1.5B_instruct_cot_recipes_N_4_M_2_C_3} shows that the $\tacc$ of the SFT'ed \texttt{Qwen-2.5-1.5B-Instruct} model with $\alpha=0$ is $0.1551$. Figure~\ref{fig:1.5B_instruct_alpha_0_step_breakdown} shows that only $36.5\%$ of the $1551$ prompts that have correct final answers have both the intermediate step predictions to be correct. Although this fraction is higher than the \texttt{0.5B-Instruct} case in Figure~\ref{fig:0.5B_instruct_alpha_0_step_breakdown}, it is still a considerably low fraction. On the other hand, when the model produces incorrect final answers, we can observe a relatively higher importance of the Step 2 predictions compared to Step 1. See also Figure~\ref{fig:1.5B_instruct_alpha_0.5_step_breakdown} ($\alpha=0.5$), Figure~\ref{fig:1.5B_instruct_alpha_1_step_breakdown} ($\alpha=1$), Figure~\ref{fig:1.5B_instruct_alpha_2_step_breakdown} ($\alpha=2$) and Figure~\ref{fig:1.5B_instruct_alpha_inf_step_breakdown} ($\alpha=\infty$).

\begin{figure}[h]
\centering
\begin{subfigure}{0.45\linewidth}
\centering
\begin{tikzpicture}[scale=1.2]
    % Left heatmap - Correct Final Answers
    \fill[med] (0,1) rectangle (1,2);    % 81.5%
    \fill[med] (1,1) rectangle (2,2);     % 7.8%
    \fill[high] (0,0) rectangle (1,1);     % 5.0%
    \fill[low] (1,0) rectangle (2,1);     % 5.6%

    \draw[thick] (0,0) rectangle (2,2);
    \draw[thick] (0,1) -- (2,1);
    \draw[thick] (1,0) -- (1,2);

   % fill here
\node at (0.5,0.5) {\textbf{36.5\%}}; % cc  
\node at (1.5,1.5) {\textbf{30.0\%}}; % ii  
\node at (1.5,0.5) {\textbf{13.7\%}}; % ic  
\node at (0.5,1.5) {\textbf{19.8\%}}; % ci

    % Axis labels
    \node[anchor=center] at (0.5,-0.25) {\footnotesize Correct};
    \node[anchor=center] at (1.5,-0.25) {\footnotesize Incorrect};
    \node[anchor=center, rotate=90] at (-0.25,0.5) {\footnotesize Correct};
    \node[anchor=center, rotate=90] at (-0.25,1.5) {\footnotesize Incorrect};

    % Step labels
    \node at (1,2.3) {\small Step1};
    \node[rotate=90] at (-0.6,1) {\small Step2};
\end{tikzpicture}
\caption{\checkmark Answer}
\end{subfigure}
\begin{subfigure}{0.45\linewidth}
\centering
\begin{tikzpicture}[scale=1.2]
    % Right heatmap - Incorrect Final Answers
    \fill[med] (0,1) rectangle (1,2);     % 30.5%
    \fill[high] (1,1) rectangle (2,2);     % 22.4%
    \fill[low] (0,0) rectangle (1,1);     % 11.6%
    \fill[low] (1,0) rectangle (2,1);    % 35.4%

    \draw[thick] (0,0) rectangle (2,2);
    \draw[thick] (0,1) -- (2,1);
    \draw[thick] (1,0) -- (1,2);

   %fill here    
\node at (0.5,0.5) {\textbf{6.5\%}};  % cc  
\node at (1.5,1.5) {\textbf{73.3\%}}; % ii  
\node at (1.5,0.5) {\textbf{7.7\%}};  % ic  
\node at (0.5,1.5) {\textbf{12.5\%}}; % ci

 % Axis labels
    \node[anchor=center] at (0.5,-0.25) {\footnotesize Correct};
    \node[anchor=center] at (1.5,-0.25) {\footnotesize Incorrect};
    \node[anchor=center, rotate=90] at (-0.25,0.5) {\footnotesize Correct};
    \node[anchor=center, rotate=90] at (-0.25,1.5) {\footnotesize Incorrect};

    % Step labels
    \node at (1,2.3) {\small Step1};
    \node[rotate=90] at (-0.6,1) {\small Step2};
\end{tikzpicture}
\caption{$\times$ Answer}
\end{subfigure}
\caption{Trained Qwen2.5 1.5B with $\alpha=0$ achieves $15.51\%$ final answer accuracy.
% For both correct and incorrect final answers, the model achieves $36.5\%$ and $6.5\%$ accuracy for the intermediate tokens over both steps respectively.
}
\label{fig:1.5B_instruct_alpha_0_step_breakdown}
\end{figure}
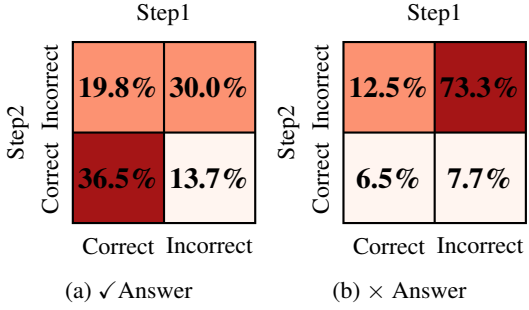

\begin{figure}[h]
\centering
\begin{subfigure}{0.45\linewidth}
\centering
\begin{tikzpicture}[scale=1.2]
    % Left heatmap - Correct Final Answers
    \fill[med] (0,1) rectangle (1,2);    % 81.5%
    \fill[med] (1,1) rectangle (2,2);     % 7.8%
    \fill[med] (0,0) rectangle (1,1);     % 5.0%
    \fill[low] (1,0) rectangle (2,1);     % 5.6%

    \draw[thick] (0,0) rectangle (2,2);
    \draw[thick] (0,1) -- (2,1);
    \draw[thick] (1,0) -- (1,2);

   % fill here
\node at (0.5,0.5) {\textbf{31.1\%}}; % cc  
\node at (1.5,1.5) {\textbf{33.4\%}}; % ii  
\node at (1.5,0.5) {\textbf{14.6\%}}; % ic  
\node at (0.5,1.5) {\textbf{20.9\%}}; % ci

    % Axis labels
    \node[anchor=center] at (0.5,-0.25) {\footnotesize Correct};
    \node[anchor=center] at (1.5,-0.25) {\footnotesize Incorrect};
    \node[anchor=center, rotate=90] at (-0.25,0.5) {\footnotesize Correct};
    \node[anchor=center, rotate=90] at (-0.25,1.5) {\footnotesize Incorrect};

    % Step labels
    \node at (1,2.3) {\small Step1};
    \node[rotate=90] at (-0.6,1) {\small Step2};
\end{tikzpicture}
\caption{\checkmark Answer}
\end{subfigure}
\begin{subfigure}{0.45\linewidth}
\centering
\begin{tikzpicture}[scale=1.2]
    % Right heatmap - Incorrect Final Answers
    \fill[med] (0,1) rectangle (1,2);     % 30.5%
    \fill[high] (1,1) rectangle (2,2);     % 22.4%
    \fill[low] (0,0) rectangle (1,1);     % 11.6%
    \fill[low] (1,0) rectangle (2,1);    % 35.4%

    \draw[thick] (0,0) rectangle (2,2);
    \draw[thick] (0,1) -- (2,1);
    \draw[thick] (1,0) -- (1,2);

   %fill here    
\node at (0.5,0.5) {\textbf{4.8\%}};  % cc  
\node at (1.5,1.5) {\textbf{78.0\%}}; % ii  
\node at (1.5,0.5) {\textbf{6.2\%}};  % ic  
\node at (0.5,1.5) {\textbf{11.0\%}}; % ci

 % Axis labels
    \node[anchor=center] at (0.5,-0.25) {\footnotesize Correct};
    \node[anchor=center] at (1.5,-0.25) {\footnotesize Incorrect};
    \node[anchor=center, rotate=90] at (-0.25,0.5) {\footnotesize Correct};
    \node[anchor=center, rotate=90] at (-0.25,1.5) {\footnotesize Incorrect};

    % Step labels
    \node at (1,2.3) {\small Step1};
    \node[rotate=90] at (-0.6,1) {\small Step2};
\end{tikzpicture}
\caption{$\times$ Answer}
\end{subfigure}
\caption{Trained Qwen2.5 1.5B with $\alpha=0.5$ achieves $11.75\%$ final answer accuracy.
% For both correct and incorrect final answers, the model achieves $31.1\%$ and $4.8\%$ accuracy for the intermediate tokens over both steps respectively.
}
\label{fig:1.5B_instruct_alpha_0.5_step_breakdown}
\end{figure}

\begin{figure}[h]
\centering
\begin{subfigure}{0.45\linewidth}
\centering
\begin{tikzpicture}[scale=1.2]
    % Left heatmap - Correct Final Answers
    \fill[med] (0,1) rectangle (1,2);    % 81.5%
    \fill[high] (1,1) rectangle (2,2);     % 7.8%
    \fill[med] (0,0) rectangle (1,1);     % 5.0%
    \fill[med] (1,0) rectangle (2,1);     % 5.6%

    \draw[thick] (0,0) rectangle (2,2);
    \draw[thick] (0,1) -- (2,1);
    \draw[thick] (1,0) -- (1,2);

   % fill here
\node at (0.5,0.5) {\textbf{15.6\%}}; % cc  
\node at (1.5,1.5) {\textbf{51.7\%}}; % ii  
\node at (1.5,0.5) {\textbf{15.6\%}}; % ic  
\node at (0.5,1.5) {\textbf{17.1\%}}; % ci

    % Axis labels
    \node[anchor=center] at (0.5,-0.25) {\footnotesize Correct};
    \node[anchor=center] at (1.5,-0.25) {\footnotesize Incorrect};
    \node[anchor=center, rotate=90] at (-0.25,0.5) {\footnotesize Correct};
    \node[anchor=center, rotate=90] at (-0.25,1.5) {\footnotesize Incorrect};

    % Step labels
    \node at (1,2.3) {\small Step1};
    \node[rotate=90] at (-0.6,1) {\small Step2};
\end{tikzpicture}
\caption{\checkmark Answer}
\end{subfigure}
\begin{subfigure}{0.45\linewidth}
\centering
\begin{tikzpicture}[scale=1.2]
    % Right heatmap - Incorrect Final Answers
    \fill[low] (0,1) rectangle (1,2);     % 30.5%
    \fill[high] (1,1) rectangle (2,2);     % 22.4%
    \fill[low] (0,0) rectangle (1,1);     % 11.6%
    \fill[low] (1,0) rectangle (2,1);    % 35.4%

    \draw[thick] (0,0) rectangle (2,2);
    \draw[thick] (0,1) -- (2,1);
    \draw[thick] (1,0) -- (1,2);

   %fill here    
\node at (0.5,0.5) {\textbf{1.2\%}};  % cc  
\node at (1.5,1.5) {\textbf{91.9\%}}; % ii  
\node at (1.5,0.5) {\textbf{2.9\%}};  % ic  
\node at (0.5,1.5) {\textbf{3.9\%}};  % ci

 % Axis labels
    \node[anchor=center] at (0.5,-0.25) {\footnotesize Correct};
    \node[anchor=center] at (1.5,-0.25) {\footnotesize Incorrect};
    \node[anchor=center, rotate=90] at (-0.25,0.5) {\footnotesize Correct};
    \node[anchor=center, rotate=90] at (-0.25,1.5) {\footnotesize Incorrect};

    % Step labels
    \node at (1,2.3) {\small Step1};
    \node[rotate=90] at (-0.6,1) {\small Step2};
\end{tikzpicture}
\caption{$\times$ Answer}
\end{subfigure}
\caption{Trained Qwen2.5 1.5B with $\alpha=1$ achieves $4.74\%$ final answer accuracy.
% For both correct and incorrect final answers, the model achieves $15.6\%$ and $1.2\%$ accuracy for the intermediate tokens over both steps respectively.
}
\label{fig:1.5B_instruct_alpha_1_step_breakdown}
\end{figure}

\begin{figure}[h]
\centering
\begin{subfigure}{0.45\linewidth}
\centering
\begin{tikzpicture}[scale=1.2]
    % Left heatmap - Correct Final Answers
    \fill[med] (0,1) rectangle (1,2);    % 81.5%
    \fill[high] (1,1) rectangle (2,2);     % 7.8%
    \fill[med] (0,0) rectangle (1,1);     % 5.0%
    \fill[med] (1,0) rectangle (2,1);     % 5.6%

    \draw[thick] (0,0) rectangle (2,2);
    \draw[thick] (0,1) -- (2,1);
    \draw[thick] (1,0) -- (1,2);

   % fill here
\node at (0.5,0.5) {\textbf{18.3\%}}; % cc  
\node at (1.5,1.5) {\textbf{48.7\%}}; % ii  
\node at (1.5,0.5) {\textbf{15.0\%}}; % ic  
\node at (0.5,1.5) {\textbf{17.9\%}}; % ci

    % Axis labels
    \node[anchor=center] at (0.5,-0.25) {\footnotesize Correct};
    \node[anchor=center] at (1.5,-0.25) {\footnotesize Incorrect};
    \node[anchor=center, rotate=90] at (-0.25,0.5) {\footnotesize Correct};
    \node[anchor=center, rotate=90] at (-0.25,1.5) {\footnotesize Incorrect};

    % Step labels
    \node at (1,2.3) {\small Step1};
    \node[rotate=90] at (-0.6,1) {\small Step2};
\end{tikzpicture}
\caption{\checkmark Answer}
\end{subfigure}
\begin{subfigure}{0.45\linewidth}
\centering
\begin{tikzpicture}[scale=1.2]
    % Right heatmap - Incorrect Final Answers
    \fill[low] (0,1) rectangle (1,2);     % 30.5%
    \fill[high] (1,1) rectangle (2,2);     % 22.4%
    \fill[low] (0,0) rectangle (1,1);     % 11.6%
    \fill[low] (1,0) rectangle (2,1);    % 35.4%

    \draw[thick] (0,0) rectangle (2,2);
    \draw[thick] (0,1) -- (2,1);
    \draw[thick] (1,0) -- (1,2);

   %fill here    
\node at (0.5,0.5) {\textbf{0.9\%}};  % cc  
\node at (1.5,1.5) {\textbf{93.7\%}}; % ii  
\node at (1.5,0.5) {\textbf{2.2\%}};  % ic  
\node at (0.5,1.5) {\textbf{3.2\%}};  % ci

 % Axis labels
    \node[anchor=center] at (0.5,-0.25) {\footnotesize Correct};
    \node[anchor=center] at (1.5,-0.25) {\footnotesize Incorrect};
    \node[anchor=center, rotate=90] at (-0.25,0.5) {\footnotesize Correct};
    \node[anchor=center, rotate=90] at (-0.25,1.5) {\footnotesize Incorrect};

    % Step labels
    \node at (1,2.3) {\small Step1};
    \node[rotate=90] at (-0.6,1) {\small Step2};
\end{tikzpicture}
\caption{$\times$ Answer}
\end{subfigure}
\caption{Trained Qwen2.5 1.5B with $\alpha=2$ achieves $2.73\%$ final answer accuracy.
% For both correct and incorrect final answers, the model achieves $18.3\%$ and $0.9\%$ accuracy for the intermediate tokens over both steps respectively.
}
\label{fig:1.5B_instruct_alpha_2_step_breakdown}
\end{figure}
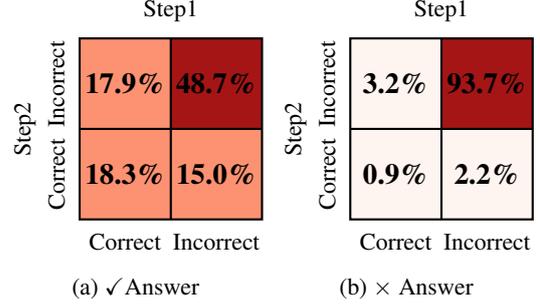

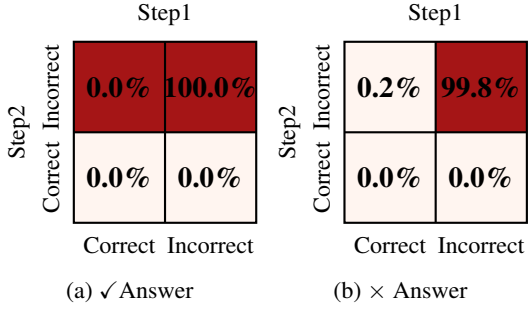
\begin{figure}[h]
\centering
\begin{subfigure}{0.45\linewidth}
\centering
\begin{tikzpicture}[scale=1.2]
    % Left heatmap - Correct Final Answers
    \fill[high] (0,1) rectangle (1,2);    % 81.5%
    \fill[high] (1,1) rectangle (2,2);     % 7.8%
    \fill[low] (0,0) rectangle (1,1);     % 5.0%
    \fill[low] (1,0) rectangle (2,1);     % 5.6%

    \draw[thick] (0,0) rectangle (2,2);
    \draw[thick] (0,1) -- (2,1);
    \draw[thick] (1,0) -- (1,2);

   % fill here
\node at (0.5,0.5) {\textbf{0.0\%}};   % cc  
\node at (1.5,1.5) {\textbf{100.0\%}}; % ii  
\node at (1.5,0.5) {\textbf{0.0\%}};   % ic  
\node at (0.5,1.5) {\textbf{0.0\%}};   % ci

    % Axis labels
    \node[anchor=center] at (0.5,-0.25) {\footnotesize Correct};
    \node[anchor=center] at (1.5,-0.25) {\footnotesize Incorrect};
    \node[anchor=center, rotate=90] at (-0.25,0.5) {\footnotesize Correct};
    \node[anchor=center, rotate=90] at (-0.25,1.5) {\footnotesize Incorrect};

    % Step labels
    \node at (1,2.3) {\small Step1};
    \node[rotate=90] at (-0.6,1) {\small Step2};
\end{tikzpicture}
\caption{\checkmark Answer}
\end{subfigure}
\begin{subfigure}{0.45\linewidth}
\centering
\begin{tikzpicture}[scale=1.2]
    % Right heatmap - Incorrect Final Answers
    \fill[low] (0,1) rectangle (1,2);     % 30.5%
    \fill[high] (1,1) rectangle (2,2);     % 22.4%
    \fill[low] (0,0) rectangle (1,1);     % 11.6%
    \fill[low] (1,0) rectangle (2,1);    % 35.4%

    \draw[thick] (0,0) rectangle (2,2);
    \draw[thick] (0,1) -- (2,1);
    \draw[thick] (1,0) -- (1,2);

   %fill here    
\node at (0.5,0.5) {\textbf{0.0\%}};   % cc  
\node at (1.5,1.5) {\textbf{99.8\%}};  % ii  
\node at (1.5,0.5) {\textbf{0.0\%}};   % ic  
\node at (0.5,1.5) {\textbf{0.2\%}};   % ci

 % Axis labels
    \node[anchor=center] at (0.5,-0.25) {\footnotesize Correct};
    \node[anchor=center] at (1.5,-0.25) {\footnotesize Incorrect};
    \node[anchor=center, rotate=90] at (-0.25,0.5) {\footnotesize Correct};
    \node[anchor=center, rotate=90] at (-0.25,1.5) {\footnotesize Incorrect};

    % Step labels
    \node at (1,2.3) {\small Step1};
    \node[rotate=90] at (-0.6,1) {\small Step2};
\end{tikzpicture}
\caption{$\times$ Answer}
\end{subfigure}
\caption{Trained Qwen2.5 1.5B with $\alpha=\infty$ achieves only $0.01\%$ final answer accuracy.
% For both correct and incorrect final answers, the model fails at predicting any correct intermediate tokens.
}
\label{fig:1.5B_instruct_alpha_inf_step_breakdown}
\end{figure}

\paragraph{Qwen-2.5-7B-Instruct.} In addition to the observations in Section~\ref{subsec:cil_langsym_exps} with $\alpha=0$, an interesting trend with increasing $\alpha$ is that incorrect Step 2 predictions tend to have a relatively higher role in the incorrect final answer prediction. For instance when $\alpha=0.5$ results in $\tacc=0.6296$ (Figure~\ref{fig:7B_instruct_alpha_0.5_step_breakdown}), then $22.9\%$ of the $3704$ incorrect final answers can be attributed to incorrect Step 2 predictions and $11.8\%$ to incorrect Step 1 predictions. As $\alpha$ increases to $\infty$ (Figure~\ref{fig:7B_instruct_alpha_inf_step_breakdown}), these numbers change to $40.3\%$ and $4.3\%$ respectively. The gradual change can be clearly observed in Figure~\ref{fig:7B_instruct_alpha_1_step_breakdown} for $\alpha=1$ and Figure~\ref{fig:7B_instruct_alpha_2_step_breakdown} for $\alpha=2$.

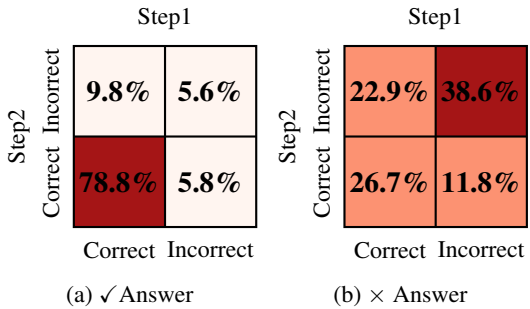
\begin{figure}[h]
\centering
\begin{subfigure}{0.45\linewidth}
\centering
\begin{tikzpicture}[scale=1.2]
    % Left heatmap - Correct Final Answers
    \fill[low] (0,1) rectangle (1,2);    % 81.5%
    \fill[low] (1,1) rectangle (2,2);     % 7.8%
    \fill[high] (0,0) rectangle (1,1);     % 5.0%
    \fill[low] (1,0) rectangle (2,1);     % 5.6%

    \draw[thick] (0,0) rectangle (2,2);
    \draw[thick] (0,1) -- (2,1);
    \draw[thick] (1,0) -- (1,2);

 \node at (0.5,0.5) {\textbf{78.8\%}}; % cc  
\node at (1.5,1.5) {\textbf{5.6\%}};  % ii  
\node at (1.5,0.5) {\textbf{5.8\%}};  % ic  
\node at (0.5,1.5) {\textbf{9.8\%}};  % ci

    % Axis labels
    \node[anchor=center] at (0.5,-0.25) {\footnotesize Correct};
    \node[anchor=center] at (1.5,-0.25) {\footnotesize Incorrect};
    \node[anchor=center, rotate=90] at (-0.25,0.5) {\footnotesize Correct};
    \node[anchor=center, rotate=90] at (-0.25,1.5) {\footnotesize Incorrect};

    % Step labels
    \node at (1,2.3) {\small Step1};
    \node[rotate=90] at (-0.6,1) {\small Step2};
\end{tikzpicture}
\caption{\checkmark Answer}
\end{subfigure}
\begin{subfigure}{0.45\linewidth}
\centering
\begin{tikzpicture}[scale=1.2]
    % Right heatmap - Incorrect Final Answers
    \fill[med] (0,1) rectangle (1,2);     % 30.5%
    \fill[high] (1,1) rectangle (2,2);     % 22.4%
    \fill[med] (0,0) rectangle (1,1);     % 11.6%
    \fill[med] (1,0) rectangle (2,1);    % 35.4%

    \draw[thick] (0,0) rectangle (2,2);
    \draw[thick] (0,1) -- (2,1);
    \draw[thick] (1,0) -- (1,2);

\node at (0.5,0.5) {\textbf{26.7\%}}; % cc  
\node at (1.5,1.5) {\textbf{38.6\%}}; % ii  
\node at (1.5,0.5) {\textbf{11.8\%}}; % ic  
\node at (0.5,1.5) {\textbf{22.9\%}}; % ci

 % Axis labels
    \node[anchor=center] at (0.5,-0.25) {\footnotesize Correct};
    \node[anchor=center] at (1.5,-0.25) {\footnotesize Incorrect};
    \node[anchor=center, rotate=90] at (-0.25,0.5) {\footnotesize Correct};
    \node[anchor=center, rotate=90] at (-0.25,1.5) {\footnotesize Incorrect};

    % Step labels
    \node at (1,2.3) {\small Step1};
    \node[rotate=90] at (-0.6,1) {\small Step2};
\end{tikzpicture}
\caption{$\times$ Answer}
\end{subfigure}
\caption{Trained Qwen2.5 7B with $\alpha=0.5$ achieves $62.96\%$ final answer accuracy.
% For both correct and incorrect final answers, the model achieves $78.8\%$ and $26.7\%$ accuracy for the intermediate tokens over both steps respectively.
}
\label{fig:7B_instruct_alpha_0.5_step_breakdown}
\end{figure}

\begin{figure}[h]
\centering
\begin{subfigure}{0.45\linewidth}
\centering
\begin{tikzpicture}[scale=1.2]
    % Left heatmap - Correct Final Answers
    \fill[low] (0,1) rectangle (1,2);    % 81.5%
    \fill[low] (1,1) rectangle (2,2);     % 7.8%
    \fill[high] (0,0) rectangle (1,1);     % 5.0%
    \fill[low] (1,0) rectangle (2,1);     % 5.6%

    \draw[thick] (0,0) rectangle (2,2);
    \draw[thick] (0,1) -- (2,1);
    \draw[thick] (1,0) -- (1,2);

\node at (0.5,0.5) {\textbf{77.8\%}}; % cc  
\node at (1.5,1.5) {\textbf{6.5\%}};  % ii  
\node at (1.5,0.5) {\textbf{5.8\%}};  % ic  
\node at (0.5,1.5) {\textbf{9.9\%}};  % ci

    % Axis labels
    \node[anchor=center] at (0.5,-0.25) {\footnotesize Correct};
    \node[anchor=center] at (1.5,-0.25) {\footnotesize Incorrect};
    \node[anchor=center, rotate=90] at (-0.25,0.5) {\footnotesize Correct};
    \node[anchor=center, rotate=90] at (-0.25,1.5) {\footnotesize Incorrect};

    % Step labels
    \node at (1,2.3) {\small Step1};
    \node[rotate=90] at (-0.6,1) {\small Step2};
\end{tikzpicture}
\caption{\checkmark Answer}
\end{subfigure}
\begin{subfigure}{0.45\linewidth}
\centering
\begin{tikzpicture}[scale=1.2]
    % Right heatmap - Incorrect Final Answers
    \fill[med] (0,1) rectangle (1,2);     % 30.5%
    \fill[high] (1,1) rectangle (2,2);     % 22.4%
    \fill[med] (0,0) rectangle (1,1);     % 11.6%
    \fill[med] (1,0) rectangle (2,1);    % 35.4%

    \draw[thick] (0,0) rectangle (2,2);
    \draw[thick] (0,1) -- (2,1);
    \draw[thick] (1,0) -- (1,2);

\node at (0.5,0.5) {\textbf{25.6\%}}; % cc  
\node at (1.5,1.5) {\textbf{39.8\%}}; % ii  
\node at (1.5,0.5) {\textbf{10.9\%}}; % ic  
\node at (0.5,1.5) {\textbf{23.7\%}}; % ci

 % Axis labels
    \node[anchor=center] at (0.5,-0.25) {\footnotesize Correct};
    \node[anchor=center] at (1.5,-0.25) {\footnotesize Incorrect};
    \node[anchor=center, rotate=90] at (-0.25,0.5) {\footnotesize Correct};
    \node[anchor=center, rotate=90] at (-0.25,1.5) {\footnotesize Incorrect};

    % Step labels
    \node at (1,2.3) {\small Step1};
    \node[rotate=90] at (-0.6,1) {\small Step2};
\end{tikzpicture}
\caption{$\times$ Answer}
\end{subfigure}
\caption{Trained Qwen2.5 7B with $\alpha=1$ achieves $58.85\%$ final answer accuracy.
% For both correct and incorrect final answers, the model achieves $77.8\%$ and $25.6\%$ accuracy for the intermediate tokens over both steps respectively.
}
\label{fig:7B_instruct_alpha_1_step_breakdown}
\end{figure}

\begin{figure}[h]
\centering
\begin{subfigure}{0.45\linewidth}
\centering
\begin{tikzpicture}[scale=1.2]
    % Left heatmap - Correct Final Answers
    \fill[low] (0,1) rectangle (1,2);    % 81.5%
    \fill[low] (1,1) rectangle (2,2);     % 7.8%
    \fill[high] (0,0) rectangle (1,1);     % 5.0%
    \fill[low] (1,0) rectangle (2,1);     % 5.6%

    \draw[thick] (0,0) rectangle (2,2);
    \draw[thick] (0,1) -- (2,1);
    \draw[thick] (1,0) -- (1,2);

   % fill here
\node at (0.5,0.5) {\textbf{74.2\%}}; % cc  
\node at (1.5,1.5) {\textbf{7.9\%}};  % ii  
\node at (1.5,0.5) {\textbf{6.7\%}};  % ic  
\node at (0.5,1.5) {\textbf{11.2\%}}; % ci

    % Axis labels
    \node[anchor=center] at (0.5,-0.25) {\footnotesize Correct};
    \node[anchor=center] at (1.5,-0.25) {\footnotesize Incorrect};
    \node[anchor=center, rotate=90] at (-0.25,0.5) {\footnotesize Correct};
    \node[anchor=center, rotate=90] at (-0.25,1.5) {\footnotesize Incorrect};

    % Step labels
    \node at (1,2.3) {\small Step1};
    \node[rotate=90] at (-0.6,1) {\small Step2};
\end{tikzpicture}
\caption{\checkmark Answer}
\end{subfigure}
\begin{subfigure}{0.45\linewidth}
\centering
\begin{tikzpicture}[scale=1.2]
    % Right heatmap - Incorrect Final Answers
    \fill[med] (0,1) rectangle (1,2);     % 30.5%
    \fill[high] (1,1) rectangle (2,2);     % 22.4%
    \fill[med] (0,0) rectangle (1,1);     % 11.6%
    \fill[low] (1,0) rectangle (2,1);    % 35.4%

    \draw[thick] (0,0) rectangle (2,2);
    \draw[thick] (0,1) -- (2,1);
    \draw[thick] (1,0) -- (1,2);

   %fill here    
\node at (0.5,0.5) {\textbf{24.5\%}}; % cc  
\node at (1.5,1.5) {\textbf{40.6\%}}; % ii  
\node at (1.5,0.5) {\textbf{11.1\%}}; % ic  
\node at (0.5,1.5) {\textbf{23.8\%}}; % ci

 % Axis labels
    \node[anchor=center] at (0.5,-0.25) {\footnotesize Correct};
    \node[anchor=center] at (1.5,-0.25) {\footnotesize Incorrect};
    \node[anchor=center, rotate=90] at (-0.25,0.5) {\footnotesize Correct};
    \node[anchor=center, rotate=90] at (-0.25,1.5) {\footnotesize Incorrect};

    % Step labels
    \node at (1,2.3) {\small Step1};
    \node[rotate=90] at (-0.6,1) {\small Step2};
\end{tikzpicture}
\caption{$\times$ Answer}
\end{subfigure}
\caption{Trained Qwen2.5 7B with $\alpha=2$ achieves $53.86\%$ final answer accuracy.
% For both correct and incorrect final answers, the model achieves $74.2\%$ and $24.5\%$ accuracy for the intermediate tokens over both steps respectively.
}
\label{fig:7B_instruct_alpha_2_step_breakdown}
\end{figure}

\begin{figure}[h]
\centering
\begin{subfigure}{0.45\linewidth}
\centering
\begin{tikzpicture}[scale=1.2]
    % Left heatmap - Correct Final Answers
    \fill[high] (0,1) rectangle (1,2);    % 81.5%
    \fill[med] (1,1) rectangle (2,2);     % 7.8%
    \fill[high] (0,0) rectangle (1,1);     % 5.0%
    \fill[low] (1,0) rectangle (2,1);     % 5.6%

    \draw[thick] (0,0) rectangle (2,2);
    \draw[thick] (0,1) -- (2,1);
    \draw[thick] (1,0) -- (1,2);

   % fill here
\node at (0.5,0.5) {\textbf{38.5\%}}; % cc  
\node at (1.5,1.5) {\textbf{27.3\%}}; % ii  
\node at (1.5,0.5) {\textbf{3.2\%}};  % ic  
\node at (0.5,1.5) {\textbf{31.0\%}}; % ci

    % Axis labels
    \node[anchor=center] at (0.5,-0.25) {\footnotesize Correct};
    \node[anchor=center] at (1.5,-0.25) {\footnotesize Incorrect};
    \node[anchor=center, rotate=90] at (-0.25,0.5) {\footnotesize Correct};
    \node[anchor=center, rotate=90] at (-0.25,1.5) {\footnotesize Incorrect};

    % Step labels
    \node at (1,2.3) {\small Step1};
    \node[rotate=90] at (-0.6,1) {\small Step2};
\end{tikzpicture}
\caption{\checkmark Answer}
\end{subfigure}
\begin{subfigure}{0.45\linewidth}
\centering
\begin{tikzpicture}[scale=1.2]
    % Right heatmap - Incorrect Final Answers
    \fill[high] (0,1) rectangle (1,2);     % 30.5%
    \fill[high] (1,1) rectangle (2,2);     % 22.4%
    \fill[med] (0,0) rectangle (1,1);     % 11.6%
    \fill[low] (1,0) rectangle (2,1);    % 35.4%

    \draw[thick] (0,0) rectangle (2,2);
    \draw[thick] (0,1) -- (2,1);
    \draw[thick] (1,0) -- (1,2);

   %fill here    
\node at (0.5,0.5) {\textbf{11.4\%}}; % cc  
\node at (1.5,1.5) {\textbf{43.9\%}}; % ii  
\node at (1.5,0.5) {\textbf{4.3\%}};  % ic  
\node at (0.5,1.5) {\textbf{40.3\%}}; % ci

 % Axis labels
    \node[anchor=center] at (0.5,-0.25) {\footnotesize Correct};
    \node[anchor=center] at (1.5,-0.25) {\footnotesize Incorrect};
    \node[anchor=center, rotate=90] at (-0.25,0.5) {\footnotesize Correct};
    \node[anchor=center, rotate=90] at (-0.25,1.5) {\footnotesize Incorrect};

    % Step labels
    \node at (1,2.3) {\small Step1};
    \node[rotate=90] at (-0.6,1) {\small Step2};
\end{tikzpicture}
\caption{$\times$ Answer}
\end{subfigure}
\caption{Trained Qwen2.5 7B with $\alpha=\infty$ achieves $5.06\%$ final answer accuracy.
% For both correct and incorrect final answers, the model achieves $38.5\%$ and $11.4\%$ accuracy for the intermediate tokens over both steps respectively.
}
\label{fig:7B_instruct_alpha_inf_step_breakdown}
\end{figure}
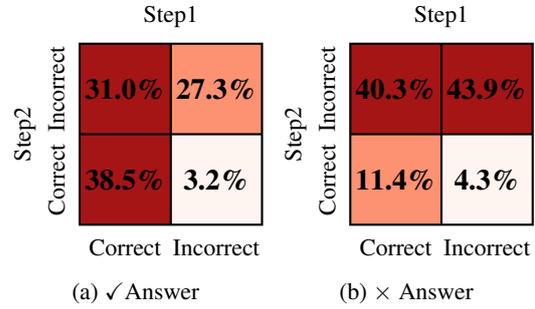

\paragraph{Reliance on the underlying DAG.} In addition to verifying the correctness of intermediate steps, we also analyze their relevance by checking if the final answer is causally dependent on them via the ground truth DAG. The results for \texttt{Qwen-2.5-0.5B-Instruct} are presented in Figure~\ref{fig:0.5B_instruct_alpha_0_dag_breakdown} ($\alpha=0$), Figure~\ref{fig:0.5B_instruct_alpha_0.5_dag_breakdown} ($\alpha=0.5$), Figure~\ref{fig:0.5B_instruct_alpha_1_dag_breakdown} ($\alpha=1$), Figure~\ref{fig:0.5B_instruct_alpha_2_dag_breakdown} ($\alpha=2$) and Figure~\ref{fig:0.5B_instruct_alpha_inf_dag_breakdown} ($\alpha=\infty$). In particular, Figure~\ref{fig:0.5B_instruct_alpha_0_dag_breakdown:cor_ans_step_1} for $\alpha=0$ illustrates that out of the $44$ prompts with correct final answer predictions, $0$ prompts had correct intermediate Step 1 predictions when the final answer had a causal dependency on this intermediate word (i.e Step 1 output). Similarly, when there was no underlying causal dependency, $6$ prompts had correct whereas $34$ of them had incorrect Step 1 predictions.

The results for \texttt{Qwen-2.5-1.5B-Instruct} are presented in Figure~\ref{fig:1.5B_instruct_alpha_0_dag_breakdown} ($\alpha=0$), Figure~\ref{fig:1.5B_instruct_alpha_0.5_dag_breakdown} ($\alpha=0.5$), Figure~\ref{fig:1.5B_instruct_alpha_1_dag_breakdown} ($\alpha=1$), Figure~\ref{fig:1.5B_instruct_alpha_2_dag_breakdown} ($\alpha=2$) and Figure~\ref{fig:1.5B_instruct_alpha_inf_dag_breakdown} ($\alpha=\infty$) and for \texttt{Qwen-2.5-7B-Instruct} are presented in Figure~\ref{fig:7B_instruct_alpha_0_dag_breakdown} ($\alpha=0$), Figure~\ref{fig:7B_instruct_alpha_0.5_dag_breakdown} ($\alpha=0.5$), Figure~\ref{fig:7B_instruct_alpha_1_dag_breakdown} ($\alpha=1$), Figure~\ref{fig:7B_instruct_alpha_2_dag_breakdown} ($\alpha=2$) and Figure~\ref{fig:7B_instruct_alpha_inf_dag_breakdown} ($\alpha=\infty$). A key observation is that: the final answer is causally dependent on the Step 2 intermediate word in (relatively) more number of prompts than the Step 1 intermediate word. This inherent preference for Step 2 outputs in the DAG (Figure~\ref{fig:7B_instruct_alpha_0_dag_breakdown}) explains the results in Figure~\ref{fig:7b_instruct_alpha_0_step_analysis} (Section~\ref{subsec:cil_langsym_exps}) and the extended results presented above where incorrect Step 2 outputs resulted in relatively more incorrect final answers when compared to Step 1. In essence, as the model continues to learn the underlying DAG to predict the final answer, its sensitivity to important intermediate steps also increases. Thus, mistakes in important intermediate steps can lead to incorrect final answers.

%----- 0.5B -------
\begin{figure}[h]
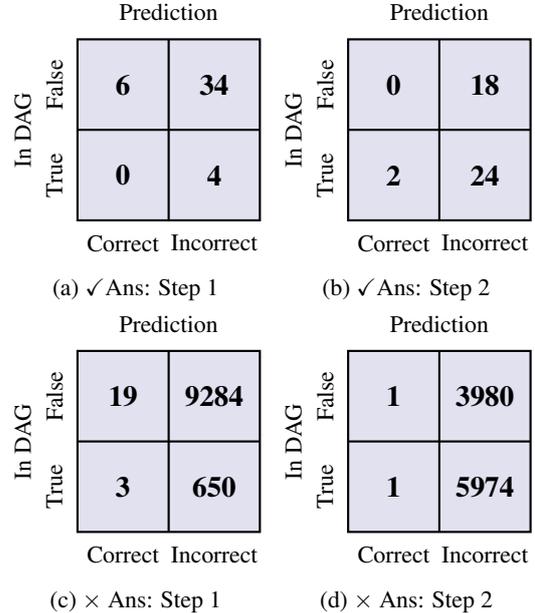

\centering
% ---- correct answer
\begin{subfigure}{0.45\linewidth}
\centering

\caption{\checkmark Ans: Step 1}
\label{fig:0.5B_instruct_alpha_0_dag_breakdown:cor_ans_step_1}
\end{subfigure}
\begin{subfigure}{0.45\linewidth}
\centering
%
\caption{\checkmark Ans: Step 2}
\end{subfigure}

% ---- incorrect answer
\begin{subfigure}{0.45\linewidth}
\centering
%
\caption{ $\times$ Ans: Step 1}
\end{subfigure}
\begin{subfigure}{0.45\linewidth}
\centering
%
\caption{$\times$ Ans: Step 2}
\end{subfigure}

\caption{Breakdown of \texttt{Qwen-2.5-0.5B-Instruct} model predictions trained with $\alpha=0$, based on correct (\checkmark) and incorrect ($\times$) final answers, and inclusion of intermediate steps in the DAG.}
\label{fig:0.5B_instruct_alpha_0_dag_breakdown}
\end{figure}

\begin{figure}[h!]
\centering
\begin{subfigure}{0.45\linewidth}
\centering
%
\caption{\checkmark Ans: Step 1}
\end{subfigure}
\begin{subfigure}{0.45\linewidth}
\centering
%
\caption{\checkmark Ans: Step 2}
\end{subfigure}

% ---- incorrect answer
\begin{subfigure}{0.45\linewidth}
\centering
%
\caption{ $\times$ Ans: Step 1}
\end{subfigure}
\begin{subfigure}{0.45\linewidth}
\centering
%
\caption{$\times$ Ans: Step 2}
\end{subfigure}
\caption{DAG breakdown for Trained Qwen2.5 0.5 with $\alpha=0.5$ .}
\label{fig:0.5B_instruct_alpha_0.5_dag_breakdown}
\end{figure}

\begin{figure}[h!]
\centering
\begin{subfigure}{0.45\linewidth}
\centering
%
\caption{\checkmark Ans: Step 1}
\end{subfigure}
\begin{subfigure}{0.45\linewidth}
\centering
%
\caption{\checkmark Ans: Step 2}
\end{subfigure}

% ---- incorrect answer
\begin{subfigure}{0.45\linewidth}
\centering
%
\caption{ $\times$ Ans: Step 1}
\end{subfigure}
\begin{subfigure}{0.45\linewidth}
\centering
%
\caption{$\times$ Ans: Step 2}
\end{subfigure}
\caption{DAG breakdown for Trained Qwen2.5 0.5B with $\alpha=1$ .}
\label{fig:0.5B_instruct_alpha_1_dag_breakdown}
\end{figure}

\begin{figure}[h!]
\centering
\begin{subfigure}{0.45\linewidth}
\centering
%
\caption{\checkmark Ans: Step 1}
\end{subfigure}
\begin{subfigure}{0.45\linewidth}
\centering
%
\caption{$\times$ Ans: Step 2}
\end{subfigure}

\caption{DAG breakdown for Trained Qwen2.5 0.5B with $\alpha=2$ .}
\label{fig:0.5B_instruct_alpha_2_dag_breakdown}
\end{figure}

\begin{figure}[h!]
\centering
\begin{subfigure}{0.45\linewidth}
\centering
%
\caption{$\times$ Ans: Step 2}
\end{subfigure}
\caption{DAG breakdown for Trained Qwen2.5 0.5B with $\alpha=\infty$ .}
\label{fig:0.5B_instruct_alpha_inf_dag_breakdown}
\end{figure}

%---------1.5B

\begin{figure}[h!]
\centering
\begin{subfigure}{0.45\linewidth}
\centering
%
\caption{$\times$ Ans: Step 2}
\end{subfigure}

\caption{DAG breakdown for Trained Qwen2.5 1.5B with $\alpha=0$ .}
\label{fig:1.5B_instruct_alpha_0_dag_breakdown}
\end{figure}

\begin{figure}[h!]
\centering
\begin{subfigure}{0.45\linewidth}
\centering
%
\caption{$\times$ Ans: Step 2}
\end{subfigure}
\caption{DAG breakdown for Trained Qwen2.5 1.5B with $\alpha=0.5$ .}
\label{fig:1.5B_instruct_alpha_0.5_dag_breakdown}
\end{figure}

\begin{figure}[h!]
\centering
\begin{subfigure}{0.45\linewidth}
\centering
%
\caption{$\times$ Ans: Step 2}
\end{subfigure}

\caption{DAG breakdown for Trained Qwen2.5 1.5B with $\alpha=\infty$ .}
\label{fig:1.5B_instruct_alpha_inf_dag_breakdown}
\end{figure}

%------ 7B

\begin{figure}[h!]
\centering
\begin{subfigure}{0.45\linewidth}
\centering
%
\caption{$\times$ Ans: Step 2}
\end{subfigure}

\caption{DAG breakdown for Trained Qwen2.5 7B with $\alpha=0$ .}
\label{fig:7B_instruct_alpha_0_dag_breakdown}
\end{figure}

%----alpha 0.5
\begin{figure}[h!]
\centering
\begin{subfigure}{0.45\linewidth}
\centering
%
\caption{$\times$ Ans: Step 2}
\end{subfigure}

\caption{DAG breakdown for Trained Qwen2.5 7B with $\alpha=0.5$ .}
\label{fig:7B_instruct_alpha_0.5_dag_breakdown}
\end{figure}

%----alpha 1
\begin{figure}[h!]
\centering
\begin{subfigure}{0.45\linewidth}
\centering
%
\caption{$\times$ Ans: Step 2}
\end{subfigure}

\caption{DAG breakdown for Trained Qwen2.5 7B with $\alpha=1$ .}
\label{fig:7B_instruct_alpha_1_dag_breakdown}
\end{figure}

%----- alpha 2
\begin{figure}[h!]
\centering
\begin{subfigure}{0.45\linewidth}
\centering
%
\caption{$\times$ Ans: Step 2}
\end{subfigure}

\caption{DAG breakdown for Trained Qwen2.5 7B with $\alpha=2$ .}
\label{fig:7B_instruct_alpha_2_dag_breakdown}
\end{figure}

%------ alpha inf
\begin{figure}[h!]
\centering
\begin{subfigure}{0.45\linewidth}
\centering
%
\caption{$\times$ Ans: Step 2}
\end{subfigure}

\caption{DAG breakdown for Trained Qwen2.5 7B with $\alpha=\infty$ .}
\label{fig:7B_instruct_alpha_inf_dag_breakdown}
\end{figure}

\begin{figure*}[h!]
    \centering
    \includegraphics[width=\linewidth]{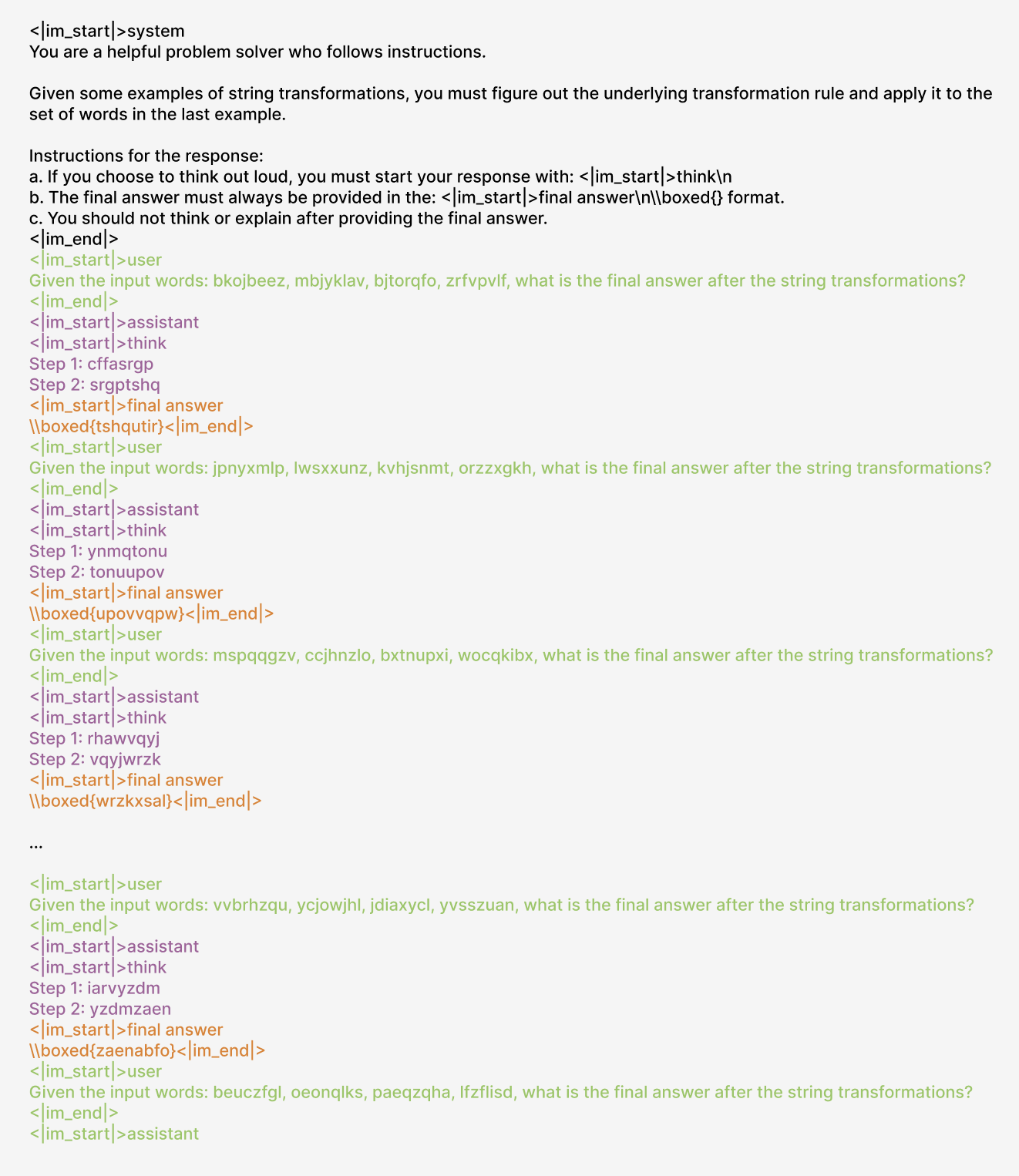}
    \caption{Example evaluation prompt from the \texttt{CIL-LangSym} dataset with $K=39$ CoT examples (truncated in the image for brevity) and the query formatted using the \texttt{Qwen-2.5-7B-Instruct} tokenizer chat template.}
    \label{fig:cil_langsym_example}
\end{figure*}

\clearpage

%% 7b 
\begin{table*}[htbp!]
\centering
\begin{tabular}{|c|c|c|}
\hline
\textbf{Model} & \textbf{`Force Think' Response} & \textbf{`Force Answer' Response} \\
\hline
\textbf{Ground Truth} & \begin{tcolorbox}[colback=black!10, colframe=black!80, width=0.3\linewidth, arc=0pt, outer arc=0pt, boxsep=0pt, left=5pt, right=5pt, top=0pt, bottom=0pt]
{\tiny\begin{lstlisting}[basicstyle=\tiny\ttfamily\color{black}\bfseries,
  keywordstyle=\color{codeblue}\bfseries, % Keywords in bold blue
  commentstyle=\color{codegreen}\itshape, % Comments in italic green
  stringstyle=\color{black}            % Strings in gray
]
Step 1: aghmarib
Step 2: aribbsjc
<|im_start|>final answer
\\boxed{bsjcctkd}
\end{lstlisting}}
\end{tcolorbox} & \begin{tcolorbox}[colback=black!10, colframe=black!80, width=0.3\linewidth, arc=0pt, outer arc=0pt, boxsep=0pt, left=5pt, right=5pt, top=0pt, bottom=0pt]
{\tiny\begin{lstlisting}[basicstyle=\tiny\ttfamily\color{black}\bfseries,
  keywordstyle=\color{codeblue}\bfseries, % Keywords in bold blue
  commentstyle=\color{codegreen}\itshape, % Comments in italic green
  stringstyle=\color{black}            % Strings in gray
]
\\boxed{bsjcctkd}
\end{lstlisting}}
\end{tcolorbox} \\
\hline

\textbf{Baseline} & \begin{tcolorbox}[colback=black!10, colframe=black!80, width=0.3\linewidth, arc=0pt, outer arc=0pt, boxsep=0pt, left=5pt, right=5pt, top=0pt, bottom=0pt]
{\tiny\begin{lstlisting}[basicstyle=\tiny\ttfamily\color{black}\bfseries,
  keywordstyle=\color{codeblue}\bfseries, % Keywords in bold blue
  commentstyle=\color{codegreen}\itshape, % Comments in italic green
  stringstyle=\color{black}            % Strings in gray
]
Step 1: cfdwgrjm
Step 2: grjmhnko
Step 3: hnkoiplp
...
Step 90: jjjjjjjj
\end{lstlisting}}
\end{tcolorbox} & \begin{tcolorbox}[colback=black!10, colframe=black!80, width=0.3\linewidth, arc=0pt, outer arc=0pt, boxsep=0pt, left=5pt, right=5pt, top=0pt, bottom=0pt]
{\tiny\begin{lstlisting}[basicstyle=\tiny\ttfamily\color{black}\bfseries,
  keywordstyle=\color{codeblue}\bfseries, % Keywords in bold blue
  commentstyle=\color{codegreen}\itshape, % Comments in italic green
  stringstyle=\color{black}            % Strings in gray
]
\\boxed{mndyehjt}
\end{lstlisting}}
\end{tcolorbox} \\
\hline

\textbf{Trained ($\alpha = 0$)} & \begin{tcolorbox}[colback=black!10, colframe=black!80, width=0.3\linewidth, arc=0pt, outer arc=0pt, boxsep=0pt, left=5pt, right=5pt, top=0pt, bottom=0pt]
{\tiny\begin{lstlisting}[basicstyle=\tiny\ttfamily\color{black}\bfseries,
  keywordstyle=\color{codeblue}\bfseries, % Keywords in bold blue
  commentstyle=\color{codegreen}\itshape, % Comments in italic green
  stringstyle=\color{black}            % Strings in gray
]
Step 1: aghmarib
Step 2: aribbsjc
<|im_start|>final answer
\\{bsjcttkd}
\end{lstlisting}}
\end{tcolorbox} &  \begin{tcolorbox}[colback=black!10, colframe=black!80, width=0.3\linewidth, arc=0pt, outer arc=0pt, boxsep=0pt, left=5pt, right=5pt, top=0pt, bottom=0pt]
{\tiny\begin{lstlisting}[basicstyle=\tiny\ttfamily\color{black}\bfseries,
  keywordstyle=\color{codeblue}\bfseries, % Keywords in bold blue
  commentstyle=\color{codegreen}\itshape, % Comments in italic green
  stringstyle=\color{black}            % Strings in gray
]
\\boxed{mhmbairo}
\end{lstlisting}}
\end{tcolorbox} \\
\hline

\textbf{Trained ($\alpha = 0.5$)} & \begin{tcolorbox}[colback=black!10, colframe=black!80, width=0.3\linewidth, arc=0pt, outer arc=0pt, boxsep=0pt, left=5pt, right=5pt, top=0pt, bottom=0pt]
{\tiny\begin{lstlisting}[basicstyle=\tiny\ttfamily\color{black}\bfseries,
  keywordstyle=\color{codeblue}\bfseries, % Keywords in bold blue
  commentstyle=\color{codegreen}\itshape, % Comments in italic green
  stringstyle=\color{black}            % Strings in gray
]
Step 1: aghmarib
Step 2: aribbsjc
<|im_start|>final answer
\\boxed{bsjcctkd}
\end{lstlisting}}
\end{tcolorbox} & \begin{tcolorbox}[colback=black!10, colframe=black!80, width=0.3\linewidth, arc=0pt, outer arc=0pt, boxsep=0pt, left=5pt, right=5pt, top=0pt, bottom=0pt]
{\tiny\begin{lstlisting}[basicstyle=\tiny\ttfamily\color{black}\bfseries,
  keywordstyle=\color{codeblue}\bfseries, % Keywords in bold blue
  commentstyle=\color{codegreen}\itshape, % Comments in italic green
  stringstyle=\color{black}            % Strings in gray
]
\\boxed{btfjcugh}
\end{lstlisting}}
\end{tcolorbox} \\
\hline

\textbf{Trained ($\alpha = 1$)} & \begin{tcolorbox}[colback=black!10, colframe=black!80, width=0.3\linewidth, arc=0pt, outer arc=0pt, boxsep=0pt, left=5pt, right=5pt, top=0pt, bottom=0pt]
{\tiny\begin{lstlisting}[basicstyle=\tiny\ttfamily\color{black}\bfseries,
  keywordstyle=\color{codeblue}\bfseries, % Keywords in bold blue
  commentstyle=\color{codegreen}\itshape, % Comments in italic green
  stringstyle=\color{black}            % Strings in gray
]
Step 1: aghmairb
Step 2: airbbjsc
<|im_start|>final answer
\\boxed{bjsccktn}
\end{lstlisting}}
\end{tcolorbox} & \begin{tcolorbox}[colback=black!10, colframe=black!80, width=0.3\linewidth, arc=0pt, outer arc=0pt, boxsep=0pt, left=5pt, right=5pt, top=0pt, bottom=0pt]
{\tiny\begin{lstlisting}[basicstyle=\tiny\ttfamily\color{black}\bfseries,
  keywordstyle=\color{codeblue}\bfseries, % Keywords in bold blue
  commentstyle=\color{codegreen}\itshape, % Comments in italic green
  stringstyle=\color{black}            % Strings in gray
]
\\boxed{btfccugh}
\end{lstlisting}}
\end{tcolorbox} \\
\hline

\textbf{Trained ($\alpha = 2$)} & \begin{tcolorbox}[colback=black!10, colframe=black!80, width=0.3\linewidth, arc=0pt, outer arc=0pt, boxsep=0pt, left=5pt, right=5pt, top=0pt, bottom=0pt]
{\tiny\begin{lstlisting}[basicstyle=\tiny\ttfamily\color{black}\bfseries,
  keywordstyle=\color{codeblue}\bfseries, % Keywords in bold blue
  commentstyle=\color{codegreen}\itshape, % Comments in italic green
  stringstyle=\color{black}            % Strings in gray
]
Step 1: aghmarih
Step 2: arihbsji
<|im_start|>final answer
\\boxed{bsjictkj}
\end{lstlisting}}
\end{tcolorbox} & \begin{tcolorbox}[colback=black!10, colframe=black!80, width=0.3\linewidth, arc=0pt, outer arc=0pt, boxsep=0pt, left=5pt, right=5pt, top=0pt, bottom=0pt]
{\tiny\begin{lstlisting}[basicstyle=\tiny\ttfamily\color{black}\bfseries,
  keywordstyle=\color{codeblue}\bfseries, % Keywords in bold blue
  commentstyle=\color{codegreen}\itshape, % Comments in italic green
  stringstyle=\color{black}            % Strings in gray
]
\\boxed{tmbuuncv}
\end{lstlisting}}
\end{tcolorbox} \\
\hline

\textbf{Trained ($\alpha = \infty$)} & \begin{tcolorbox}[colback=black!10, colframe=black!80, width=0.3\linewidth, arc=0pt, outer arc=0pt, boxsep=0pt, left=5pt, right=5pt, top=0pt, bottom=0pt]
{\tiny\begin{lstlisting}[basicstyle=\tiny\ttfamily\color{black}\bfseries,
  keywordstyle=\color{codeblue}\bfseries, % Keywords in bold blue
  commentstyle=\color{codegreen}\itshape, % Comments in italic green
  stringstyle=\color{black}            % Strings in gray
]
Step 1: aghmarib
Step 2: nbsjbstk
<|im_start|>final answer
\\boxed{cbtkcdul}
\end{lstlisting}}
\end{tcolorbox} & \begin{tcolorbox}[colback=black!10, colframe=black!80, width=0.3\linewidth, arc=0pt, outer arc=0pt, boxsep=0pt, left=5pt, right=5pt, top=0pt, bottom=0pt]
{\tiny\begin{lstlisting}[basicstyle=\tiny\ttfamily\color{black}\bfseries,
  keywordstyle=\color{codeblue}\bfseries, % Keywords in bold blue
  commentstyle=\color{codegreen}\itshape, % Comments in italic green
  stringstyle=\color{black}            % Strings in gray
]
\\boxed{aghmbnhc}
\end{lstlisting}}
\end{tcolorbox} \\
\hline

\end{tabular}
\caption{Ground truth output for the evaluation prompt in Figure~\ref{fig:cil_langsym_example} and responses from baseline and SFT'ed \texttt{Qwen-2.5-7B-Instruct} models using `Force Think' strategy with prompt suffix: \texttt{<|im\_start|>think\textbackslash n} and `Force Answer' strategy with prompt suffix: \texttt{<|im\_start|>final answer\textbackslash n}. Notice that only the $\alpha=0.5$ model gives the correct answer with `Force Think' strategy whereas none of the responses with `Force Answer' are correct. }
\label{tab:cil_ft_fa}
\end{table*}
\clearpage

% \section{Summary of Notations}
% \label{app:notation_summary}

\begin{table*}[!htbp]%[ht]
\centering
\begin{tabular}{|c|c|}
\hline
\textbf{Notation} & \textbf{Description}  \\ \hline
$t_{\texttt{pad}}$ & The padding token \\ \hline
$t_{\texttt{bos}}$ & The begin-of-sequence token \\ \hline
$t_{\texttt{eos}}$ & The end-of-sequence token \\ \hline
$t_{\texttt{inp\_start}}$ & The token to indicate the start of $N$  \textit{input} tokens. \\ \hline
$t_{\texttt{inp\_end}}$ & The token to indicate the end of $N$  \textit{input} tokens. \\ \hline
$t_{\texttt{think\_start}}$ & The token to indicate the start of $C-1$  \textit{intermediate} tokens. \\ \hline
$t_{\texttt{think\_end}}$ & The token to indicate the end of $C-1$  \textit{intermdiate} tokens. \\ \hline
$t_{\texttt{ans\_start}}$ & The token to indicate the start of the  \textit{answer} token. \\ \hline
$t_{\texttt{ans\_end}}$ & The token to indicate the end of the  \textit{answer} token. \\ \hline
$\gV_{\texttt{normal}}$ & A vocabulary of abstract tokens used as inputs and chain tokens. \\ \hline
$\gV_{\texttt{special}}$ & A vocabulary of abstract tokens used as delimiters \\ \hline
$\gV$ & A unified vocabulary of abstract tokens ($\gV = \gV_{\texttt{normal}} \cup \gV_{\texttt{special}}$) \\ \hline
$\mE_{\texttt{data}}$ & The `unknown' data embedding matrix corresponding to $\gV$ \\ \hline
$d$ & Embedding dimension for abstract tokens based on $\mE_{\texttt{data}} \in \sR^{|\gV| \times d}$ \\ \hline
$\gG$ & Function class to filter tokens \\ \hline
$\gH$ & Function class to process abstract token embeddings (i.e, rows in $\mE_{\texttt{data}}$) \\ \hline
$\gF$ & Function class composed of $\gG, \gH$ \\ \hline
$l$ & MLP depth in $\gH$ \\ \hline
$\phi$ & MLP activation function in $\gH$ \\ \hline
 $N$   & Number of input tokens per example   \\ \hline  
 $M$   & Number of tokens selected by $\gG$   \\ \hline  
 $C$   & Number of chain tokens (Chain length)  \\ \hline 
 $K$   & Number of examples per sequence   \\ \hline  
 \texttt{TF}   & A decoder only transformer model   \\ \hline  
 $\vx = (x_1, \cdots, x_N) \in \gV^N$ & Input tokens in an example \\ \hline  
 $\vy = (y_1, \cdots, y_C) \in \gV^C$ & Chain tokens in an example \\ \hline  
  $\vp^K(f, r_{CoT})$ & Tokenized sequence generated using $f \in \gF$ and $r_{CoT}$ with $K$ examples. \\ \hline
  $\tTF(\cdot)$ & A single auto-regressive greedy token generation by the $\tTF$ model. \\ \hline
   $\tTF^{\circ C_{eos}}(\cdot)$ & The auto-regressive greedy token generation by the $\tTF$ model until $t_{eos}$. \\ \hline
\end{tabular}
\caption{A summary of notations used throughout the paper.}
\label{tab:notations}
\end{table*}

%-----------End of heatmaps-----
\end{document}

%% file: notations.tex
\makeatletter
\renewcommand{\paragraph}{%
  \@startsection{paragraph}{4}%
  {\z@}{1.5ex \@plus 1ex \@minus .2ex}{-1em}%
  {\normalfont\normalsize\bfseries}%
}
\makeatother

\def\mC{{\mathbf{C}}}
\def\mD{{\mathbf{D}}}
\def\mE{{\mathbf{E}}}

\def\mM{{\mathbf{M}}}
\def\mN{{\mathbf{N}}}

\def\va{{\mathbf{a}}}

\def\ve{{\mathbf{e}}}

\def\vh{{\mathbf{h}}}

\def\vm{{\mathbf{m}}}

\def\vp{{\mathbf{p}}}

\def\vt{{\mathbf{t}}}

\def\vx{{\mathbf{x}}}
\def\vy{{\mathbf{y}}}
\def\vz{{\mathbf{z}}}

\def\gD{{\mathcal{D}}}

\def\gF{{\mathcal{F}}}
\def\gG{{\mathcal{G}}}
\def\gH{{\mathcal{H}}}

\def\gN{{\mathcal{N}}}

\def\gU{{\mathcal{U}}}
\def\gV{{\mathcal{V}}}

\def\sI{{\mathbb{I}}}

\def\sR{{\mathbb{R}}}

\def\tTF{{\texttt{TF}}}
\def\tCE{{\texttt{CE}}}
\def\tacc{{\texttt{accuracy}}}

\newtheorem{theorem}{Theorem}[section]